\documentclass[Afour,times,sageh]{sagej}
\pdfoutput=1
\usepackage[T1]{fontenc}
\usepackage[utf8]{inputenc}
\usepackage{amsmath,amssymb,amsthm}
\usepackage{microtype}
\usepackage{enumitem}
\usepackage{graphicx}
\usepackage{subcaption}
\usepackage{multirow,balance}
\usepackage{url}

\newif\ifarxiv
\arxivtrue
\ifarxiv
\makeatletter
\def\@maketitle{%
\vspace*{-34pt}%
\null%
\begin{center}
\begin{sf}
\begin{minipage}[t]{\textwidth}
  \vskip 12.5pt%
    {\raggedright\titlesize\textbf{\@title} \par}%
    \vskip 1.5em%
\end{minipage}
{\par\large%
      \lineskip .5em%
      {\raggedright\textbf{\@author}
      \par}}
     \vskip 40pt%
    {\noindent\usebox\absbox\par}
    {\vspace{20pt}%
      {\noindent\normalsize\@keywords}\par}
      \end{sf}
      \end{center}
      \vspace{22pt}
        \par%
  }
\def\ps@title{%
\def\@oddhead{\parbox{\textwidth}{\mbox{}\\[-1pt]%
\noindent\rule{\textwidth}{0.5pt}%
}}%
\let\@evenhead\@oddhead
\def\@oddfoot{}%
\let\@evenfoot\@oddfoot}
\def\ps@sagepage{%
\let\@mkboth\@gobbletwo
\def\@evenhead{\parbox{\textwidth}{%
\normalsize\sagesf\thepage\hfill\\[-6pt]
\noindent\rule{\textwidth}{0.25pt}}}
\def\@oddhead{\parbox{\textwidth}{%
\normalsize\sagesf{\itshape{\leftmark}}\hfill\thepage\\[-6pt]
\noindent\rule{\textwidth}{0.25pt}}}
\def\@evenfoot{}%
\def\@oddfoot{\@evenfoot}}
\makeatother
\fi
\usepackage{booktabs}
\usepackage{colortbl}
\usepackage{xcolor}
\definecolor{dkyellow}{rgb}{0.72,0.53,0.0}
\definecolor{grayblue}{rgb}{0.13,0.22,0.42}
\definecolor{dkgreen}{HTML}{1B6B2A}
\definecolor{dkred}{HTML}{A11212}
\newcommand{\cgood}[1]{\textcolor{dkgreen}{$\mathbf{#1}$}}
\newcommand{\cmid}[1]{\textcolor{dkyellow}{$#1$}}
\newcommand{\cbad}[1]{\textcolor{dkred}{$#1$}}
\newcommand{\wms}[1]{%
  \ifdim#1sp<300sp\textcolor{dkgreen}{$#1$}%
  \else\ifdim#1sp>1500sp\textcolor{dkred}{$#1$}%
  \else\textcolor{dkyellow}{$#1$}\fi\fi}
\newcommand{\wmsb}[1]{%
  \ifdim#1sp<300sp\textcolor{dkgreen}{$\mathbf{#1}$}%
  \else\ifdim#1sp>1500sp\textcolor{dkred}{$\mathbf{#1}$}%
  \else\textcolor{dkyellow}{$\mathbf{#1}$}\fi\fi}
\usepackage{algorithm}
\usepackage[indLines=true]{algpseudocodex}
\makeatletter
\patchcmd{\algpx@drawIndentLine}
  {\draw[algpxIndentLine]}
  {\draw[algpxIndentLine,algpxIndentLine#1/.try]}
  {\typeout{ALGPX-PATCH-OK}}{\typeout{ALGPX-PATCH-FAILED}}
\makeatother
\tikzset{
  algpxIndentLine0/.style={draw=kwrep},
  algpxIndentLine1/.style={draw=kwfor},
  algpxIndentLine2/.style={draw=kwfor},
  algpxIndentLine3/.style={draw=kwif},
  algpxIndentLine4/.style={draw=kwrep},
  algpxIndentLine5/.style={draw=kwfor},
  algpxIndentLine6/.style={draw=kwif},
}

\graphicspath{{./}}

\makeatletter
\renewcommand\subsubsection{\@startsection{subsubsection}{3}{\z@}{0.5\@bls plus .3\@bls minus .1\@bls}{4pt\@afterindentfalse}{\sagesf\normalsize\itshape}}
                                
\makeatother

\allowdisplaybreaks
\setlist[itemize]{leftmargin=1.4em,itemsep=1pt,topsep=2pt}
\makeatletter
\def\thmhead@parens#1#2#3{%
  \thmname{#1}\thmnumber{\@ifnotempty{#1}{ }\@upn{#2}}%
  \thmnote{ {\the\thm@notefont(#3).}}}
\let\thmhead\thmhead@parens
\makeatother

\theoremstyle{definition}
\newtheorem{definition}{Definition}[section]

\newtheorem{remark}[definition]{Remark}
\newtheorem{example}[definition]{Example}
\newtheorem{procedure}[definition]{Procedure}
\theoremstyle{plain}
\newtheorem{lemma}[definition]{Lemma}
\newtheorem{theorem}[definition]{Theorem}
\newtheorem{proposition}[definition]{Proposition}

\newtheorem{claim}[definition]{Claim}

\newcommand{\X}{\mathcal{X}}
\newcommand{\A}{\mathcal{A}}

\newcommand{\R}{\mathbb{R}}
\newcommand{\Z}{\mathbb{Z}}
\newcommand{\tra}{\pi}

\newcommand{\preds}{\mathcal{G}}

\newcommand{\W}{\mathcal{W}}
\newcommand{\roll}{\operatorname{roll}}
\newcommand{\jet}{\operatorname{jet}}

\begin{document}

\runninghead{Deshpande and How}

\title{Dispersive Forward Tree Search for Optimal Control:
Coverage, Complexity, and Computation}

\author{Shashank A. Deshpande\affilnum{1} and
Jonathan P. How\affilnum{1}}

\affiliation{\affilnum{1}Department of Aeronautics and
Astronautics, Massachusetts Institute of Technology,
Cambridge, MA, USA}

\corrauth{Shashank A. Deshpande, Department of Aeronautics and
Astronautics, Massachusetts Institute of Technology,
77 Massachusetts Avenue, Cambridge, MA 02139, USA.}

\email{croshank@mit.edu}

\begin{abstract}
Steering-based planners require solutions to state-to-state
boundary value problems, which can be inaccessible for nonlinear
platforms.
Forward propagation evades the steering requirement, but the
finite-sample behavior of the associated planners remains
uncharacterized and their implementations underperform in
practice. This paper develops a propagation-based kinodynamic
planner with deterministic finite-sample near-optimality
guarantees. We work
within the large class of differentially flat nonlinear
systems and show that a forward tree of locally dispersive
control commands contains a near-optimal trajectory at a
certified tree size. We provide a general mechanism to construct
dispersive command sets for control-affine systems, which are
necessary to implement the search algorithm prescribed by the
theory. We show that covering the certified trajectory
class irrespective of cost provably demands a tree exponentially
sized in the problem horizon, and present a cost-conditioned
dominance pruning procedure that retains near-optimality at a
tree size polynomial in the horizon. We implement the resulting
search algorithm, Dispersive Forward Tree search
(DFT\textsuperscript{*}), as breadth-first expansion
of the forward tree, which maps naturally onto parallel
hardware. We design efficient dispersive samplers for the
unicycle, the trailer car, and the quadrotor and evaluate
challenging planning tasks for these platforms. DFT\textsuperscript{*}
delivers consistently competitive and often substantially better
solution quality than state-of-the-art kinodynamic planners at
comparable solution times on embedded-tier processors,
accelerating further as parallel compute is scaled.
We also implement DFT\textsuperscript{*} in a receding-horizon loop to demonstrate
real-time planning in dynamic environments at embedded-tier
compute budgets.
\end{abstract}


\maketitle
\footnotetext[0]{\sagesf Code:
\url{https://github.com/croshank/DFTSearch}}

\section{Introduction}
\label{sec:intro}

This paper addresses the finite-sample performance
of sampling-based kinodynamic planners that search the reachable
solution space by forward propagation of sampled control
commands. For the class of differentially flat nonlinear
systems, we develop deterministic, non-asymptotic guarantees on
the solution quality attainable at a finite sampling budget.
Sampling-based algorithms constitute a canonical approach to
kinodynamic motion planning, with theoretical guarantees of
completeness and
optimality as their foundational properties. Early algorithms in
this field explore a configuration space by randomized sampling,
most prominently under the Voronoi bias, and guarantee
probabilistic convergence to a solution \citep{kavraki1996prm,
lavalle2001rrt, kleinbort2019rrt}. Successive refinements
strengthen this guarantee to asymptotic optimality and provide
probabilistic convergence to an optimal solution
\citep{karaman2011sampling, li2016asymptotically}.

A vast majority of sampling-based algorithms sample system states
in the configuration space and rely on state-to-state steering
executed by solving a boundary value problem (BVP). Strong
theoretical properties are available in the incumbent literature
for this class of planning algorithms. First, the asymptotic
optimality of PRM\textsuperscript{*} and RRT\textsuperscript{*} \citep{karaman2011sampling} transfers
to their kinodynamic instantiation when optimal steering between
states is available \citep{perez2012lqrrrt, webb2013kinodynamic,
schmerling2015drift}. Further, FMT\textsuperscript{*} \citep{janson2015fmt} provides
a non-asymptotic analysis and derives a convergence rate to
optimality alongside a quantified probability of a near-optimal
path at finite sample count. This analysis is extended to a
deterministic set of guarantees in \citet{janson2018deterministic}.
Optimal point-to-point steering, however, is rarely available for
nonlinear systems, as the required BVP can be challenging to
solve.

\begin{figure}[t!]
  \centering
  \includegraphics[width=\columnwidth]{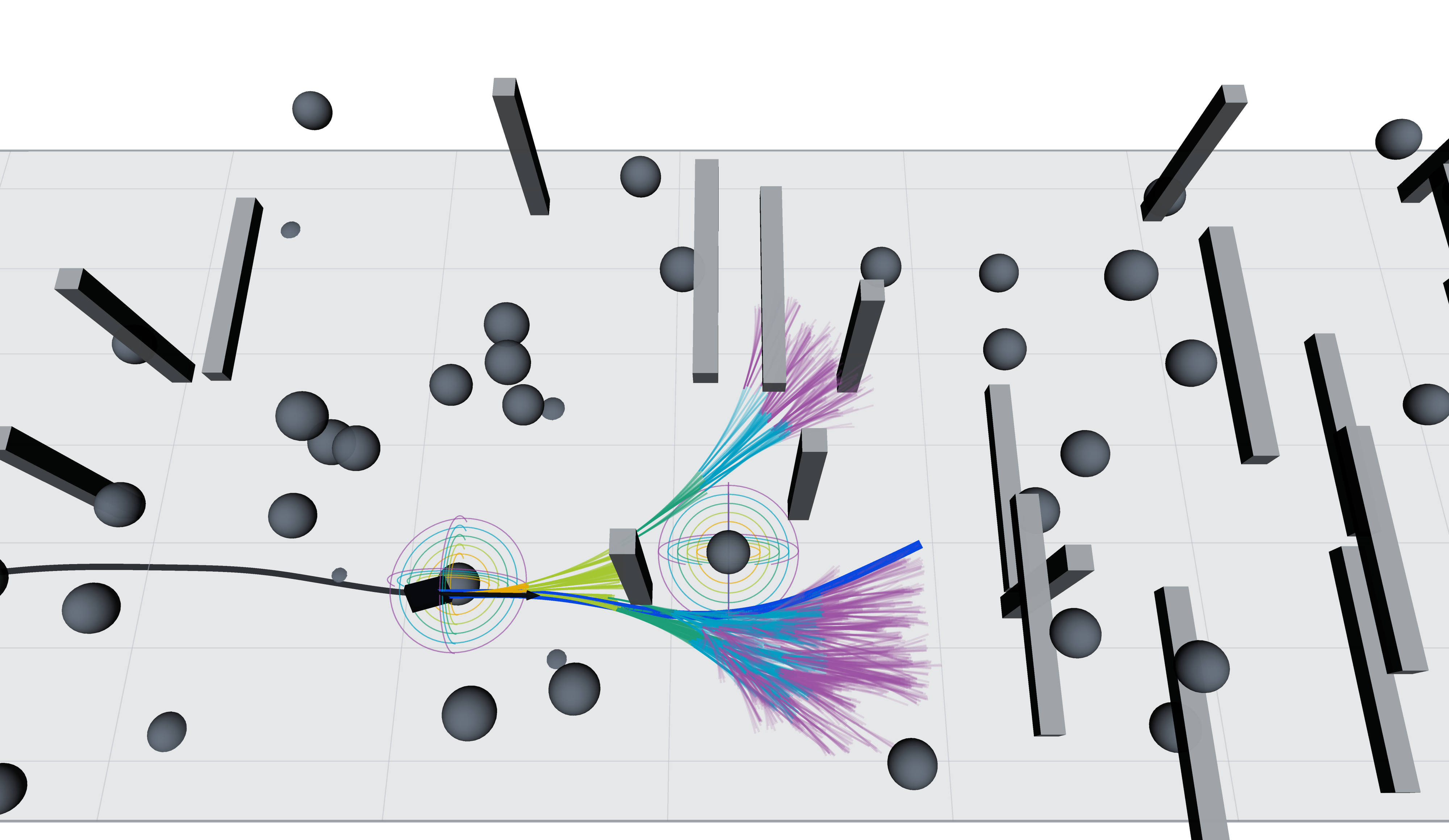}
  \caption{Dispersive forward tree search for online navigation
  in dynamic environments. The forward tree of dispersive
  commands is grown subject to safety predicates from the static
  and moving obstacles. The colored segmentation records tree
  depth, and the concentric rings mark the worst-case reachable
  sets of nearby moving obstacles at the matching depths.}
  \label{fig:teaser}
\end{figure}

A second approach avoids the need for steering by sampling in the
control space and numerically propagating the dynamics under the
sampled command. While asymptotic guarantees of completeness and
optimality are available for this class of algorithms
\citep{lavalle2001rrt, kleinbort2019rrt, li2016asymptotically,
paden2016glc}, a non-asymptotic analysis, including convergence
rates and finite-sample near-optimality guarantees, remains absent.
The central difficulty is that finite-sample guarantees require
quantifying how well a finite collection of sampled commands
covers the resulting space of dynamically feasible trajectories.

In this paper, we address this difficulty for the large family of
nonlinear systems that are differentially flat
\citep{murray1995catalog}. Differential flatness provides a
finite-dimensional representation that allows us to construct
locally dispersive command sets, whose rollouts cover a
reference trajectory class at a specified distance, without
requiring point-to-point steering.
We show that the resulting forward tree provides a quantified
cover of a trajectory class executable by an eroded control
authority. This yields deterministic near-optimality bounds,
measured against the optimal cost over the class of
smooth trajectories that are robust to the problem constraints
with a quantified margin, for all sufficiently regular cost
functionals. Such coverage, however,
requires tree sizes exponential in the problem horizon, rendering
direct deployment intractable. For a permissive class of cost
functionals that aggregate recursively in time, together with
temporally composable constraints, we therefore introduce a
pruning procedure that retains near-optimality while reducing the
required tree size, and hence the deployment computation, to a
polynomial in the problem horizon. The constructed locally
dispersive command sets, combined with the dominance pruning
procedure, form our algorithm, Dispersive Forward Tree search
(DFT\textsuperscript{*}).

In summary, this paper makes the following
contributions:

\noindent $\bullet$~\textbf{Finite-sample theory.} Deterministic,
finite-sample near-optimality guarantees for kinodynamic search
under forward propagation on differentially flat nonlinear
systems.

\noindent $\bullet$~\textbf{Dispersive sampling.} Locally
dispersive command samplers with cardinality bounds and an
explicit construction for control-affine systems, alongside
efficient designs for the unicycle, the trailer car, and the
quadrotor.

\noindent $\bullet$~\textbf{Algorithm and implementation.} DFT\textsuperscript{*},
a massively parallelized implementation of our steering-free
search, with a dominance pruning procedure to control the tree
size in the problem horizon, and an
online adaptation, WWDFT\textsuperscript{*}, for real-time planning in dynamic
environments.

\noindent $\bullet$~\textbf{Evaluation.} Offline evaluation on
the Dynobench instances against state-of-the-art kinodynamic
planners to validate dominant near-optimality, and real-time
evaluation of WWDFT\textsuperscript{*} in dynamic environments to validate
safety-assured deployment at embedded-tier compute.

The remainder of this paper is organized as follows.
Section~\ref{sec:relwork} reviews related literature in the
field. Sections~\ref{sec:prelim}
and~\ref{sec:setting} recall preliminaries and develop the
planning problem alongside its flat reformulation.
Section~\ref{sec:treesearch} develops our theory of dispersive
tree search, and Section~\ref{sec:experiments} presents the DFT\textsuperscript{*}
algorithm, its online adoption WWDFT\textsuperscript{*}, and the full evaluation
suite. All proofs are delegated to the Appendix.

\section{Related Work}
\label{sec:relwork}

In this section, we review the principal approaches to
kinodynamic motion planning alongside their theoretical
properties and situate our contribution against this landscape.

\begin{table*}[t]
\centering
\caption{Convergence guarantees in sampling-based kinodynamic
planning.
\emph{Notation:}
$c_{n}$ --- returned cost at $n$ samples;
$c^{*}$ --- optimal cost;
$c^{*}_{\delta}$ --- optimum over $\delta$-robust solutions;
$c^{*}_{\varepsilon}$ --- optimum under $\varepsilon$-eroded
control authority;
$D_{n}$ --- sample dispersion;
$\delta_{A}(n)$ --- command-set dispersion of an $n$-node tree.
Remaining symbols are schematic.}
\label{tab:taxonomy}
\footnotesize
\renewcommand{\arraystretch}{1.15}
\setlength{\tabcolsep}{5pt}
\newcommand{\vlow}[1]{\raisebox{0.95ex}{#1}}%
\begin{tabular}{|l|l l l|}
\hline
Class & Planning algorithm & Mechanism & Convergence property \\
\hline
\multirow{3}{*}{\shortstack[l]{Probabilistic,\\asymptotic}}
& \shortstack[l]{\rule{0pt}{2.6ex}RRT\textsuperscript{*}/PRM\textsuperscript{*}\\{\scriptsize\citep{perez2012lqrrrt,
webb2013kinodynamic, schmerling2015drift}}} & \vlow{Optimal steering}
& \vlow{$\Pr\bigl(\lim_{n} c_{n} = c^{*}\bigr) = 1$} \\
& \cellcolor{gray!12}\shortstack[l]{\rule{0pt}{2.6ex}Kinodynamic RRT\\{\scriptsize\citep{lavalle2001rrt,
kleinbort2019rrt}}} & \cellcolor{gray!12}\vlow{Forward propagation} &
\cellcolor{gray!12}\vlow{$\Pr\bigl(c_{n} < \infty\bigr) \to 1$} \\
& \cellcolor{gray!12}\shortstack[l]{\rule{0pt}{2.6ex}SST\\{\scriptsize\citep{li2016asymptotically}}} &
\cellcolor{gray!12}\vlow{Forward propagation} &
\cellcolor{gray!12}\vlow{$\Pr\bigl(\lim_{n} c_{n} \le
(1+\alpha(\delta))\, c^{*}_{\delta}\bigr) = 1$} \\
\hline
\multirow{2}{*}{\shortstack[l]{Probabilistic,\\finite-sample}}
& \shortstack[l]{\rule{0pt}{2.6ex}FMT\textsuperscript{*}\\{\scriptsize\citep{janson2015fmt}}} & \vlow{Geometric connections}
& \vlow{$\Pr\bigl(c_{n} \le (1+\epsilon)\, c^{*}\bigr) \ge 1 -
\rho(n)$} \\
& \shortstack[l]{\rule{0pt}{2.6ex}FLASK\\{\scriptsize\citep{king2026akinopdf}}} & \vlow{Flat LQMT steering}
& \vlow{$\Pr\bigl(c_{n} \le (1+\epsilon)\, c^{*}\bigr) \ge 1 -
\rho(n)$} \\
\hline
\multirow{2}{*}{\shortstack[l]{Deterministic,\\asymptotic}}
& \shortstack[l]{\rule{0pt}{2.6ex}gPRM\\{\scriptsize\citep{janson2018deterministic}}} & \vlow{Geometric connections}
& \vlow{$\lim_{n} c_{n} = c^{*}$} \\
& \cellcolor{gray!12}\shortstack[l]{\rule{0pt}{2.6ex}GLC\\{\scriptsize\citep{paden2016glc}}} &
\cellcolor{gray!12}\vlow{Forward propagation} &
\cellcolor{gray!12}\vlow{$\lim_{n} c_{n} = c^{*}$} \\
\hline
\multirow{2}{*}{\shortstack[l]{Deterministic,\\finite-sample}}
& \shortstack[l]{\rule{0pt}{2.6ex}gPRM\\{\scriptsize\citep{janson2018deterministic}}} & \vlow{Geometric connections}
& \vlow{$c_{n} \le \bigl(1 + \tfrac{2 D_{n}}{r_{n} - 2 D_{n}}\bigr)\,
c^{*}_{\delta}$} \\
& \cellcolor{gray!12}\shortstack[l]{\rule{0pt}{2.6ex}\textbf{DFT\textsuperscript{*}}\\{\scriptsize(This work)}} &
\cellcolor{gray!12}\vlow{Forward propagation} &
\cellcolor{gray!12}\vlow{$c_{n} \le c^{*}_{\varepsilon} +
O(\delta_{A}(n))$} \\
\hline
\end{tabular}
\end{table*}

\subsection{Convergence Properties of Sampling-based Motion
Planning}
\label{sec:relwork_convergence}

Sampling-based motion planners divide broadly into two
categories. Planners in the first category sample states in the
configuration space and connect them by solving a two-point
boundary value problem (BVP), while those in the other sample
commands in the control space and numerically propagate the
dynamics forward. In either case, the convergence properties of
sampling-based planners have guided progress in this field.
In Table~\ref{tab:taxonomy} we organize this progression for
the two categories of sampling-based planners based on
convergence to optimality being probabilistic or deterministic
and whether finite-sample performance guarantees are available.

\subsubsection{Configuration-space sampling and BVP steering}

When optimal state-to-state steering is available,
sampling-based kinodynamic planners inherit their convergence
guarantees from geometric planning. \citet{karaman2011sampling} establish the asymptotic optimality
of PRM\textsuperscript{*} and RRT\textsuperscript{*} for geometric problems, which transfers to
kinodynamic planning for linear systems under quadratic cost
\citep{perez2012lqrrrt, webb2013kinodynamic} and linear systems
with drift \citep{schmerling2015drift}. The geometric statements of \citet{karaman2011sampling} are refined by FMT\textsuperscript{*}
\citep{janson2015fmt}, and the probability of attaining a
near-optimal solution at a finite sample count is quantified.
This refinement is subsequently derandomized in \citet{janson2018deterministic}, where dispersive sampling is
leveraged to obtain deterministic finite-sample guarantees
alongside asymptotic optimality. The finite-sample statements of \citet{janson2015fmt} rest on
probabilistic exhaustivity, the property that every path of
sufficient clearance is traced, with a given probability, by a
path within a surrounding spatial tube. A recent contribution,
FLASK \citep{king2026akinopdf}, carries probabilistic
exhaustivity to samples in the flat output space and develops
fast BVP connections between flat state samples for
linear-quadratic minimum-time (LQMT) costs.

\subsubsection{Control-space sampling and forward propagation}

The steering requirement of solving a state-to-state boundary
value problem can be evaded altogether by propagating sampled
control commands forward through the system's dynamics.
Available convergence properties of planners in the propagation
arm however remain relatively weak. Probabilistic completeness
of the original kinodynamic RRT \citep{lavalle2001rrt} was
established only recently \citep{kleinbort2019rrt}, and
guarantees the discovery of a suboptimal
solution in the asymptotic limit. Stable Sparse RRT (SST) \citep{li2016asymptotically}
modifies the sampling mechanism of the kinodynamic RRT and
introduces a witness pruning procedure to obtain probabilistic
asymptotic convergence to a $\delta$-robust optimum. GLC \citep{paden2016glc} searches a deterministic grid of
control samples, dominance-prunes on state-space cells, and
guarantees convergence to the optimal cost as the grid refines.
However, the sampling resolution demanded by the analysis grows
exponentially with the planning horizon.

\subsection{Graph Search over Motion Primitives}
\label{sec:relwork_primitives}

When the requisite BVP for state-to-state steering is expensive
to solve, as is often the case for nonlinear systems, graph
search over a precomputed library of motion primitives remains
an effective approach \citep{pivtoraiko2009lattice,
frazzoli2005maneuver, sakcak2019lattice, ortizharo2024idbastar}.
\citet{sakcak2019lattice} propose regeneration of
boundary-value-optimal primitives at each resolution of a
lattice in the state space and prove a resolution-dependent
near-optimality of the cost. iDb-A\textsuperscript{*}
\citep{ortizharo2024idbastar} relaxes primitive matching
requirements and allows bounded discontinuities in the solution
trajectories, which are repaired afterwards by trajectory
optimization. This relaxation lets iDb-A\textsuperscript{*} solve challenging
kinodynamic problems across the Dynobench suite
\citep{ortizharo2024idbastar}, including aerial and articulated
platforms, at state-of-the-art solution quality.

\subsection{Optimization and Entropy Methods}
\label{sec:relwork_opt}

Nonlinear trajectory optimization frameworks such as TrajOpt
\citep{schulman2014trajopt}, GuSTO \citep{bonalli2019gusto}, and
Crocoddyl \citep{mastalli2020crocoddyl} compute locally optimal
trajectories under the full system dynamics, with constraints
imposed explicitly or via penalties in the cost. These methods can find
high-quality solutions but remain expensive as well as sensitive
to initialization \citep{ortizharo2024idbastar}. For online
planning, reduced-order models admit lighter formulations
optimizing polynomial splines in real time, most prominently
deployed for ground and aerial navigation \citep{wang2022minco,
zhou2021ego, kondo2026mighty}.
Sampling-based optimizers such as MPPI \citep{williams2017mppi}
and the cross-entropy method \citep{kobilarov2012cem} replace
gradients with importance-weighted rollouts of sampled controls.
Every problem requirement is superposed into the rollout cost,
and hard constraints such as collision avoidance are imposed as
soft penalties over the same.

\subsection{Parallel Computation for Fast Planning}
\label{sec:relwork_parallel}

Modern sampling-based planners have increasingly benefitted
from hardware parallelism. Parallel processing on GPUs is
particularly suitable for parallel propagation of control
samples \citep{perrault2025kinopax} as well as batched frontier
expansions of steering-based planners \citep{ichter2017gmt}.
SIMD cores on consumer CPUs have also been shown to sharply
accelerate forward kinematics and collision checks by
vectorization over batched robot configurations
\citep{thomason2024vamp, king2026akinopdf}. In this work, we accordingly exploit the data-parallel structure
of depth-synchronized control propagation and dominance pruning
and implement DFT\textsuperscript{*} with fused GPU kernels that sample,
propagate, and grade candidate trajectories.


\section{Preliminaries}
\label{sec:prelim}

In this section, we recollect some preliminary definitions and
results utilized by our theoretical development later.

\subsection{Dynamic Feedback Linearization and Differential
Flatness}
\label{sec:prelim-flat}

Consider a nonlinear control system on the state space
$\X \subseteq \R^{n}$, with the compact input set
$U \subset \R^{m}$ and the smooth dynamics
\begin{equation}
\label{eq:sys}
  \dot{x} = f(x, u) .
\end{equation}

\begin{definition}[Differential flatness
\citep{fliess1995flatness}]
\label{def:flat}
The system \eqref{eq:sys} is \emph{differentially flat} if there
exist a \emph{flat output}
\[
  z = h\bigl( x, u, \dot{u}, \dots, u^{(p)} \bigr) \in \R^{d},
  \qquad d = m,
\]
with differentially independent
components,\footnote{No non-trivial differential relation
$\Phi\bigl( z, \dot{z}, \dots, z^{(k)} \bigr) = 0$ --- e.g.\
$\dot{z}_{1} = z_{2}$ --- holds along all trajectories.} and
smooth maps $\Psi_{x}, \Psi_{u}$ of $z$ and finitely many of its
derivatives such that, along every trajectory of the system,
\[
  x = \Psi_{x}\bigl( z, \dot{z}, \dots \bigr),
  \qquad
  u = \Psi_{u}\bigl( z, \dot{z}, \dots \bigr) .
\]
\end{definition}

Differential flatness is equivalent to dynamic feedback
linearizability in the following sense, which vastly simplifies
the analytic structure of flat systems. A \emph{dynamic feedback}
\[
  \dot{\zeta} = a(x, \zeta, w),
  \qquad
  u = \kappa(x, \zeta, w),
\]
with compensator state $\zeta \in \R^{q}$ and new input
$w \in \R^{m}$, is \emph{endogenous} if
\[
  \zeta \;=\; \sigma\bigl( x, u, \dot{u}, \dots, u^{(s)} \bigr)
\]
along every closed-loop trajectory.

\begin{theorem}[Fliess--L\'evine--Martin--Rouchon
\citep{fliess1995flatness, fliess1999liebacklund}]
\label{thm:flmr}
The system \eqref{eq:sys} is differentially flat in a neighborhood
of a regular point\footnote{A point at which the flatness maps are
defined and smooth. The singular locus --- zero net thrust for the
quadrotor, zero velocity for the unicycle --- is excluded.} if and
only if there exist indices $R = (r_{1}, \dots, r_{d})$, an
endogenous dynamic feedback, and a local diffeomorphism
\[
  (x, \zeta)
  \;\mapsto\;
  \bigl( z_{i}, \dot{z}_{i}, \dots, z_{i}^{(r_{i}-1)}
  \bigr)_{i=1}^{d}
\]
under which the closed-loop system takes the Brunovsk\'y form
\[
  z_{i}^{(r_{i})} \;=\; w_{i},
  \qquad i = 1, \dots, d .
\]
\end{theorem}

The variable $w = (w_{1}, \dots, w_{d})$ is called the \emph{flat
input}, the tuple
$\bigl( z_{i}, \dot{z}_{i}, \dots, z_{i}^{(r_{i}-1)}
\bigr)_{i=1}^{d}$ is the \emph{flat state} of dimension
$|R| = \sum_{i} r_{i}$, and the integer $d$ is the \emph{flat
dimension}.

Following Definition~\ref{def:flat}, the flat output $z$ may
additionally depend on the input and its derivatives. Such a
dependence can be absorbed into the compensator state following
Theorem~\ref{thm:flmr}, as the following remark records.

\begin{remark}[State-determined flat output]
\label{rem:state-flat}
Consider the extended system of Theorem~\ref{thm:flmr} with state
$(x, \zeta)$ and input $w$. Each component $z_{i}$ and each
derivative $z_{i}^{(k)}$, $0 \le k \le r_{i} - 1$, is a
coordinate of the diffeomorphism
\[
  (x, \zeta)
  \;\mapsto\;
  \bigl( z_{i}, \dot{z}_{i}, \dots, z_{i}^{(r_{i}-1)}
  \bigr)_{i=1}^{d} .
\]
There accordingly exists a map $\tilde{h}$ such that
\[
  z(t) \;=\; \tilde{h}\bigl( x(t), \zeta(t) \bigr)
\]
along every trajectory of the extended system. We work with the
extended state whenever necessary and write $z = h(x)$ without
loss of generality.
\end{remark}

\subsection{Dispersion and Metric Entropy}
\label{sec:prelim-entropy}

Let $(\mathcal{M}, \rho)$ be a metric space and let
$\mathcal{Z} \subseteq \mathcal{M}$ be totally bounded.

\begin{definition}[Dispersion \citep{niederreiter1992random}]
\label{def:disp-classical}
The \emph{dispersion} of a finite set $P \subset \mathcal{M}$
over $\mathcal{Z}$ is
\[
  \mathrm{disp}(P;\, \mathcal{Z})
  \;=\;
  \sup_{\zeta \in \mathcal{Z}} \, \min_{p \in P} \,
  \rho(\zeta, p) ,
\]
the radius of the largest ball centered in $\mathcal{Z}$ that
contains no point of $P$.
\end{definition}

\begin{example}
\label{ex:disp-interval}
Let $\mathcal{Z} = \mathcal{M} = [0, 1]$ with the metric
$\rho(x, y) = |x - y|$.
\begin{itemize}
\item[(i)] The midpoints of the three equal subintervals of
$[0, 1]$,
\[
  P_{3} \;=\; \Bigl\{\, \frac{1}{6},\; \frac{1}{2},\;
  \frac{5}{6} \,\Bigr\} ,
\]
satisfy $\mathrm{disp}(P_{3};\, [0, 1]) = 1/6$.
\item[(ii)] In general, the set of the midpoints of the $n$
equal subintervals of $[0, 1]$,
\[
  P_{n} \;=\; \Bigl\{\, \frac{2k - 1}{2n} \;:\; k = 1, \dots, n
  \,\Bigr\} ,
\]
attains the least dispersion $1/(2n)$ among sets of $n$ points.
\end{itemize}
\end{example}

\begin{definition}[Covering and packing numbers]
\label{def:covpack}
Let $s > 0$.
\begin{itemize}
\item[(i)] The \emph{covering number} of $\mathcal{Z}$ is the
size of its smallest subset with dispersion at most $s$ over
$\mathcal{Z}$,
\[
  N_{s}(\mathcal{Z})
  \;=\;
  \min \bigl\{\, |Q| \;:\;
  Q \subseteq \mathcal{Z} ,\;
  \mathrm{disp}(Q;\, \mathcal{Z}) \le s
  \,\bigr\} .
\]
\item[(ii)] The \emph{packing number} of $\mathcal{Z}$ is the
size of its largest subset with members pairwise more than $s$
apart,
\[
  P_{s}(\mathcal{Z})
  \;=\;
  \max \Bigl\{\, |Q| \;:\;
  Q \subseteq \mathcal{Z} ,\;
  \min_{q \ne q' \in Q} \rho(q, q') > s
  \,\Bigr\} .
\]
\end{itemize}
\end{definition}

\begin{example}
\label{ex:covpack-interval}
Let $\mathcal{Z} = [0, 1]$ as in
Example~\ref{ex:disp-interval}.
\begin{itemize}
\item[(i)] The midpoint sets $P_{n}$ of
Example~\ref{ex:disp-interval} are the smallest covers, so
\[
  N_{s}\bigl( [0, 1] \bigr)
  \;=\;
  \Bigl\lceil \frac{1}{2s} \Bigr\rceil .
\]
\item[(ii)] The $\lceil 1/s \rceil$ evenly spaced points form a
largest set with pairwise separation exceeding $s$, so
\[
  P_{s}\bigl( [0, 1] \bigr)
  \;=\;
  \Bigl\lceil \frac{1}{s} \Bigr\rceil .
\]
\end{itemize}
\end{example}

At scale $s$, the covering and the packing numbers obey the
relationship
\[
  P_{2s}(\mathcal{Z})
  \;\le\;
  N_{s}(\mathcal{Z})
  \;\le\;
  P_{s}(\mathcal{Z}) .
\]
The quantity $\log N_{s}(\mathcal{Z})$ is called the
\emph{metric entropy} of $\mathcal{Z}$ at scale $s$
\citep{kolmogorov1959entropy}. In the following result, we recall
the classical entropy estimate for Lipschitz function classes,
which our complexity results build upon.

\begin{theorem}[Entropy of Lipschitz classes
\citep{kolmogorov1959entropy}]
\label{fact:lipentropy}
For $M > 0$, let $\mathrm{Lip}_{M}$ be the set of $M$-Lipschitz
continuous functions on $[0, 1]$ that vanish at zero,
\[
  \begin{aligned}
  \mathrm{Lip}_{M}
  \;=\;
  \bigl\{\, f \in C([0, 1]) \;:\;& f(0) = 0, \\
  &|f(\tau) - f(\tau')| \le M\, |\tau - \tau'| \\
  &\forall\, \tau, \tau' \in [0, 1] \,\bigr\} ,
  \end{aligned}
\]
and let it carry the supremum metric
\[
  \rho(f, h)
  \;=\;
  \sup_{\tau \in [0, 1]} \bigl| f(\tau) - h(\tau) \bigr| .
\]
There exists a universal constant
$c_{0}$ such that
\[
  \log N_{s}\bigl( \mathrm{Lip}_{M} \bigr)
  \;\le\;
  c_{0}\, \frac{M}{s}
  \qquad \text{for every } 0 < s \le M .
\]
\end{theorem}

\section{Problem Formulation}
\label{sec:setting}

Consider a platform governed by the nonlinear control system
\eqref{eq:sys}, with state $x \in \X \subset \R^{n}$ and control
input $u \in U \subset \R^{m}$. Our theoretical development
applies to a generic finite-horizon optimal control problem
(OCP) for such a platform, with the following structure:
\begin{equation}
\label{eq:ocp}
\begin{aligned}
  \min_{u(\cdot)} \quad
  & J\bigl( x(\cdot) \bigr)
    \\[0.2em]
  \text{s.t.} \quad
  & \dot{x}(t) = f\bigl(x(t), u(t)\bigr),
    \qquad t \in [0, T], \\[0.2em]
  & g\bigl( x(\cdot) \bigr)
    \ge 0 \qquad \forall\, g \in \preds, \\[0.2em]
  & x(0) = x_{\mathrm{init}},
\end{aligned}
\end{equation}
where the cost functional $J : C([0,T]; \X) \to \R$ and the
predicates $g : C([0,T]; \X) \to \R$, $g \in \preds$, read the
state trajectory. This structure subsumes kinodynamic motion
planning instances, which are of typical interest in this work.

\begin{example}[Kinodynamic planning]
\label{ex:kino-instance}
Given a time-varying free space
$\X_{\mathrm{free}}(t) \subseteq \X$, a goal
region $\X_{\mathrm{goal}} \subseteq \X$, and an initial state
$x_{\mathrm{init}}$, the kinodynamic motion planning problem
\begin{equation}
\label{eq:kino}
\begin{aligned}
  \min_{T_{f},\, u(\cdot)} \quad
  & T_{f}
    \\[0.2em]
  \text{s.t.} \quad
  & \dot{x}(t) = f\bigl(x(t), u(t)\bigr),
    \qquad t \in [0, T_{f}], \\[0.2em]
  & x(t) \in \X_{\mathrm{free}}(t),
    \qquad t \in [0, T_{f}], \\[0.2em]
  & x(T_{f}) \in \X_{\mathrm{goal}},
    \qquad x(0) = x_{\mathrm{init}}
\end{aligned}
\end{equation}
asks for a control signal whose trajectory from the initial
state remains in the free space and arrives in the goal region
in minimum time. Let
$\operatorname{sd}( x,\, \X \setminus \X_{\mathrm{free}}(t) )$
denote the signed distance of the state $x$ to the obstacle set
$\X \setminus \X_{\mathrm{free}}(t)$, and
let $\operatorname{dist}( x,\, \X_{\mathrm{goal}} )$ denote the
distance of $x$ to the goal region. Then, setting the predicates
and the cost as
\[
  \begin{aligned}
  g_{\mathrm{free}}\bigl( x(\cdot) \bigr)
  &\;=\;
  \min_{t \in [0, T]}
  \operatorname{sd}\bigl( x(t),\,
  \X \setminus \X_{\mathrm{free}}(t) \bigr) , \\
  g_{\mathrm{goal}}\bigl( x(\cdot) \bigr)
  &\;=\;
  - \min_{t \in [0, T]}
  \operatorname{dist}\bigl( x(t),\, \X_{\mathrm{goal}} \bigr) ,
  \\
  J\bigl( x(\cdot) \bigr)
  &\;=\;
  \inf \bigl\{\, t \;:\; x(t) \in \X_{\mathrm{goal}} \,\bigr\}
  \end{aligned}
\]
recovers \eqref{eq:kino} within the structure of \eqref{eq:ocp}
on a horizon $T$ long enough to contain candidate solutions.
\end{example}

\subsection{Assumptions}
\label{sec:assumptions}

The theoretical development of this paper rests on a generic set
of assumptions enumerated below.

(i) \textbf{Regularity:} The dynamics $f$ is regular enough that,
starting from any state $x \in \X$, each admissible control
$u : [0, T] \to U$ generates a unique trajectory of
\eqref{eq:sys}.

(ii) \textbf{Differential flatness with indices $R$:} The system is
differentially flat in the sense of Theorem~\ref{thm:flmr}, with
flat output $z = h(x) \in \R^{d}$ following
Remark~\ref{rem:state-flat} and indices
$R = (r_{1}, \dots, r_{d})$,
on a fixed open domain of regular points. Denote
\[
  \begin{aligned}
  \jet z
  &\;:=\;
  \bigl( z_{i},\, \dot{z}_{i},\, \dots,\, z_{i}^{(r_{i}-1)}
  \bigr)_{i=1}^{d}
  \;\in\; \R^{|R|}, \\
  z^{(R)}
  &\;:=\;
  \bigl( z_{i}^{(r_{i})} \bigr)_{i=1}^{d}
  \;\in\; \R^{d}, \\
  |R|
  &\;:=\;
  \textstyle\sum_{i} r_{i},
  \qquad
  \bar{r} := \textstyle\max_{i} r_{i},
  \end{aligned}
\]
and let $|\cdot|_{\infty}$ denote the entrywise maximum norm. The flatness maps of Definition~\ref{def:flat} accordingly take
the continuously differentiable form
\[
  x(t) = \Psi_{x}\bigl( \jet z (t) \bigr),
  \quad
  u(t) = \Psi_{u}\bigl( \jet z (t),\, z^{(R)}(t) \bigr)
\]
along every trajectory.

Conversely, let $z(\cdot)$ be a curve with components $z_{i}$ in
the Sobolev space $W^{r_{i},\infty}([0,\tau]; \R)$ of scalar
functions with essentially bounded derivatives up to order
$r_{i}$. Suppose that, for almost every $t \in [0, \tau]$,
\[
  \Psi_{u}\bigl( \jet z (t),\, z^{(R)}(t) \bigr)
  \;\in\; U .
\]
Then $z = h\bigl( x(\cdot) \bigr)$ for a trajectory $x(\cdot)$ of
the system, realized by the control
$u(t) = \Psi_{u}\bigl( \jet z (t),\, z^{(R)}(t) \bigr)$ from the
state $x(0) = \Psi_{x}\bigl( \jet z (0) \bigr)$.

(iii) \textbf{Compact input set:} The input set $U \subset \R^{m}$
is compact.

(iv) \textbf{Jet observability:} There is a map
$\Xi = (\Xi_{i,k})_{1 \le i \le d,\; 0 \le k \le r_{i}-1}$ such
that, along every trajectory of the system,
\[
  \jet z (t)
  \;=\;
  \Xi\bigl( x(t) \bigr) .
\]
Following Remark~\ref{rem:state-flat}, an extension of the state
renders the jet observable whenever required.

\subsection{Flat Reformulation}
\label{sec:flatform}

It is convenient for our development downstream that the
OCP~\eqref{eq:ocp} be recast into a flat reformulation, which we
attend to in this section.

Let $\tra : [0, \tau] \to \X$ be a \emph{trajectory}, a
continuous curve in the state space, with the \emph{output trace}
\[
  z_{\tra} \;=\; h \circ \tra \;\in\; C([0,\tau]; \R^{d}) .
\]
By jet observability, the output jet along a trajectory is the
function $\Xi$ of its state, and we write
\[
  \jet z_{\tra}(t) \;=\; \Xi\bigl( \tra(t) \bigr),
  \qquad
  z^{(k)}_{\tra,i}(t) \;=\; \Xi_{i,k}\bigl( \tra(t) \bigr) .
\]

By flatness (ii) and jet observability (iv), the state along a
trajectory is a function of the output jet,
\[
  x(t) \;=\; \Psi_{x}\bigl( \jet z (t) \bigr),
  \qquad
  \jet z (t) \;=\; \Xi\bigl( x(t) \bigr),
\]
so the cost and the predicates of \eqref{eq:ocp} induce functionals
of the output jet curve, which we continue to denote by $J$ and
$g$,
\[
  \begin{aligned}
  J\bigl( x(\cdot) \bigr)
  &\;=\;
  J\bigl( \Psi_{x}( \jet z (\cdot) ) \bigr)
  \;=:\;
  J\bigl( \jet z (\cdot) \bigr), \\
  g\bigl( x(\cdot) \bigr)
  &\;=:\;
  g\bigl( \jet z (\cdot) \bigr),
  \quad g \in \preds .
  \end{aligned}
\]
Moreover, by the flatness converse in (ii), the trajectories of
admissible controls from $x_{\mathrm{init}}$ are exactly the state
realizations of the output curves $z(\cdot)$,
$z_{i} \in W^{r_{i},\infty}([0,T]; \R)$, whose induced control lies
in $U$ almost everywhere and that start at the initial jet.

The problem \eqref{eq:ocp} is therefore equivalent to the
output-space problem
\begin{equation}
\label{eq:ocp-flat}
\begin{aligned}
  \min_{z(\cdot)} \quad
  & J\bigl( \jet z (\cdot) \bigr)
    \\[0.2em]
  \text{s.t.} \quad
  & \Psi_{u}\bigl( \jet z (t),\, z^{(R)}(t) \bigr) \;\in\; U
    \ \ \text{a.e.\ on } [0, T], \\[0.2em]
  & g\bigl( \jet z (\cdot) \bigr)
    \ge 0 \qquad \forall\, g \in \preds, \\[0.2em]
  & \jet z (0) = \Xi\bigl( x_{\mathrm{init}} \bigr) .
\end{aligned}
\end{equation}
Our theory below addresses \eqref{eq:ocp} through this flat
reformulation.

We close the section by recording a notational convention used
throughout the development.

\begin{remark}[Evaluation on output traces]
\label{rem:trace-eval}
Along a trajectory $\tra$, each component $z_{\tra,i}$ of the
output trace is $r_{i}$-times differentiable, and its derivatives
assemble the jet curve
\[
  \jet z_{\tra}(t)
  \;=\;
  \bigl( z_{\tra,i}(t),\, \dot{z}_{\tra,i}(t),\, \dots,\,
  z^{(r_{i}-1)}_{\tra,i}(t) \bigr)_{i=1}^{d} .
\]
The output trace therefore determines the arguments of the cost
and the predicates, and we write
\[
  J\bigl( z_{\tra} \bigr),
  \qquad
  g\bigl( z_{\tra} \bigr), \quad g \in \preds ,
\]
for their evaluations along the trajectory.
\end{remark}

\section{Dispersive Tree Search: Coverage and Complexity}
\label{sec:treesearch}

This section contains all the major theoretical results of our
article. We begin with an overview of the development before
presenting our results through the rest of the section.

\subsection{Overview of Results}
\label{sec:overview}

In this subsection, we offer an overarching summary of the
theoretical results that follow through the rest of this section.

The central object of our investigation is a \emph{dispersive
forward tree}. Given a finite set $\A$ of flat-input
\emph{commands} of unit duration, the dispersive tree
$\mathcal{T}$ of depth $K$ with its root node at the state
$x_{\mathrm{init}}$ is grown by recursively applying every
command of $\A$ at each node. The command set $\A$ is
\emph{dispersive} in the sense that each local trajectory
starting from a state $x$ falls within a deviation $\delta$
of the finite set of rollouts grown by $\A$.

Our development proceeds to establish \emph{coverage} results
for dispersive forward trees. In particular, we certify that a
smooth trajectory class generated by interior control commands is
covered by the forward tree at a deviation $O(\delta)$, uniformly
in the tree depth.

Following this coverage, dispersive trees are then able to
attain closure to the optimal cost of trajectories in this
class that remain robust to the problem predicates --- the
\emph{$M$-smooth, $\varepsilon$-eroded, $\gamma$-robust optimum}
\[
  J^{*}_{M,\gamma,\varepsilon}
  \;=\;
  \inf_{u}\; J(z_{u})
  \;\left|\;
  \begin{array}{l}
  u \ \text{$\varepsilon$-interior to } U , \\
  g(z_{u}) \ge \gamma \quad \forall\, g \in \preds , \\
  z_{u} \ \text{$M$-smooth} ,
  \end{array}
  \right.
\]
where $z_{u}$ denotes the output trace of the trajectory under
the control $u$ from $x_{\mathrm{init}}$.

Our results permit the control erosion $\varepsilon$
and the predicate robustness $\gamma$ to scale with the local
dispersion $\delta_{\A}$ of the command set. Attaining a finer
dispersion, however, requires exponentially many commands, and
the complete dispersive tree grows intractably large over long
horizons. Indeed, using packing number estimates, we show that a tree size
exponential in the horizon is necessary to maintain dispersive
coverage over the trajectory class.

Fortunately, we show that for a large class of costs and
predicates encompassing common planning objectives, it is
possible to deploy spatio-temporal pruning at a desired scale
$s$, and prune down the required tree size to a polynomial
complexity $O\bigl( K^{d+1} ( M / s )^{|R|} \bigr)$, and with it
the deployment compute to
$O\bigl( |\A|\, K^{d+1} ( M / s )^{|R|} \bigr)$, all while
maintaining closure to the optimal trajectories of the class.

The remainder of this section develops these results
in order. In Section~\ref{sec:sampling}, we formalize dispersive
local sampling and exhibit the complexity and generic
constructions of dispersive command sets. In
Section~\ref{sec:correction}, we prove coverage of the smooth
interior trajectory class and closure of the search cost to the
optimum $J^{*}_{M,\gamma,\varepsilon}$.
Section~\ref{sec:sparse} establishes the necessity of an
exponential tree, and in Section~\ref{sec:genpruning}, we
develop the dominance pruning procedure that reduces the
required size to a polynomial in the horizon.

\subsection{Dispersive Local Sampling}
\label{sec:sampling}

Our theory throughout operates under the following metric defined
on the trajectories of the flat system.

\begin{definition}[Sobolev metric]
\label{def:dist}
For state trajectories $\tra, \tra'$ both defined on an interval
$I$, the \emph{Sobolev metric} reads their output jets,
$d\bigl(\tra, \tra'; I\bigr) =$
\[
  \sup_{t \in I}\; \max_{1 \le i \le d}\;
  \max_{0 \le k \le r_{i}-1}\;
  \bigl| z^{(k)}_{\tra,i}(t) - z^{(k)}_{\tra',i}(t) \bigr| ;
\]
when $I$ is omitted it is the common domain of the two trajectories.
\end{definition}

Further, we operate within smoothness envelopes on the system's
trajectories, handling the $M$-smooth trajectory classes defined
below throughout.

\begin{definition}[$M$-smooth trajectory class]
\label{def:msmooth}
For $M > 0$, a trajectory $\tra$ of the system is
\emph{$M$-smooth} if its jet and flat input are bounded by $M$,
i.e.,
\begin{equation}
\label{eq:envelope}
  \bigl| z^{(k)}_{\tra,i}(t) \bigr| \;\le\; M,
  \qquad 1 \le k \le r_{i}, \quad 1 \le i \le d ,
\end{equation}
with $w := z^{(R)}$ as the flat input. We denote the class of
$M$-smooth trajectories by $\mathcal{E}_{M}$.
\end{definition}

In this subsection, we are concerned with the existence,
complexity, and construction of a dispersive command set. A
command set generates a finite set of local trajectories of the
system starting from a given state, as the defining sequence
below clarifies.

\begin{definition}
\label{def:accepted}
(i) A \emph{command} $a$ is a measurable flat input
$w_{a} : [0,1] \to \R^{d}$. Following the integrator chain of
Theorem~\ref{thm:flmr}, the command $a$ at a state $x$ produces
the flat trajectory
\begin{equation}
\label{eq:rollout}
  \begin{aligned}[t]
  &z^{x}_{a,i}(t)
  \\
  &\;=\;
  \sum_{k=0}^{r_{i}-1} \Xi_{i,k}(x)\, \frac{t^{k}}{k!}
  +
  \int_{0}^{t} \frac{(t - s)^{r_{i}-1}}{(r_{i}-1)!}\,
  w_{a,i}(s)\, \mathrm{d}s
  \\
  &\;=:\;
  \roll_{x}(a)_{i}(t) ,
  \end{aligned}
\end{equation}
where $\Xi(x)$ is the output jet at $x$ from jet observability
(iv); we call $z^{x}_{a}$ the \emph{flat rollout} of $a$ from
$x$, and the control it induces is
\[
  u^{x}_{a}(t)
  \;=\;
  \Psi_{u}\bigl( \jet z^{x}_{a}(t),\, w_{a}(t) \bigr) .
\]
A \emph{command set} $\A$ is a finite set of commands.

(ii) The \emph{accepted set} of $\A$ at $x$ is
\[
  \A(x)
  \;=\;
  \bigl\{\, a \in \A \;:\;
  u^{x}_{a}(t) \in U \ \text{a.e.\ on } [0, 1]
  \,\bigr\} ,
\]
the commands whose induced control is admissible.

(iii) A command sequence $(a_{1}, \dots, a_{K})$ is
\emph{accepted} from $x$ if each $a_{k}$ is accepted at the state
reached by its predecessor.
\end{definition}

Given a state $x$, the accepted commands $a \in \A(x)$ sample
finitely many local trajectories over a time window. We are
interested in the scale at which this finite set covers the class
of all locally admissible trajectories. We work with $M$-smooth
trajectories generated by control inputs sampled in the
$\varepsilon$-interior of the input set $U$, as formalized below.

\begin{definition}[Eroded admissibility]
\label{def:erosion}
For $\varepsilon > 0$, the \emph{$\varepsilon$-eroded input set} is
\[
  U^{\varepsilon}
  \;=\;
  \bigl\{\, u \in U \;:\;
  u + [-\varepsilon, \varepsilon]^{m} \subseteq U \,\bigr\} ,
\]
and an admissible control $u$ is \emph{$\varepsilon$-eroded}, i.e.,
$u \in \mathcal{U}^{\varepsilon}$, if $u(t) \in U^{\varepsilon}$
a.e. Let $\tra^{x}_{u}$ be the trajectory generated by $u$ from a
state $x$. The class
\[
  \mathcal{U}^{\varepsilon}_{M}(x)
  \;:=\;
  \bigl\{\, u \in \mathcal{U}^{\varepsilon} \;:\;
  \tra^{x}_{u} \in \mathcal{E}_{M} \,\bigr\}
\]
defines the set of $\varepsilon$-eroded controls whose trajectory
from $x$ is $M$-smooth.
\end{definition}

The coverage offered by a command set over a local trajectory
class is measured by its dispersion over the class, as defined
below. Recall Definition~\ref{def:disp-classical}, here taken in
the Sobolev metric $d(\cdot, \cdot;\, [0, 1])$ of
Definition~\ref{def:dist}.

\begin{definition}[Dispersive Command Sets]
\label{def:dispersive}
Let $\mathcal{V}$ be a class of admissible controls on $[0, 1]$.
The command set $\A$ is \emph{$\delta$-dispersive over
$\mathcal{V}$ at a state $x \in \X$} if
\[
  \mathrm{disp}\Bigl(
  \bigl\{\, \roll_{x}(a) \;:\; a \in \A(x) \,\bigr\} ;\;
  \bigl\{\, \tra^{x}_{v} \;:\; v \in \mathcal{V} \,\bigr\}
  \Bigr)
  \;\le\; \delta .
\]
\end{definition}

\label{sec:realization}

Our downstream development that certifies coverage and
optimality for the OCP~\eqref{eq:ocp} presupposes locally
dispersive command sets. In our first result below, we record the
existence of such a set at every state and estimate its size via
the metric entropy of a local trajectory class.

\begin{proposition}[Existence and complexity]
\label{prop:floor}
Let $x \in \X$ and $\delta \in (0, M]$. There exists a command
set $\A^{*}$ that is $\delta$-dispersive at $x$ over the
$M$-smooth admissible class
$\bigl\{\, v \;:\; \tra^{x}_{v} \in \mathcal{E}_{M} \,\bigr\}$.
Moreover, its cardinality obeys
\[
  \log |\A^{*}| \;\le\; c_{\star}\, d\, \frac{M}{\delta} ,
\]
where $c_{\star}$ is an independent numerical constant.
\end{proposition}

To execute dispersive forward tree search, a deterministic
procedure that constructs a dispersive command set at every state
is necessary. One such procedure is provided below for
\emph{control-affine} systems. The idea is to sample and roll out
flat inputs dispersively, and reject trajectories that violate
authority requirements. Suppose that
\[
  f(x, u) \;=\; f_{0}(x) + g(x)\, u ,
\]
with $f_{0}$ and $g$ continuous and $g(x)$ of full column rank on
the domain, and that the input set $U$ is convex. Control-affine
dynamics render the input map affine in the flat input.

\begin{lemma}[Affinity of the input map]
\label{lem:affine}
Under the control-affine supposition, the input map $\Psi_{u}$ of
Assumption (ii), Section~\ref{sec:assumptions}, satisfies
\[
  \Psi_{u}\bigl( \jet z,\, w \bigr)
  \;=\;
  \alpha\bigl( \jet z \bigr) + B\bigl( \jet z \bigr)\, w ,
\]
with $\alpha$ and $B$ locally Lipschitz.
\end{lemma}

For $j \in \mathbb{N}$, write
$I^{j}_{\ell} := [\tfrac{\ell-1}{j},\, \tfrac{\ell}{j}]$,
$1 \le \ell \le j$, for the equal subwindows of $[0,1]$. A lattice quantization of the
flat input on such subwindows realizes a deterministic procedure
to construct a dispersive command set, as proven by the theorem
below.

\begin{theorem}[Dispersion realization]
\label{thm:realization}
Let $\delta \in (0, M/2]$, set
$j := \lceil 4 \sqrt{d}\, M/\delta \rceil$ and
$n := \lceil 2 \sqrt{d}\, M/\delta \rceil$, and let
\[
  G \;=\;
  \Bigl\{\, -M + (2k - 1)\, \frac{M}{n} \;:\;
  k = 1, \dots, n \,\Bigr\}^{d}
\]
consist of the centers of the uniform grid of $[-M, M]^{d}$ with
$n$ cells per axis. For every
$\varepsilon \ge 2\, L_{\Psi}\, \delta$, the lattice command set
\[
  \A \;=\; \bigl\{\, w_{g} \;:\; g \in G^{j} \,\bigr\},
  \qquad
  w_{g}\big|_{I^{j}_{\ell}} \;\equiv\; g_{\ell} ,
\]
is $\delta$-dispersive over $\mathcal{U}^{\varepsilon}_{M}(x)$ at
every state $x$, and
\[
  \log |\A|
  \;\le\;
  c_{d}\, \frac{M}{\delta}\,
  \log\!\Bigl( 1 + \frac{M}{\delta} \Bigr) ,
\]
with $c_{d} = c_{d}(d)$ a constant depending on the flat
dimension $d$ alone.
\end{theorem}

We prove Theorem~\ref{thm:realization} in
Appendix~\ref{app:lattice}. The theorem provides a
system-agnostic constructive path to dispersive sampling in
the low-dispersion regime $\delta < r_{U} / (2 L_{\Psi})$,
with $r_{U}$ the inradius of the input set. The construction
uses command sets only mildly larger than the upper estimate
of Proposition~\ref{prop:floor}. In practice, system-agnostic
sampling as prescribed by the theorem can be very inefficient
due to low acceptance rates $|\A(x)| / |\A|$, wasting
computation as most samples at large $\delta$ violate the
admissibility requirement $\Psi_{u}( \jet z,\, w ) \in U$. As
a result, experimental implementations must leverage
additional system arithmetic to sample efficiently, as we
demonstrate in our platform catalog of
Section~\ref{sec:examples}.

\subsection{Coverage and Optimality}
\label{sec:correction}

In this subsection, we establish the coverage and optimality
guarantees of dispersive forward trees. We prove that the tree
generated by a $\delta_{\A}$-dispersive command set covers the
$M$-smooth eroded trajectory class at a deviation
$O(\delta_{\A})$ throughout the problem horizon. Our result rests
on a margin correction mechanism which shows that the additional
authority available to the command set can be leveraged to track
a uniformly wide tube around an eroded smooth reference.

The mechanism addresses a central obstruction to
finite-sample coverage. While dispersive sampling places a tree
node within $O(\delta_{\A})$ of a reference trajectory over a
single window, these deviations compound through the dynamics,
and the accumulated error grows exponentially in the horizon
under the classical Gr\"onwall estimate. An eroded reference,
however, leaves the command set spare control authority which is
utilized to repeatedly steer the deviated state
back onto the reference within one window
(Figure~\ref{fig:tube}), resetting the error iteratively across
each window.

Suppose that the continuously differentiable input map
$\Psi_{u}$ holds the Lipschitz constant $L_{\Psi}$ in the
supremum norm over the Sobolev envelope $\mathcal{E}_{2M}$
(recall Definition~\ref{def:msmooth}).

\begin{lemma}[Margin Correction]
\label{lem:corrector2}
For some $\varepsilon \in (0,\, L_{\Psi} M]$ and
$M' := M + \varepsilon / (2 L_{\Psi})$, let
$u \in \mathcal{U}^{\varepsilon}_{M}(x)$ be an
$\varepsilon$-eroded admissible control on an interval $I$ from a
state $x \in \X$, and let
$z_{u}(t) = h\bigl( \tra_{u}(t) \bigr)$ be its flat trajectory.
Let $x' \in \X$ be another state satisfying
\[
  \bigl| \Xi(x') - \jet z_{u}(t) \bigr|_{\infty}
  \;\le\; s
\]
at a time $t$ with $[t,\, t + 1] \subseteq I$. There exists a
constant $C_{B} > 0$ such that if
\[
  s \;\le\; \frac{\varepsilon}{2\, L_{\Psi}\, C_{B}} ,
\]
then there exists a corrector control
$\nu \in \mathcal{U}^{\varepsilon/2}_{M'}(x')$ on
$[t,\, t + 1]$ whose trajectory $\tra_{\nu}$ from $x'$
satisfies
\[
  d\bigl( \tra_{\nu},\, \tra_{u};\, [t,\, t + 1] \bigr)
  \;\le\; C_{B}\, s,
  \ \,
  \tra_{\nu}(t + 1) = \tra_{u}(t + 1) .
\]
\end{lemma}

The corrector therefore corrects the margin error
$s = \bigl| \Xi(x') - \jet z_{u}(t) \bigr|_{\infty}$ within one
window. As shown in the proof of Lemma~\ref{lem:corrector2} in
Appendix~\ref{app:corrector}, the constant $C_{B}$ is given by
\[
  C_{B} \;=\; \max_{1 \le i \le d} c_{Q}(r_{i}) ,
\]
where $c_{Q}(r) \ge 1$ is the norm of Hermite interpolation on
the unit window, bounding every derivative through order $r$ of
the polynomial that matches prescribed jets of order $r - 1$ at
both endpoints, in proportion to the largest prescribed value.
Both $C_{B}$ and the abbreviation
$c_{B} := (2\, L_{\Psi}\, C_{B})^{-1}$ are fixed henceforth.

\begin{figure}[t]
  \centering
  \includegraphics[width=\linewidth]{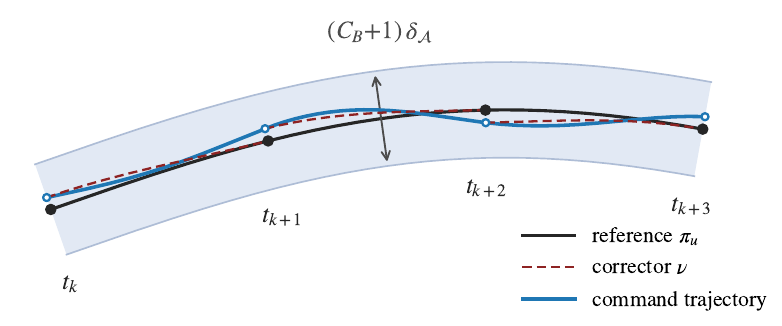}
  \caption{Corrective tracking of an eroded reference
  $\tra_{u}$ by a dispersive command trajectory, certified by
  correctors $\nu$.}
  \label{fig:tube}
\end{figure}

\label{sec:trees}

It then follows, as we show in the theorem below, that dispersive
commands anchor to the above correctors and track
$O(\delta_{\A})$ tubes around eroded $M$-smooth references
(Figure~\ref{fig:tube}). Note
that the correctors utilize some additional authority, namely
half the erosion $\varepsilon$ they can certify correction for,
as well as the additional smoothness headroom
$M' = M + \varepsilon / (2 L_{\Psi})$. As a result, the command
sets below demand local dispersion over the class
$\mathcal{U}^{\varepsilon/2}_{M'}$.

\begin{theorem}[Eroded coverage]
\label{thm:coverage}
Let $\varepsilon \in (0,\, L_{\Psi} M]$ with
$M' := M + \varepsilon / (2 L_{\Psi})$, let $K$ be a horizon,
and let $u \in \mathcal{U}^{\varepsilon}_{M}(x_{\mathrm{init}})$
be an $\varepsilon$-eroded admissible control on $[0, K]$. Suppose
$\A$ is $\delta_{\A}$-dispersive over
$\mathcal{U}^{\varepsilon/2}_{M'}(x)$ at every state $x$, with
$\delta_{\A} \le c_{B}\, \varepsilon$. Then there exists an
accepted command sequence
$(a_{1}, \dots, a_{K}) \in \A^{K}$ whose trajectory $\tra$ from
$x_{\mathrm{init}}$ satisfies
\[
  d\bigl(\tra,\, \tra_{u};\, [0, K]\bigr)
  \;\le\; ( C_{B} + 1 )\, \delta_{\A} .
\]
Hence the command tree tracks a tube of width
$( C_{B} + 1 )\, \delta_{\A}$ around the reference
trajectory $\tra_{u}$.
\end{theorem}

\label{sec:optimality}

Once tube proximity is established, the passage to cost closure is
standard following some assumed regularity on the cost functional as
well as the problem predicates
\citep{li2016asymptotically, karaman2011sampling}. Let
$L_{J}, L_{\preds} > 0$ be such that for all trajectories $\tra, \tra'$
defined on $[0, T]$,
\begin{equation}
\label{eq:lipschitz}
  \begin{aligned}
  \bigl| J(z_{\tra}) - J(z_{\tra'}) \bigr|
  &\;\le\; L_{J}\, d(\tra, \tra') , \\
  \bigl| g(z_{\tra}) - g(z_{\tra'}) \bigr|
  &\;\le\; L_{\preds}\, d(\tra, \tra')
  \quad \forall\, g \in \preds .
  \end{aligned}
\end{equation}
Since $x(t) = \Psi_{x}\bigl( \jet z (t) \bigr)$ with $\Psi_{x}$
locally Lipschitz, any Lipschitz functional of the state curve
qualifies,
\[
  \bigl| J\bigl( x(\cdot) \bigr) - J\bigl( x'(\cdot) \bigr) \bigr|
  \;\le\;
  L \, \sup_{t \in [0,T]} \bigl| x(t) - x'(t) \bigr| .
\]

We define the \emph{robust eroded optimum} of \eqref{eq:ocp} as
the optimum achieved by $M$-smooth trajectories of eroded inputs
on tightened predicates.
For a predicate margin $\gamma \ge 0$ and an input erosion
$\varepsilon > 0$, define
\[
  \begin{aligned}
  \mathcal{U}_{M, \gamma, \varepsilon}
  \;={}&
  \bigl\{\, u \in
  \mathcal{U}^{\varepsilon}_{M}(x_{\mathrm{init}}) \;:\;
  g\bigl(z_{u}\bigr) \ge \gamma
  \;\;\forall\, g \in \preds \,\bigr\}, \\
  J^{*}_{M,\gamma,\varepsilon}
  \;={}&
  \inf_{u \in \mathcal{U}_{M,\gamma,\varepsilon}}
  J\bigl(z_{u}\bigr) .
  \end{aligned}
\]

It is then immediate, as recorded by the theorem below, that
$O(\delta_{\A})$ dispersion achieves $O(\delta_{\A})$-eroded
optimality.

\begin{theorem}[Eroded optimality]
\label{thm:optimality}
Let the horizon $T \in \mathbb{N}$ and
$\varepsilon \in (0,\, L_{\Psi} M]$ with
$M' := M + \varepsilon / (2 L_{\Psi})$. Suppose
$\mathcal{U}_{M,\gamma,\varepsilon} \ne \emptyset$, and that $\A$
is $\delta_{\A}$-dispersive over
$\mathcal{U}^{\varepsilon/2}_{M'}(x)$ at every state $x$,
with $\delta_{\A} \le c_{B}\, \varepsilon$ and
$\gamma \ge L_{\preds}\, ( C_{B} + 1 )\, \delta_{\A}$. Then there exists an accepted command
sequence $(a_{1}, \dots, a_{T}) \in \A^{T}$ whose trajectory
$\tra$ from $x_{\mathrm{init}}$ is feasible for \eqref{eq:ocp},
\[
  g\bigl(z_{\tra}\bigr) \;\ge\; 0 \qquad \forall\, g \in \preds ,
\]
and whose cost satisfies
\[
  J\bigl(z_{\tra}\bigr)
  \;\le\;
  J^{*}_{M,\gamma,\varepsilon}
  + L_{J}\, ( C_{B} + 1 )\, \delta_{\A} .
\]
\end{theorem}


Coverage of the $M$-smooth trajectory class in the sense of
Theorem~\ref{thm:coverage} naively requires a command tree of size
$|\A|^{K}$ for trajectories of horizon $K$. With
$\log |\A| \sim c_{\star}\, d\, M / \delta_{\A}$ as estimated by
Proposition~\ref{prop:floor}, the tree size demanded by naive growth
is $O\bigl( e^{\, c_{\star} d K M / \delta_{\A}} \bigr)$,
exponentially exploding in the problem horizon at given
certification requirements $(\delta_{\A}, M)$. In the next
subsection, we show that such an exponential growth is an
unavoidable requirement if one demands coverage of the entire
$M$-smooth class.

\subsection{The Necessary Size of Covering Trees}
\label{sec:sparse}

A \emph{command tree}, or an \emph{$\A$-tree}, rooted at an
initial state $x_{\mathrm{init}}$ is a finite set of nodes in
$\X$ such that each node other than the root is the endpoint
state of an accepted command from a parent node. Denote by
$\Phi_{a}(x')$ the endpoint state of a command
$a \in \A(x')$ from the node $x'$, then a node at depth $k$ has
the form
\[
  x \;=\; \Phi_{a_{k}} \circ \Phi_{a_{k-1}} \circ \cdots \circ
  \Phi_{a_{1}}(x_{\mathrm{init}})
\]
for some command sequence $(a_{1}, \dots, a_{k})$ accepted from
$x_{\mathrm{init}}$. In particular,
Theorem~\ref{thm:coverage} states that the complete $\A$-tree of depth
$K$, the tree holding every accepted sequence, contains, for every
$u \in \mathcal{U}^{\varepsilon}_{M}(x_{\mathrm{init}})$, a
trajectory within $( C_{B} + 1 )\, \delta_{\A}$ of $\tra_{u}$
over $[0, K]$.

As we have argued, the complete $\A$-tree of depth $K$ holds
$O\bigl( e^{\, c_{\star} d K M / \delta_{\A}} \bigr)$ nodes,
rendering solution search upon it computationally intractable.
However, this in itself does not imply that dispersive
$M$-smooth coverage is out of reach, because the complete tree
is redundant and tubes of width
$E' := ( C_{B} + 1 )\, \delta_{\A}$ contain interchangeable
paths within the tree. A relatively sparse subtree may then
deliver our coverage guarantees and enable tractable solution
search.

In this subsection, we show that despite this
redundancy, the minimal tree that offers coverage over the
$M$-smooth class does not escape the exponential explosion in
the horizon. We show that the required size of the covering
tree is bounded below by the packing number of the target
trajectory class, and estimate the packing number of the target
using the standard metric entropy machinery introduced in
Section~\ref{sec:prelim-entropy}.

For a horizon $K$ and $\varepsilon > 0$, let $\mathcal{Z}_{K}$
be the $M$-smooth $\varepsilon$-eroded trajectory class from
$x_{\mathrm{init}}$ over the horizon $[0, K]$, carrying the
metric $d(\cdot, \cdot;\, [0, K])$, and write
$P_{s}(K) := P_{s}( \mathcal{Z}_{K} )$ for its packing number at
scale $s > 0$, following Definition~\ref{def:covpack}.

The following proposition, proven in Appendix~\ref{app:covsize},
estimates the necessary size of a command tree that maintains
dispersive coverage on $\mathcal{Z}_{K}$.

\begin{proposition}[Coverage complexity]
\label{prop:covsize}
Let $\varepsilon, E > 0$, let $K$ be a horizon, and let
$\mathcal{T}$ be a command tree of depth $K$.
\begin{itemize}
\item[(i)] If $\mathcal{T}$ contains a trajectory within $E$ of
$\tra_{u}$ over $[0, K]$ for every
$u \in \mathcal{U}^{\varepsilon}_{M}(x_{\mathrm{init}})$, then
it contains at least $P_{2E}(K)$ distinct trajectories of depth
$K$.
\item[(ii)] Suppose the eroded input set retains a ball,
$u_{0} + [-\eta,\, \eta]^{m} \subseteq U^{\varepsilon}$ for some
$u_{0} \in U$ whose constant control from $x_{\mathrm{init}}$
carries an $(M/2)$-smooth trajectory, and
$\eta \in (0,\, L_{\Psi} M/2]$. Then there is a
constant $\kappa = \kappa(R) > 0$ such that, for every
$0 < s \le \kappa\, \eta / L_{\Psi}$,
\[
  \log P_{s}(K)
  \;\ge\;
  \frac{\kappa\, \eta\, K \log 2}{2\, L_{\Psi}\, s} .
\]
It follows that the tree $\mathcal{T}$ of (i) must contain at
least $2^{\,\kappa \eta K / (4 L_{\Psi} E)}$ nodes whenever
$2 E \le \kappa\, \eta / L_{\Psi}$.
\end{itemize}
\end{proposition}

Proposition~\ref{prop:covsize} shows that so long as the smooth
eroded trajectory class retains a non-trivial interior, its packing
number grows exponentially in the horizon $K$, following which
the size of every covering tree must explode exponentially.

The linear rate $O( \delta_{\A} )$ of the optimality gap
relative to the local dispersion of the sampler is trivially
tight if one considers the cost of tracking a reference at
$K = 1$. The following example uses the packing number
machinery to present the same result for a $K$-horizon
trajectory tracking problem for completeness.

\begin{example}
\label{ex:tight}
Consider the trajectory tracking OCP
\[
  \begin{aligned}
  \min_{u(\cdot)} \quad&
  d\bigl( \tra_{u},\, \tra_{u^{\circ}};\, [0, K] \bigr) \\
  \text{s.t.} \quad&
  \dot{x} = f( x, u ) , \quad
  x(0) = x_{\mathrm{init}} ,
  \end{aligned}
\]
where $u^{\circ} \in
\mathcal{U}^{\varepsilon}_{M}(x_{\mathrm{init}})$ is a reference
control, so that $J^{*}_{M, \gamma, \varepsilon} = 0$.
\begin{itemize}
\item[(i)] Equip each state $x$ with the $\delta_{\A}$-dispersive
command set $\A^{*}(x)$ of Proposition~\ref{prop:floor} at $M'$
with $\delta_{\A} \le c_{B}\, \varepsilon$, and let $\mathcal{T}$ be
the complete tree of depth $K$, carrying
$N \le \sup_{x} | \A^{*}(x) |^{K}$ trajectories.

Applying Theorem~\ref{thm:optimality}, for every reference
$u^{\circ}$,
\[
  \exists\, \tra \in \mathcal{T} :\quad
  J( \tra ) \;\le\; ( C_{B} + 1 )\, \delta_{\A} .
\]
Conversely, apply Proposition~\ref{prop:covsize}(ii) at the
scale $2s$ with
\[
  s \;:=\; \frac{\kappa\, \eta\, K \log 2}{8\, L_{\Psi} \log N}
\]
to furnish a packing $Q \subset \mathcal{Z}_{K}$ of $|Q| > N$
members pairwise over $2s$ apart. Then, for every
$\tra \in \mathcal{T}$,
\[
  \bigl| \bigl\{\, q \in Q \;:\;
  d( \tra, q;\, [0, K] ) \le s \,\bigr\} \bigr|
  \;\le\; 1 ,
\]
and $|Q| > N$ leaves a member $\tra_{u^{\circ}} \in Q$ with
\[
  J( \tra )
  \;=\;
  d\bigl( \tra,\, \tra_{u^{\circ}};\, [0, K] \bigr)
  \;>\; s
  \qquad \forall\, \tra \in \mathcal{T} .
\]
Substituting
$\log N \le K \sup_{x} \log |\A^{*}(x)|
\le c_{\star}\, d\, K M' / \delta_{\A}$
into $s$, the two sides sandwich the worst-case optimality gap
of the complete tree,
\[
  \begin{aligned}
  \frac{\kappa\, \eta \log 2}
       {8\, c_{\star}\, d\, L_{\Psi}\, M'}\; \delta_{\A}
  \;&\le\;
  \sup_{u^{\circ} \in\,
  \mathcal{U}^{\varepsilon}_{M}(x_{\mathrm{init}})} \,
  \min_{\tra \in \mathcal{T}} \,
  J( \tra ) - J^{*}_{M, \gamma, \varepsilon}
  \\
  &\le\;
  ( C_{B} + 1 )\, \delta_{\A} .
  \end{aligned}
\]
The optimality gap estimated by Theorem~\ref{thm:optimality} is
therefore tight in its rate $O( \delta_{\A} )$.
\item[(ii)] Let instead $\mathcal{T}$ be any set of at most
$2^{K}$ trajectories. Apply Proposition~\ref{prop:covsize}(ii)
with the scale $s := \kappa\, \eta / ( 4 L_{\Psi} )$ to furnish
a packing $Q \subset \mathcal{Z}_{K}$ of
$4^{K} > |\mathcal{T}|$ members pairwise over $s$ apart. Every
trajectory of $\mathcal{T}$ then lies within $s/2$ of at most
one member of $Q$, leaving a member $\tra_{u^{\circ}} \in Q$
with
\[
  J( \tra )
  \;=\;
  d\bigl( \tra,\, \tra_{u^{\circ}};\, [0, K] \bigr)
  \;>\; \frac{s}{2}
  \;=\; \frac{\kappa\, \eta}{8\, L_{\Psi}}
  \quad \forall\, \tra \in \mathcal{T} .
\]
Subexponential trees submit to the above cost floor regardless
of the local dispersion $\delta_{\A}$ of their sampler.
\item[(iii)] We numerically instantiate the trajectory tracking
problem on the scalar double integrator $\ddot{x} = u$, with
$z = x$, $d = 1$, $U = [-1, 1]$, and $\varepsilon = 0.25$.
Theorem~\ref{thm:realization} makes dispersive command sets in
the form $\A_{j,n}$ with $j$ subwindows and $n$ grid levels. We
enumerate the complete tree of each $\A_{j,n}$ up to $K = 6$ and
record the tracking cost
\[
  e_{k} \;=\; \min_{\tra \,\in\, \A^{k}}
  d\bigl(\tra,\, \tra_{u};\, [0, k]\bigr) ,
  \qquad k = 1, \dots, K ,
\]
against random $\varepsilon$-eroded references and an extremal
reference $|u| = 1$.
\begin{itemize}
\item[--] In Figure~\ref{fig:coverage}(a), we observe that the
eroded error stays nearly flat in $k$ while the extremal error
grows steadily, as it must, since $|u| = 1$ exceeds every level
of the lattice.
\item[--] In Figure~\ref{fig:coverage}(b), the terminal error
falls linearly with the dispersion,
$e_{6} \approx 0.65\, \delta_{\A}$, below the worst-case
certified $(C_{B} + 1)\, \delta_{\A}$ with
$C_{B} = c_{Q}(2) = 10$ (gray points carry sampled lower bounds
on $\delta_{\A}$ and are excluded from the fit).
\end{itemize}

Complete trees
become quickly intractable as they hold
$|\A|^{K} \approx 10^{11}$ trajectories at $|\A| \sim 64$ at the
finest family we experiment with, which motivates the dominance
pruning procedure of the next subsection as a necessary
precursor to the deployment of dispersive forward trees for
sampling-based planning.
\end{itemize}
\end{example}

\begin{figure}[t]
  \centering
  \begin{subfigure}{0.49\linewidth}
    \includegraphics[width=\linewidth]{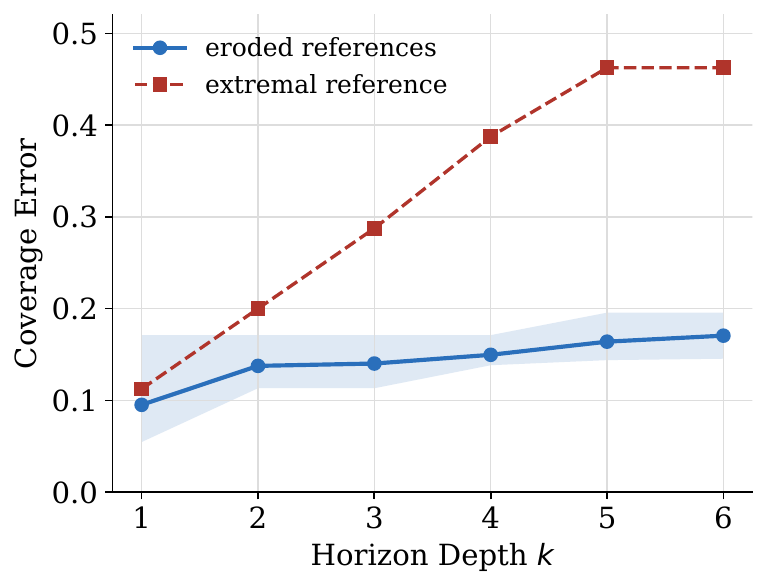}
    \caption{}
    \label{fig:coverage-a}
  \end{subfigure}\hfill
  \begin{subfigure}{0.49\linewidth}
    \includegraphics[width=\linewidth]{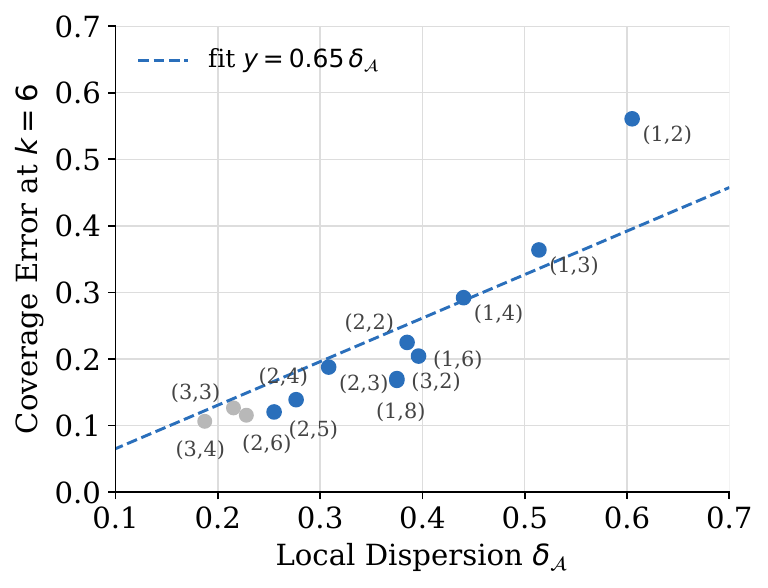}
    \caption{}
    \label{fig:coverage-b}
  \end{subfigure}
  \caption{Coverage on the double integrator
  ($\varepsilon = 0.25$). (a) Coverage error
  $e_{k} = \min_{\tra \in \A^{k}}
  d\bigl(\tra, \tra_{u}; [0, k]\bigr)$ for $(j,n) = (1,8)$, four
  eroded references (band) and the extremal reference. (b)
  $e_{6}$ vs the local dispersion $\delta_{\A}$ over
  $\mathcal{U}^{\varepsilon/2}$, one point per family.}
  \label{fig:coverage}
\end{figure}

\subsection{Dominance Pruning}
\label{sec:genpruning}

Despite the analysis of the previous subsection rendering
trajectory class coverage intractable, it is possible to
maintain dispersive coverage over a subclass of near-optimal
trajectories and tractably realize the optimality certificate
of Theorem~\ref{thm:optimality}. In this subsection, we present
a spatio-temporal dominance pruning procedure that leverages
structural properties of the underlying OCP to iteratively
prune the forward tree synchronously with its growth and allows
us to maintain near-optimality while curtailing the necessary
tree size, and with it the deployment compute, to a polynomial
order in the horizon length. The procedure we present requires
some restriction on the structure of the OCP~\eqref{eq:ocp},
namely we attend to temporally aggregating costs and separating
predicates. This setting generalizes popular structure in
kinodynamic planning such as additive Lipschitz costs and
static collision constraints considered by
SST~\citep{li2016asymptotically}.

\begin{definition}[Aggregating costs, separating predicates]
\label{as:sumstruct}
\leavevmode
\begin{itemize}
\item[(i)] The cost $J(\cdot)$ of \eqref{eq:ocp} is
\emph{recursively aggregated} if there exist an initial value
$\Sigma_{0} \in \R$, window updates
$\varrho_{k} : \R \times C\bigl( [k,\, k+1] \bigr) \to \R$, and
a nondecreasing $L_{\varphi}$-Lipschitz
$\varphi : \R \to \R$ such that, for every trajectory $\tra$ on
$[0, K]$,
\begin{equation}
\label{eq:sumcost}
  J\bigl( z_{\tra} \bigr) \;=\; \varphi\bigl( \Sigma_{K}
  \bigr) ,
  \
  \Sigma_{k+1}
  \;=\;
  \varrho_{k}\Bigl( \Sigma_{k},\;
  z_{\tra}\big|_{[k,\, k+1]} \Bigr) ,
\end{equation}
and the window updates satisfy, for all aggregates
$\Sigma, \Sigma' \in \R$ and trajectories $\tra, \tra'$,
\begin{equation*}
  \varrho_{k}\Bigl( \Sigma,\,
  z_{\tra}\big|_{[k,\, k+1]} \Bigr)
  \;\le\;
  \varrho_{k}\Bigl( \Sigma',\,
  z_{\tra}\big|_{[k,\, k+1]} \Bigr)
  \ \text{if}\ \Sigma \le \Sigma' ,
\end{equation*}
\begin{equation*}
  \Bigl| \varrho_{k}\Bigl( \Sigma,\,
  z_{\tra}\big|_{[k,\, k+1]} \Bigr)
  - \varrho_{k}\Bigl( \Sigma',\,
  z_{\tra}\big|_{[k,\, k+1]} \Bigr) \Bigr|
  \le
  \bigl| \Sigma - \Sigma' \bigr| ,
\end{equation*}
\begin{multline*}
  \Bigl| \varrho_{k}\Bigl( \Sigma,\,
  z_{\tra}\big|_{[k,\, k+1]} \Bigr)
  - \varrho_{k}\Bigl( \Sigma,\,
  z_{\tra'}\big|_{[k,\, k+1]} \Bigr) \Bigr|
  \\
  \;\le\;
  L\, d\bigl( \tra,\, \tra';\, [k,\, k+1] \bigr) .
\end{multline*}
That is, in order, the window updates are nondecreasing in the
aggregate, non-expansive relative to aggregate perturbations,
and $L$-Lipschitz in the window trajectory. The aggregate
Lipschitz constant of the cost in the sense of
\eqref{eq:lipschitz} is recovered as
$L_{J} = L_{\varphi}\, K\, L$.
\item[(ii)] The predicates $g \in \preds$ of \eqref{eq:ocp}
\emph{separate over the time windows} if there exist
$L_{\preds}$-Lipschitz window predicates $g_{k}$ such that, for
every trajectory $\tra$ on $[0, K]$,
\begin{equation}
\label{eq:predsep}
  g\bigl( z_{\tra} \bigr)
  \;=\;
  \min_{0 \le k < K}
  g_{k}\Bigl( z_{\tra}\big|_{[k,\, k+1]} \Bigr) .
\end{equation}
\end{itemize}
\end{definition}

The slightly uncomfortable abstract structure imposed by
Definition~\ref{as:sumstruct} allows the admission of some
practically useful planning instances beyond Lipschitz
additives and static obstacles, as listed by the example below.

\begin{example}
\label{ex:instances}
The following costs and predicates carry the structure of
Definition~\ref{as:sumstruct}, as verifiable by direct
inspection.
\begin{enumerate}
\item \emph{Discounted sums.}
$\varrho_{k}( \Sigma, w ) = \Sigma + \beta^{k} \ell_{k}( w )$
with $\beta \le 1$ and $\varphi = \mathrm{id}$. At $\beta = 1$
we recover the usual window-additive costs like the arc length
and the integral of a Lipschitz running cost.
\item \emph{Running extrema.}
$\varrho_{k}( \Sigma, w ) = \min\{ \Sigma,\, c_{k}( w ) \}$, and
likewise the maximum, such as the closest approach to a goal
region, the peak tracking error, and the least clearance over
the horizon.
\item \emph{Spatio-temporal robustness.} For a margin
$\mu : \R^{d} \to \R$, the robustness of the safety
specification, $\mu\bigl( z( t ) \bigr) \ge 0$ throughout
$[0, K]$, is
\[
  J\bigl( z_{\tra} \bigr)
  \;=\;
  \min_{t \in [0, K]} \mu\bigl( z_{\tra}( t ) \bigr) ,
\]
the running minimum of the window margins
$c_{k}( w ) = \min_{t \in [k,\, k+1]} \mu\bigl( w( t ) \bigr)$,
and the robustness of the reach specification,
$\mu\bigl( z( t ) \bigr) \ge 0$ at some $t \in [0, K]$, is the
corresponding running maximum \citep{fainekos2009robustness}.
\item \emph{Dynamic collision avoidance.} The predicate
$g_{k}( w )
= \min_{t \in [k,\, k+1]}
\mathrm{dist}\bigl( w( t ),\, O( t ) \bigr) - r$
maintains the clearance $r$ from the reachable set $O( t )$ of a
moving obstacle.
\item \emph{Visibility maintenance.} The predicate
$g_{k}( w )
= \min_{t \in [k,\, k+1]}
\mathrm{dist}\bigl( [ w( t ),\, p( t ) ],\, O \bigr) - r$
keeps the line of sight $[ w( t ),\, p( t ) ]$ to a moving
target $p( t )$ away from known occluders $O$.
\item \emph{Envelope margins.} The predicate
$g_{k}( w )
= \min_{t \in [k,\, k+1]}
\bigl( v_{\max} - | \dot{w}( t ) | \bigr)$
bounds the speed along the trajectory.
\end{enumerate}
\end{example}

We now describe the dominance pruning procedure.

\begin{procedure}[Dominance pruning]
\label{def:genpruning}
Let $x$ be a depth-$k$ node of the command sequence
$(a_{1}, \dots, a_{k})$ and let $\tra_{x}$ be the corresponding
trajectory on $[0, k]$, carrying the running summary
$\Sigma( x )$ along the recursion \eqref{eq:sumcost}.

Fix a radius $s > 0$ and let
\[
  q_{s}(\xi)
  \;:=\;
  s \bigl\lfloor \xi / s \bigr\rfloor ,
  \qquad \xi \in \R^{|R|} ,
\]
assign each jet the corner of its cell in the grid
$s \Z^{|R|}$, where $\lfloor \cdot \rfloor$ denotes the
componentwise floor.

The \emph{dominance-pruned tree} of depth $K$ is grown by the
recursion $V_{0} := \{ x_{\mathrm{init}} \}$ and, for
$0 \le k < K$,
\[
  \widehat{V}_{k+1}
  \;:=\;
  \left\{\,
    \Phi_{a}(x)
    \;\middle|\;
    \begin{gathered}
    x \in V_{k} ,\quad a \in \A(x) , \\
    g_{k}\bigl(
    z_{\tra_{\Phi_{a}(x)}}\big|_{[k,\, k+1]} \bigr)
    \ge 0
    \quad \forall\, g \in \preds \\
    z_{\tra_{\Phi_{a}(x)}}\big|_{[k,\, k+1]}
    \ \text{$2M$-smooth}
    \end{gathered}
  \,\right\} ,
\]
\[
  V_{k+1}
  \;:=\;
  \left\{\, x \in \widehat{V}_{k+1}
  \;\middle|\;
  \begin{gathered}
  \Sigma( x ) \le \Sigma( x' )
  \quad \forall\, x' \in \widehat{V}_{k+1} \\
  \text{with}\
  q_{s}\bigl( \Xi( x' ) \bigr) = q_{s}\bigl( \Xi( x ) \bigr)
  \end{gathered}
  \,\right\} .
\]
With ties resolved arbitrarily, dominance pruning retains
exactly one node per occupied cell of the jet-space grid at
each depth $k$. The pruned tree is a command tree with the node
set $\bigcup_{k=0}^{K} V_{k}$.
\end{procedure}

Note that jet space quantization for dominance pruning is
equivalent to state space quantization and pruning following
\[
  \mathrm{dist}( x,\, x' ) \;\le\; \rho
  \;\;\Longrightarrow\;\;
  \bigl|\, \Xi( x ) - \Xi( x' ) \,\bigr|_{\infty}
  \;\le\; L_{\mathrm{x}}\, \rho
\]
for a metric $\mathrm{dist}$ on the states.

\begin{theorem}[Sparse tree optimality]
\label{prop:genpruning}
Let $\varepsilon \in (0,\, L_{\Psi} M]$ with
$M' := M + \varepsilon / (2 L_{\Psi})$, and let the cost and
the predicates of the OCP~\eqref{eq:ocp} obey the structure of
Definition~\ref{as:sumstruct}. Suppose $\A$ is
$\delta_{\A}$-dispersive over
$\mathcal{U}^{\varepsilon/2}_{M'}(x)$ at every state $x$, with
every command input bounded by $2M$,
$\| w_{a} \|_{\infty} \le 2M$ for all $a \in \A(x)$. Pick
a pruning radius $s > 0$ and a predicate margin $\gamma \ge 0$
with
\[
  \delta_{\A} + s \;\le\; c_{B}\, \varepsilon ,
  \qquad
  \gamma \;\ge\; L_{\preds}\, (C_{B} + 1)\,
  \bigl( \delta_{\A} + s \bigr) ,
\]
and suppose the reference class
$\mathcal{U}_{M, \gamma, \varepsilon}$ is nonempty. Then the
pruned tree of Procedure~\ref{def:genpruning} at radius $s$ of
depth $K$ holds at most
\[
  K\,
  \Bigl( \frac{8 K M}{s} \Bigr)^{d}
  \Bigl( \frac{8 M}{s} \Bigr)^{|R| - d}
  + 1
\]
nodes, and contains an accepted command sequence
$(a_{1}, \dots, a_{K})$ whose trajectory $\tra$ from
$x_{\mathrm{init}}$ is feasible for \eqref{eq:ocp},
\[
  g\bigl( z_{\tra} \bigr) \;\ge\; 0
  \qquad \forall\, g \in \preds ,
\]
and whose cost satisfies
\[
  J\bigl( z_{\tra} \bigr)
  \;\le\;
  J^{*}_{M, \gamma, \varepsilon}
  + L_{J}\, (C_{B} + 1)\,
  \bigl( \delta_{\A} + s \bigr) .
\]
\end{theorem}

The theorem is proven in Appendix~\ref{app:genpruning}.
Dominance pruning thus curtails our search tree to a polynomial
size in the problem horizon while retaining the optimality
certificate of Theorem~\ref{thm:optimality} up to an increment
$\delta_{\A} \to \delta_{\A} + s = \delta_{\A} +
L_{\mathrm{x}}\, \rho$ in the certified gap.
Dominance-pruned tree search costs roughly
\[
  O\Bigl(\, |\A|\, K\,
  \Bigl( \frac{8 K M}{s} \Bigr)^{d}
  \Bigl( \frac{8 M}{s} \Bigr)^{|R| - d}
  \,\Bigr)
\]
compute iterations, with a massive parallelization scope across
each depth. This renders dispersive forward tree search
implementable for kinodynamic planning on embedded-tier
parallel processors, as we demonstrate in the next section.

\begin{figure}[t]
  \centering
  \includegraphics[width=\columnwidth]{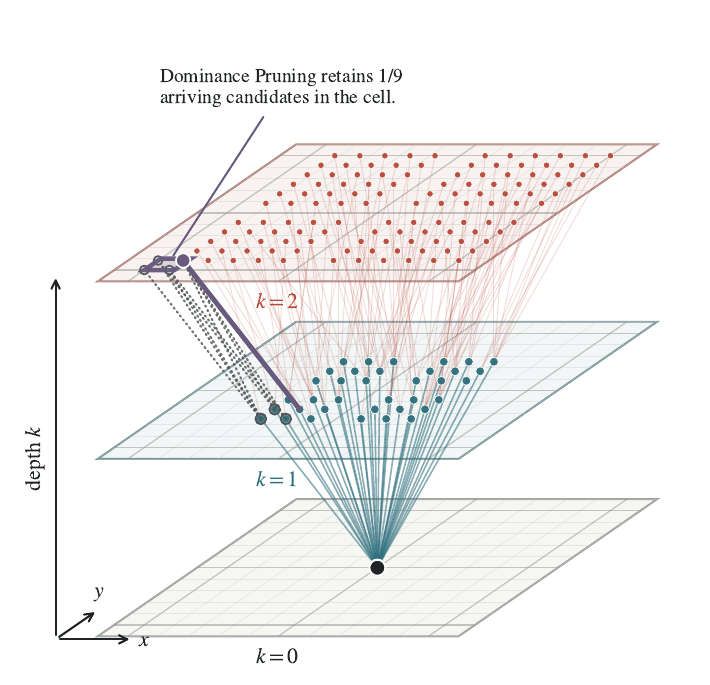}
  \caption{Dominance pruning illustrated for the 2D single
  integrator $\dot{p} = u$. At each depth $k$, one representative
  is retained per occupied state-space cell.}
  \label{fig:dominance_pruning}
\end{figure}

\section{Dispersive Tree Search: Algorithm and Experiments}
\label{sec:experiments}

In this section we present DFT\textsuperscript{*}, a massively
parallel implementation of dispersive forward tree search with
depth-synchronous dominance pruning. We evaluate this
implementation on Dynobench~\citep{ortizharo2024idbastar}, a
kinodynamic planning benchmark, against state-of-the-art
planners based on sampling, motion-primitive search, and
optimization. We
further present a receding horizon implementation of
DFT\textsuperscript{*}, WarpWalker
(WWDFT\textsuperscript{*}), and evaluate its realtime planning
performance in dynamic environments against the state of the
art. Finally, we study the near-optimality of
DFT\textsuperscript{*} on a simple $K$-integrator system where
strong nonlinear optimization frameworks are available for
relative comparison.

\subsection{Kinodynamic Planning with Forward Dispersive Search}
\label{sec:dft-bench}

To recall notation, $z = h(x) \in \R^{d}$ is the flat output
with orders $R = (r_{1}, \dots, r_{d})$, $\Xi(x)$ is the
output jet at a state $x$, and a command $a$ from a command
set $\A$ is a measurable flat input
$w_{a} : [0, \Delta] \to \W$. A command produces the flat
rollout $z^{x}_{a} = \roll_{x}(a)$ from a state $x$ through
the integrator chain, along with the endpoint state
$\Phi_{a}(x)$.
Finally, let $\mathrm{dist}$ be a metric on states and
$\mathrm{cell}(x)$ the quantization of $x$ at a pruning radius
$\rho$ in this metric.

We present DFT\textsuperscript{*} as Algorithm~\ref{alg:bfs}.

\begin{algorithm}[!htb]
\caption{DFT\textsuperscript{*} --- breadth-first forward
dispersive search}
\label{alg:bfs}
\begin{algorithmic}
\definecolor{kwfor}{HTML}{1F77B4}
\definecolor{kwif}{HTML}{E8710A}
\definecolor{kwrep}{HTML}{8E44AD}
\definecolor{kwret}{HTML}{D62728}
\definecolor{kwend}{HTML}{7F7F7F}
\algrenewcommand\algorithmicforall{{\color{kwfor}\textbf{for all}}}
\algrenewcommand\algorithmicfor{{\color{kwfor}\textbf{for}}}
\algrenewcommand\algorithmicif{{\color{kwif}\textbf{if}}}
\algrenewcommand\algorithmicthen{}
\algrenewcommand\algorithmicelse{{\color{kwif}\textbf{else}}}
\algrenewcommand\algorithmicend{{\color{kwend}\textbf{end}}}
\algrenewcommand\algorithmicrepeat{{\color{kwrep}\textbf{repeat}}}
\algrenewcommand\algorithmicuntil{{\color{kwrep}\textbf{until}}}
\algrenewcommand\algorithmicreturn{{\color{kwret}\textbf{return}}}
\State $k \gets 0$; \quad
  $V_{0} \gets \{x_{\mathrm{start}}\}$; \quad
  $\Sigma(x_{\mathrm{start}}) \gets \Sigma_{0}$
\Repeat
  \State $\widehat{V}_{k+1} \gets \emptyset$
  \ForAll{$x \in V_{k}$}
    \ForAll{$a \in \A(x)$}
    \State $z^{x}_{a} \gets \roll_{x}(a)$ on $[0, \Delta]$;
      \quad $x' \gets \Phi_{a}(x)$
    \If{$g_{k}\bigl( z^{x}_{a} \bigr) \ge 0$ for all
      $g \in \preds$:}
      \State $\Sigma(x') \gets \varrho_{k}\bigl( \Sigma(x),\,
        z^{x}_{a} \bigr)$
      \State $\mathrm{parent}(x') \gets (x, a)$
      \State add $x'$ to $\widehat{V}_{k+1}$
    \EndIf
    \EndFor
  \EndFor
  \State $V_{k+1} \gets
    \left\{ x' \in \widehat{V}_{k+1}
    \middle|
    \begin{gathered}
    \Sigma( x' ) \le \Sigma( x'' )
    \ \forall x'' \in \widehat{V}_{k+1} \\
    \text{with}\
    \mathrm{cell}( x'' ) = \mathrm{cell}( x' )
    \end{gathered}
    \right\}$
  \State $k \gets k + 1$
\Until{termination condition}
\State \Return $\tra_{x}$ minimizing $\Sigma(x)$ over
  $x \in V_{k}$
\end{algorithmic}
\end{algorithm}

Note that at each depth $k$, the pairs $(x, a)$ with
$x \in V_{k}$ and $a \in \A(x)$ evaluate independently of one
another, so the at most $|V_{k}| \cdot |\A|$ pairs expand
concurrently, followed by the pruning step. The nested loops
parallelize across the available threads, and the loop over
$k$ is the only sequential dimension of the search. The
construction of $V_{k+1}$ from $V_{k}$ is thus implemented as
an invocation of a fused CUDA kernel. Each thread integrates
the window of its pair $(x, a)$ in substeps and checks the
predicates.

The performance of DFT\textsuperscript{*} freely improves
within a given wall-clock budget as the available parallel
compute scales. As a result, we emulate different tiers of
deployment by capping the kernel grid to the
streaming-multiprocessor (SM) count of representative
devices. We use $8$ SMs to emulate the Jetson Orin Nano and
$16$ for the AGX Orin, a popular set of onboard processors,
while the full desktop GPU evaluations use all $128$ SMs of an
RTX 4090.

\subsubsection{Algorithmic Variants}
\label{sec:variants}

Our headline algorithm DFT\textsuperscript{*}
(Algorithm~\ref{alg:bfs}) performs an unbiased breadth-first
search on the dispersive tree and thus carries the
near-optimality certificate of Theorem~\ref{prop:genpruning}
for the general class of optimal control
problems~\eqref{eq:ocp}. Kinodynamic planning problems admit
additional problem structure which planning
algorithms~\citep{ortizharo2024idbastar, king2026akinopdf}
routinely leverage. Such structural leverage is also admissible
within the DFT\textsuperscript{*} framework, which can bias the
search on the dispersive tree while converging to the
\emph{same} near-optimal solutions and maintaining the
\emph{same} certificate. We therefore include two
specializations of Algorithm~\ref{alg:bfs} in our evaluation
for completeness.

\noindent\textbf{Spatial Dominance Pruning:} When the problem predicates are static,
$g_{k} = g$ for every depth $k$, such as in static obstacle
evasion, the pruning procedure keeps one dominant node per
cell, as opposed to one per cell for each depth, so that with
$V_{0:k} := \bigcup_{j \le k} V_{j}$,
\[
  V_{k+1}
  :=
  \left\{ x \in \widehat{V}_{k+1}
  \,\middle|\,
  \begin{gathered}
  \Sigma( x ) \le \Sigma( x' )
  \ \forall\, x' \in \widehat{V}_{k+1} \cup V_{0:k} \\
  \text{with}\ \mathrm{cell}( x' ) = \mathrm{cell}( x )
  \end{gathered}
  \right\} .
\]

\noindent\textbf{A\textsuperscript{*} Beam Search:} When the
cost-to-go $J^{*}( x, t )$ admits a consistent estimate $h$,
the running summary $\Sigma$ of Algorithm~\ref{alg:bfs}
estimates the total cost as $f := \Sigma + h$, and
A\textsuperscript{*} beam
search~\citep{hart1968astar} is optimal on the dispersive tree
at a lower sample count. Expansion proceeds over batches of
estimated minimum cost nodes instead of the lowest depth
ordering of Algorithm~\ref{alg:bfs}.

In addition to the headline algorithm DFT\textsuperscript{*},
we evaluate DFT\textsuperscript{*}-Static,
DFT-A\textsuperscript{*} and DFT-A\textsuperscript{*}-Static,
which apply either or both of the above specializations and
recover the same solutions at lower computation wall times.
The Dynobench instances of Section~\ref{sec:dynobench} are
static kinodynamic planning problems, so that both
specializations apply and save up to $44\times$ of the
computation time. We report these results in
Appendix~\ref{app:variants}.

\subsection{Kinodynamic Planning on Dynobench}
\label{sec:dynobench}

We now evaluate DFT\textsuperscript{*} on kinodynamic planning
problems from Dynobench~\citep{ortizharo2024idbastar}, across
platforms spanning the first- and second-order unicycles, the
car with a trailer, and quadrotors under two parameterizations.

\subsubsection{Pushforward Construction of Dispersive Commands}
\label{sec:pushforward}

In our evaluation we deploy dispersive command sets
constructed by a \emph{pushforward} mechanism. For each
platform, we select a discrete set $G \subset \mathcal{U}$ of
admissible control signals and push it through the system flow,
\[
  \A \;=\; \A(x) \;=\; \bigl\{\, \tra^{x}_{u} \;:\; u \in G \,\bigr\}.
\]
The command set at a state $x$ is therefore an image of $G$
under the dynamical flow of the system. For each of our
constructions, we provide dispersion estimates of the form
\[
  \delta_{\A} \;\le\; \varrho(G,\, M)
\]
over the $M$-smooth trajectory class, where the estimate
$\varrho$ vanishes as $G$ is refined.

This construction mechanism is notably in contrast to the
`universal' pullback construction mechanism suggested by
Theorem~\ref{thm:realization}. The flat input lattice sampled by
Theorem~\ref{thm:realization} remains agnostic to the underlying
dynamical system and enforces control admissibility by rejecting
the commands
\[
  \bigl\{\, w \in \A \;:\;
  \Psi_{u}\bigl( \jet z(t),\, w(t) \bigr) \notin U
  \ \text{for some } t \in [0, 1] \,\bigr\} .
\]
Such pullback rejections can only be processed after a forward
pass through the system dynamics. We find that the acceptance
rates $|\A(x)| / |\A|$ can be quite low at deployment
resolutions and massively waste compute resources, limiting the
system-agnostic pullback mechanism to an object of theoretical
completeness. Theorem~\ref{thm:realization} therefore serves as
a certificate of existence and constructibility of dispersive
command sets across the entire flat class, while effective
deployment of DFT\textsuperscript{*} leverages additional
platform structure via the specific pushforward designs of the
catalog below.

\subsubsection{Command Sets for the Catalog}
\label{sec:examples}

We present the catalog of differentially flat platforms we
use for our evaluation. We briefly describe the dynamics of
each platform, followed by its flat reduction and the command
sets we use for our deployment. A compact summary of these
constructions is presented in this section, and we refer the
reader to Appendix~\ref{app:platform_catalog} for the detailed
calculus. Figure~\ref{fig:platform_fans} visualizes the
dispersive trees generated by our construction.

\medskip\noindent\textbf{Notation:} $I_{1}, \dots, I_{j}$
denote the $j$ equal subwindows of the window $[0, \Delta]$,
indexed by $\ell$, and a command holds its sampled channels
constant on each subwindow. $G_{n}[l, u]$ denotes the
$n$-level grid
\[
  G_{n}[l, u]
  \;=\;
  \Bigl\{\, l + \tfrac{k}{n-1}\,(u - l) :
  k = 0, \dots, n-1 \,\Bigr\}
\]
of spacing $(u - l)/(n - 1)$, and
$d_{S^{1}}(\theta, \theta') = \min\bigl\{ |\theta - \theta'|,
\; 2\pi - |\theta - \theta'| \bigr\}$ the angular distance on
$S^{1}$. The dispersion bounds provided over the $M$-smooth
trajectory class are derived in
Appendix~\ref{app:platform_catalog}.

\begin{figure*}[t]
  \centering
  \includegraphics[width=\textwidth]{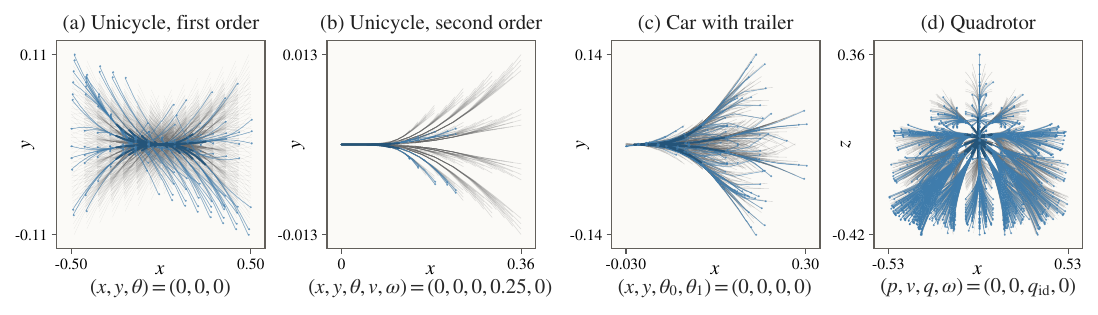}
  \caption{Depth-two forward trees under the constructed
  command sets. Gray trajectories represent those pruned off.}
  \label{fig:platform_fans}
\end{figure*}

\paragraph{1.\ Unicycle, first order.}
The first-order unicycle operates with the state
$(x, y, \theta) \in \R^{2} \times S^{1}$, the position and
orientation, the control input
$u = (v, \omega) \in [-\tfrac{1}{2}, \tfrac{1}{2}]^{2}$, the
forward speed and angular velocity, and the dynamics
$$
  \dot x = v \cos\theta, \qquad
  \dot y = v \sin\theta, \qquad
  \dot\theta = \omega .
$$
The system is flat with the output $z = (x, y)$ and orders
$R = (2, 2)$, the defining relations
$$
  v = \pm \|\dot z\|, \qquad
  \theta = \operatorname{atan2}\bigl( \pm\dot y,\, \pm\dot x
  \bigr), \qquad
  \omega = \frac{\dot z \wedge \ddot z}{\|\dot z\|^{2}} .
$$

Let
$g = (v_{\ell}, \omega_{\ell})_{\ell=1}^{j} \in
G_{n}\bigl[-\tfrac{1}{2}, \tfrac{1}{2}\bigr]^{2j}$. A command
$w_{g}$ holds the control constant at
$(v, \omega) = (v_{\ell}, \omega_{\ell})$ on
$I_{\ell}$, yielding
\[
  \begin{aligned}
  \A(x)
  &\;=\;
  \Bigl\{\, w_{g} \;:\; g \in
  G_{n}\bigl[-\tfrac{1}{2}, \tfrac{1}{2}\bigr]^{2j} \,\Bigr\}, \\
  w_{g}(t)
  &\;=\;
  v_{\ell}\,\omega_{\ell}\,
  \bigl( -\sin\theta(t),\; \cos\theta(t) \bigr),
  \quad t \in I_{\ell},
  \end{aligned}
\]
with $\theta(t) = \theta_{\ell-1} + \omega_{\ell}\,
(t - t_{\ell-1})$, and the dispersion
$$
  \delta_{n, j}
  \;\le\;
  \frac{1 + \Delta/2}{2\,(n-1)}
  + \Bigl( \sqrt{2}\, M + \tfrac{1}{8} \Bigr) \frac{\Delta}{j}
  \;\xrightarrow[\; n \uparrow,\ j \uparrow \;]{}\; 0 .
$$

We use the state-space quantization metric
$$
  \mathrm{dist} \;=\; \bigl\|(x - x',\, y - y')\bigr\|
  + \tfrac{1}{2}\, d_{S^{1}}(\theta, \theta') .
$$
For evaluation, we use the window
$\Delta = 0.5$\,s, $j = 1$, $n = 8$, the pruning radius
$\rho = 0.08$, and the goal tolerance $\delta_{0} = 0.1$.

\paragraph{2.\ Unicycle, second order.}
The second-order unicycle operates with the state
$(x, y, \theta, v, \omega)$, the position, orientation, forward
speed $|v| \le \tfrac{1}{2}$, and angular velocity
$|\omega| \le \tfrac{1}{2}$, the control input
$u = (a, \alpha) \in [-\tfrac{1}{4}, \tfrac{1}{4}]^{2}$, the
forward and the angular acceleration, and the dynamics
\[
  \dot x = v \cos\theta, \quad
  \dot y = v \sin\theta, \quad
  \dot\theta = \omega, \quad
  \dot v = a, \quad
  \dot\omega = \alpha .
\]
The system is flat with the output $z = (x, y)$ and orders
$R = (3, 3)$, the first-order relations defining
$(v, \theta, \omega)$ as before, and
\[
  a = \frac{\langle \dot z, \ddot z \rangle}{\|\dot z\|},
  \qquad
  \alpha = \frac{\dot z \wedge \dddot z}{\|\dot z\|^{2}}
  - 2\, \frac{\langle \dot z, \ddot z \rangle\,
  (\dot z \wedge \ddot z)}{\|\dot z\|^{4}} .
\]

Let
$g = (a_{\ell}, \alpha_{\ell})_{\ell=1}^{j} \in
G_{n}\bigl[-\tfrac{1}{4}, \tfrac{1}{4}\bigr]^{2j}$. A command
$w_{g}$ holds the control constant at
$(a, \alpha) = (a_{\ell}, \alpha_{\ell})$ on
$I_{\ell}$, yielding
\[
  \begin{aligned}
  \A(x)
  &\;=\;
  \Bigl\{\, w_{g} \;:\; g \in
  G_{n}\bigl[-\tfrac{1}{4}, \tfrac{1}{4}\bigr]^{2j} \,\Bigr\}, \\
  w_{g}(t)
  &\;=\;
  \begin{pmatrix}
  - v \omega^{2} \cos\theta
  - \bigl( 2 a_{\ell} \omega + v \alpha_{\ell} \bigr)
  \sin\theta \\
  - v \omega^{2} \sin\theta
  + \bigl( 2 a_{\ell} \omega + v \alpha_{\ell} \bigr)
  \cos\theta
  \end{pmatrix},
  \quad t \in I_{\ell},
  \end{aligned}
\]
where $(v, \omega, \theta)$ obey $\dot v = a_{\ell}$,
$\dot\omega = \alpha_{\ell}$, $\dot\theta = \omega$ on
$I_{\ell}$, and the dispersion
$$
  \delta_{n, j}
  \;\le\;
  \frac{2 + 3 \Delta}{8\,(n-1)}
  + \Bigl( \sqrt{2}\, M + \tfrac{5}{16} \Bigr) \frac{\Delta}{j}
  \;\xrightarrow[\; n \uparrow,\ j \uparrow \;]{}\; 0 .
$$

We use the state-space quantization metric
\[
  \mathrm{dist} =
  \|(x - x', y - y')\|
  + \tfrac{d_{S^{1}}(\theta, \theta')}{2}
  + \tfrac{|v - v'| + |\omega - \omega'|}{4} .
\]
For evaluation, we use the window
$\Delta = 0.5$\,s, $j = 1$, $n = 4$, the pruning radius
$\rho = 0.05$, and the goal tolerance $\delta_{0} = 0.1$.

\paragraph{3.\ Car with trailer.}
The car with trailer operates with the state
$(x, y, \theta_{0}, \theta_{1})$, the car position, car
heading, and trailer heading, with wheelbase
$L = \tfrac{1}{4}$, hitch length $\ell = \tfrac{1}{2}$, and
hitch bound
$|\theta_{0} - \theta_{1}|_{S^{1}} < \tfrac{\pi}{4}$. The
control input
$u = (v, \phi) \in [-\tfrac{1}{10}, \tfrac{1}{2}] \times
[-\tfrac{\pi}{3}, \tfrac{\pi}{3}]$ is the rolling speed and
the steering angle. The dynamics obey
\[
  \begin{aligned}
  \dot x &= v \cos\theta_{0}, &
  \dot y &= v \sin\theta_{0}, \\
  \dot\theta_{0} &= \frac{v}{L} \tan\phi, &
  \dot\theta_{1} &= \frac{v}{\ell}
  \sin(\theta_{0} - \theta_{1}) .
  \end{aligned}
\]
The system is flat with the output the trailer position
$z = (x, y) - \ell\, (\cos\theta_{1}, \sin\theta_{1})$ and
orders $R = (3, 3)$, the defining relations
\[
  \begin{aligned}
  \theta_{1}
  &= \operatorname{atan2}\bigl( \pm\dot z_{y},\,
  \pm\dot z_{x} \bigr), &
  \theta_{0}
  &= \theta_{1} + \arctan
  \frac{\ell\, (\dot z \wedge \ddot z)}{\|\dot z\|^{3}}, \\
  v
  &= \frac{\pm\|\dot z\|}{\cos(\theta_{0} - \theta_{1})},
  &
  \phi
  &= \arctan \frac{L\, \dot\theta_{0}}{v} .
  \end{aligned}
\]

Let
$g = \bigl( v_{0},\, (\dot v_{\ell},
\phi_{\ell})_{\ell=1}^{j} \bigr)$. A command $w_{g}$ applies
the control signals $v(t), \phi(t)$ over the window
$[t_{0}, t_{j}]$, where
\[
  v(t) = v(t_{\ell-1}) + \dot v_{\ell}\, (t - t_{\ell-1}),
  \qquad
  \phi(t) = \phi_{\ell},
  \quad t \in I_{\ell},
\]
with
\[
  \begin{aligned}
  &v(t_{0}) \in
  G_{n_{0}}\bigl[-\tfrac{1}{10}, \tfrac{1}{2}\bigr], \quad
  \dot v_{\ell} \in G_{n_{v}}\bigl[-\tfrac{1}{2},
  \tfrac{1}{2}\bigr], \\
  &\phi_{\ell} \in G_{n_{\phi}}\bigl[-\tfrac{\pi}{3},
  \tfrac{\pi}{3}\bigr],
  \end{aligned}
\]
gated by the cheap rejection of commands with
$v(t) \notin \bigl[-\tfrac{1}{10}, \tfrac{1}{2}\bigr]$.
The flat input follows through the reduction,
\[
  \begin{aligned}
  \A(x)
  &\;=\;
  \Bigl\{\, w_{g} \;:\; g \in G_{n_{0}} \times
  \bigl( G_{n_{v}} \times G_{n_{\phi}} \bigr)^{j} \,\Bigr\}, \\
  w_{g}(t)
  &\;=\;
  \begin{pmatrix}
  \bigl( \ddot s - s \omega_{1}^{2} \bigr) \cos\theta_{1}
  - \bigl( 2 \dot s\, \omega_{1} + s \dot\omega_{1} \bigr)
  \sin\theta_{1} \\
  \bigl( \ddot s - s \omega_{1}^{2} \bigr) \sin\theta_{1}
  + \bigl( 2 \dot s\, \omega_{1} + s \dot\omega_{1} \bigr)
  \cos\theta_{1}
  \end{pmatrix},
  \end{aligned}
\]
where the trailer speed
$s = v \cos(\theta_{0} - \theta_{1})$, the trailer rate
$\omega_{1} = \dot\theta_{1}$, and their derivatives follow
from the dynamics under the command.

The command set provides the dispersion
$$
  \delta_{n, j}
  \;\le\;
  c(\Delta) \Bigl( \tfrac{1}{n_{0}} + \tfrac{1}{n_{v}}
  + \tfrac{1}{n_{\phi}} + \tfrac{M \Delta}{j} \Bigr)
  \;\xrightarrow[\; n \uparrow,\ j \uparrow \;]{}\; 0
$$
over $M$-smooth references with mild regularity
restrictions on the control signals.

We use the state-space quantization metric
\[
  \mathrm{dist} =
  \|(x - x', y - y')\|
  + \tfrac{d_{S^{1}}(\theta_{0}, \theta_{0}')
  + d_{S^{1}}(\theta_{1}, \theta_{1}')}{2} .
\]
For evaluation, we use the window
$\Delta = 0.3$\,s, $j = 1$,
$(n_{0}, n_{v}, n_{\phi}) = (6, 3, 4)$ amounting to
$|\A| = 52$ after the rejection, the pruning radius
$\rho = 0.06$, and the goal tolerance $\delta_{0} = 0.1$.

\paragraph{4.\ Quadrotor.}
The quadrotor operates with the state $(p, q, v, \omega)$, the
position $p \in \R^{3}$, attitude $q \in S^{3}$, velocity, and
body rate, with the mass $m$, inertia $J$, and mixer $B_{0}$
mapping the rotor forces to the total thrust and the torques
$(c, \tau)$, the control input $u \in [0, 1.3]^{4}$, the four
rotor thrusts in hover units, and the dynamics
\[
  \begin{aligned}
  \dot p &= v, &
  m \dot v &= c\, b_{3}(q) - m g e_{3}, \\
  \dot q &= \tfrac12\, q \otimes \omega, &
  J \dot\omega &= \tau - \omega \times J \omega .
  \end{aligned}
\]
The system is flat with the output $z = (p, \psi)$, the
position and the yaw, and orders $R = (4, 4, 4, 2)$. The
thrust, the attitude, the body rates, and the torque, which
in turn specify the controls, are recovered as
\[
  \begin{aligned}
  c &= m\, \|\ddot p + g e_{3}\|, \quad
  b_{3} = \frac{\ddot p + g e_{3}}{\|\ddot p + g e_{3}\|}, \\
  R &= \bigl[\, b_{2} \times b_{3},\; b_{2},\; b_{3}
  \,\bigr], \quad
  b_{2} = \frac{b_{3} \times e_{\psi}}{\| b_{3} \times
  e_{\psi} \|}, \\
  \omega &= R^{\top} \bigl( b_{3} \times \dot b_{3}
  + s\, b_{3} \bigr), \quad
  \dot b_{3} = \frac{m \bigl( I - b_{3} b_{3}^{\top} \bigr)\,
  \dddot p}{c}, \\
  s &= \frac{\dot\psi\, \langle b_{3}, e_{3} \rangle}
  {\| b_{3} \times e_{\psi} \|^{2}}
  - \frac{\langle \dot b_{3} \times e_{\psi},\,
  b_{2} \times b_{3} \rangle}{\| b_{3} \times e_{\psi} \|},
  \\
  \tau &= J\, \dot\omega + \omega \times J \omega,
  \end{aligned}
\]
with $e_{\psi} = (\cos\psi, \sin\psi, 0)$ the yaw direction,
$s$ the axial body rate, and $\dot\omega$ the derivative of
the third row, a function of the snap $p^{(4)}$ and
$\ddot\psi$.

The command set specification for the quadrotor is slightly
more involved and uses a geometric $SE(3)$
controller~\citep{lee2010geometric}. The controller tracks a
reference signal that specifies the thrust vector
$f = \ddot p + g e_{3}$ and the axial body rate
$s = \langle \omega,\, e_{3} \rangle$.

Let
$g_{\mathrm{ref}} = \bigl( f_{0},\, \dot f_{0},\,
\ddot f,\, r \bigr)$ and the
corresponding thrust signal be the quadratic ramp
\[
  f^{\mathrm{ref}}(t) = f_{0}
  + \dot f_{0}\, t
  + \tfrac{\ddot f}{2}\, t^{2},
  \qquad
  \dot f^{\mathrm{ref}}(t) = \dot f_{0}
  + \ddot f\, t,
\]
along with the axial body rate $r$. The reference
signal immediately determines the collective thrust by
projecting the reference on the current axis
$b_{3} = R(q)\, e_{3}$,
\begin{equation}
\label{eq:quad-thrust}
  c = m\, \mathrm{clamp}\bigl( \langle f^{\mathrm{ref}},\, b_{3}
  \rangle,\, F_{-},\, F_{+} \bigr).
\end{equation}
It remains to determine
the torque to specify the control signal
$u(t) = \mathrm{clip}\bigl( B_{0}^{-1}(c, \tau),\,
[0, 1.3]^{4} \bigr)$, the mixer demand projected to the rotor
box, corresponding to a command $w_{\mathrm{ref}}$. We follow the standard stabilizing cascade of the
PX4 autopilot~\citep{meier2015px4}, where the torque tracks
a demanded body rate as
\begin{equation}
\label{eq:quad-torque}
  \tau = J \bigl( k_{\omega} ( \omega^{\mathrm{ref}}
  - \omega )
  + \dot\omega^{\mathrm{ref}} \bigr)
  + \omega \times J \omega,
\end{equation}
where the body-rate reference is set to align the body axis
along the thrust direction
$b_{3}^{\mathrm{ref}} = f^{\mathrm{ref}} /
\| f^{\mathrm{ref}} \|$, clipped to the body-rate box:
\begin{equation}
\label{eq:quad-rate}
  \begin{aligned}
  \tilde\omega &= R(q)^{\top} \Bigl( k_{b}\, \theta\,
  \tfrac{b_{3} \times b_{3}^{\mathrm{ref}}}
  {\| b_{3} \times b_{3}^{\mathrm{ref}} \|}
  + b_{3}^{\mathrm{ref}} \times \dot b_{3}^{\mathrm{ref}}
  + r\, b_{3}^{\mathrm{ref}} \Bigr), \\
  \omega^{\mathrm{ref}} &= \mathrm{clip}\bigl(
  \tilde\omega,\, [-\bar\omega, \bar\omega]^{3} \bigr),
  \end{aligned}
\end{equation}
with $\theta = \arccos \bigl\langle b_{3},\,
b_{3}^{\mathrm{ref}} \bigr\rangle$, and the normalized axis
term set to $0$ whenever
$b_{3} \times b_{3}^{\mathrm{ref}} = 0$. Each $g_{\mathrm{ref}}$
is picked from
\[
  \begin{aligned}
  &f_{0} \in g e_{3}
  {+} G_{n_{0}}[-\bar a,\bar a]^{3}, \quad
  \dot f_{0} \in
  G_{n_{1}}[-\bar\jmath,\bar\jmath]^{3}, \\
  &\ddot f \in
  G_{n_{2}}[-\bar s,\bar s]^{3}, \quad
  r \in G_{n_{s}}[-\bar r,\bar r],
  \end{aligned}
\]
with $\bar a$, $\bar\jmath$, $\bar s$, and $\bar r$
respectively constraining the acceleration, jerk, snap, and
the axial body rate. The flat input follows along the
rollout as the snap and the yaw acceleration,
\[
  \A(x)
  =
  \bigl\{ w_{\mathrm{ref}} = ( p^{(4)}, \ddot\psi )
  :
  g_{\mathrm{ref}} = ( f_{0}, \dot f_{0},
  \ddot f, r ) \bigr\}.
\]

The command set is $\delta$-dispersive with
$$
  \delta_{n}
  \;\le\;
  c(M, \Delta) \Bigl( \tfrac{1}{n_{0}} + \tfrac{1}{n_{1}}
  + \tfrac{1}{n_{2}} + \tfrac{1}{n_{s}}
  + M \Delta
  \Bigr)
  \xrightarrow[n \uparrow,\ \Delta \downarrow]{} 0 .
$$

We use the state-space quantization metric
\[
  \mathrm{dist} = \|p - p'\|
  + \tfrac{\arccos | \langle q, q' \rangle |}{2}
  + \tfrac{\|v - v'\|}{10}
  + \tfrac{\|\omega - \omega'\|}{20} .
\]
For evaluation, we use the window
$\Delta = 0.25$\,s, the lateral grids
$G_{5}[-6, 6]$ and the vertical grid $G_{5}[-3, 2.9]$, the
gains $(k_{b}, k_{\omega}) = (12, 30)$, the pruning radius
$\rho = 0.35$, and the goal tolerance $\delta_{0} = 0.7$,
which amounts to $|\A| = 81$ after the thrust is constrained
to $(F_{-}, F_{+}) = (0.05,\, 1.3)\, g$.

\subsubsection{Evaluation}
\label{sec:benchmark}

Dynobench is a collection of kinodynamic motion planning
instances for each of the above platforms, specified by a tuple
$(x_{\mathrm{start}}, \mathcal{O}, x_{\mathrm{goal}},
\delta_{0})$ with $\mathcal{O} \subset \X$ the obstacle set.
Each task is an instance of
\eqref{eq:ocp} whose cost is the trajectory
duration and predicates are the obstacle clearance and terminal
goal proximity
\[
  \begin{aligned}
  g_{\mathrm{obs}}\bigl( x(\cdot) \bigr)
  &\;=\;
  \inf_{t \in [0, T]}\,
  \mathrm{dist}\bigl( x(t),\, \mathcal{O} \bigr) - r, \\
  g_{\mathrm{goal}}\bigl( x(\cdot) \bigr)
  &\;=\;
  \delta_{0} - \mathrm{dist}\bigl( x(T),\, x_{\mathrm{goal}}
  \bigr),
  \end{aligned}
\]
with $r \ge 0$ a clearance radius. We run
Algorithm~\ref{alg:bfs} along with the introduced command sets
and the termination condition
$$
  \min_{x' \in V_{k}}\,
  \mathrm{dist}\bigl( x',\, x_{\mathrm{goal}} \bigr)
  \;\le\; \delta_{0} .
$$

We benchmark DFT\textsuperscript{*} on the Dynobench against
iDb-A\textsuperscript{*}~\citep{ortizharo2024idbastar} with its
released motion primitives, as well as the two OMPL baselines
SST\textsuperscript{*} and RRT\textsuperscript{*}+TO with
their official dynoplan integrations. We run
iDb-A\textsuperscript{*} at its released configuration for
each platform, including its own discontinuity bound
$\delta_{0}$, and take its wall time to the first
post-optimized feasible trajectory. We report
the first solution costs and times for each of the planners in
Table~\ref{tab:dynobench}(a), the DFT\textsuperscript{*} wall
times at the embedded tier and, in blue, on the full RTX~4090.
Additionally, for
SST\textsuperscript{*} we report an anytime solution at the
$60$\,s budget. Extension~2 shows the DFT\textsuperscript{*}
first solutions for the trailer car and the quadrotor.

\begin{table*}[t]
\caption{Offline planning in static environments.
(a) First solutions on the Dynobench instances,
averaged over $5$ runs for each of the evaluated planners and
DFT\textsuperscript{*}. DFT\textsuperscript{*} wall times
reported use the Orin Nano equivalent compute tier ($8$ SMs),
with the bracketed blue walls scaling the same computation to
the full desktop RTX~4090;
the CPU-based planners run on an Intel i9-14900KF.
(b) Post-optimization, which both improves the trajectory and
repairs its terminal discontinuity, closing the goal tolerance
from $\delta_{0}$ to ${\sim}0$. For iDb-A\textsuperscript{*}
the optimizer additionally repairs the discontinuities its
search leaves between motion primitives, so its first solution
is already an optimizer output and we report that cost alone.
(c) DFT\textsuperscript{*} against the results reported by
FLASK~\citep{king2026akinopdf}, the blue walls again the full
RTX~4090, our two rows averaged over $5$ runs as in (a). The
reported cost is the path
length but we continue to solve for minimum time, consistent
with FLASK's LQMT connections.
The search variants of Section~\ref{sec:variants} are
evaluated on the same instances as presented in
Table~\ref{tab:variants} of Appendix~\ref{app:variants}. The
cost therein remains the same but the computation times
improve due to the narrower search trees.}
\label{tab:dynobench}
\centering
\footnotesize
\setlength{\tabcolsep}{3pt}
\begin{tabular*}{\textwidth}{|@{\extracolsep{\fill}\hspace{\tabcolsep}}ll|cc|cc|ccc|cc|}
\hline
\textbf{(a)} & & \multicolumn{2}{c|}{DFT\textsuperscript{*} (ours)} &
\multicolumn{2}{c|}{iDb-A\textsuperscript{*}} &
\multicolumn{3}{c|}{SST\textsuperscript{*}} &
\multicolumn{2}{c|}{RRT\textsuperscript{*}+TO} \\
\hline
system & task &
cost (s) & wall (ms) &
cost (s) & wall (ms) &
cost (s) & wall (ms) & cost@60s &
cost (s) & wall (ms) \\
\hline
\multirow{3}{*}{\shortstack[l]{Unicycle\\1st order}}
& bugtrap & \cgood{21.1} & \wmsb{202} \textcolor{grayblue}{($32$)}
& $22.5$ & \wms{269}
& \cbad{72.7} & \wms{325} & $24.4$
& $27.8$ & \wms{5171} \\
 & kink & \cgood{14.0} & \wmsb{61} \textcolor{grayblue}{($18$)}
& $21.4$ & \wms{147}
& \cbad{49.7} & \wms{106} & $15.7$
& $24.5$ & \wms{418} \\
 & parallelpark & \cgood{3.5} & \wmsb{4} \textcolor{grayblue}{($5$)}
& $3.9$ & \wms{5}
& \cbad{12.8} & \wms{17} & $3.2$
& $3.7$ & \wms{6} \\
\hline
\multirow{3}{*}{\shortstack[l]{Unicycle\\2nd order}}
& bugtrap & \cgood{24.5} & \wms{3218} \textcolor{grayblue}{($310$)}
& $25.3$ & \wmsb{687}
& $65.9$ & \wms{5696} & $58.5$
& $40.7$ & \wms{8156} \\
 & kink & $18.9$ & \wms{1116} \textcolor{grayblue}{($107$)}
& \cgood{18.1} & \wmsb{235}
& \cbad{95.6} & \wms{6856} & $88.2$
& $29.1$ & \wms{1399} \\
 & parallelpark & $6.7$ & \wms{11} \textcolor{grayblue}{($4$)}
& \cgood{5.8} & \wmsb{8}
& \cbad{20.2} & \wms{991} & $13.0$
& $6.2$ & \wms{29} \\
\hline
\multirow{3}{*}{\shortstack[l]{Car with\\a trailer}}
& bugtrap & $19.8$ & \wms{2713} \textcolor{grayblue}{($195$)}
& \cgood{19.5} & \wmsb{425}
& $51.1$ & \wms{11921} & $49.7$
& $40.1$ \textcolor{dkred}{($2/5$)} & \wms{4637} \\
 & kink & \cgood{15.9} & \wms{1326} \textcolor{grayblue}{($95$)}
& $22.9$ & \wmsb{634}
& \cbad{64.2} \textcolor{dkred}{($4/5$)} & \wms{2666} & $59.9$
& $33.5$ & \wms{3247} \\
 & parallelpark & $3.7$ & \wmsb{8} \textcolor{grayblue}{($3$)}
& $8.1$ & \wms{33}
& \cbad{13.1} \textcolor{dkred}{($1/5$)} & \wms{15053} & $13.1$
& \cgood{3.5} & \wms{20} \\
\hline
\multirow{4}{*}{\shortstack[l]{Quadrotor\\
  $u = (u_{i})_{i=1}^{4}$}}
& block & \cgood{2.50} & \wms{2443} \textcolor{grayblue}{($168$)}
& $6.58$ & \wms{753}
& -- \textcolor{dkred}{($0/5$)} & -- & --
& $4.82$ & \wmsb{720} \\
 & window & \cgood{2.00} & \wmsb{393} \textcolor{grayblue}{($30$)}
& $4.62$ & \wms{1543}
& -- \textcolor{dkred}{($0/5$)} & -- & --
& $2.76$ \textcolor{dkred}{($1/5$)} & \wms{4024} \\
 & inversion & \cgood{3.25} & \wms{420} \textcolor{grayblue}{($30$)}
& $6.14$ & \wmsb{365}
& -- \textcolor{dkred}{($0/5$)} & -- & --
& -- \textcolor{dkred}{($0/5$)} & -- \\
& inversion obs. & \cgood{3.80} & \wms{2978} \textcolor{grayblue}{($250$)}
& $5.23$ & \wmsb{697}
& -- \textcolor{dkred}{($0/5$)} & -- & --
& -- \textcolor{dkred}{($0/5$)} & -- \\
\hline
\end{tabular*}

\vspace{0.9em}

\begin{minipage}[t]{0.40\textwidth}
\centering
\setlength{\tabcolsep}{2pt}
\renewcommand{\arraystretch}{1.24}
\begin{tabular}{|ll|cc|c|}
\hline
\textbf{(b)} & & \multicolumn{2}{c|}{DFT\textsuperscript{*} (ours)} &
iDb-A\textsuperscript{*} \\
\hline
system & task &
opt.\ (s) & wall (ms) &
opt.\ (s) \\
\hline
\multirow{3}{*}{\shortstack[l]{Unicycle\\1st order}}
& bugtrap & \cgood{20.9} & \wms{38} & $22.5$ \\
 & kink & \cgood{13.3} & \wms{30} & $21.4$ \\
 & parallelpark & \cgood{3.2} & \wms{6} & $3.9$ \\
\hline
\multirow{3}{*}{\shortstack[l]{Unicycle\\2nd order}}
& bugtrap & \cgood{25.1} & \wms{48} & $25.3$ \\
 & kink & \cgood{17.8} & \wms{57} & $18.1$ \\
 & parallelpark & $5.9$ & \wms{13} & \cgood{5.8} \\
\hline
\multirow{3}{*}{\shortstack[l]{Car with\\a trailer}}
& bugtrap & \cgood{19.0} & \wms{30} & $19.5$ \\
 & kink & \cgood{13.5} & \wms{68} & $22.9$ \\
 & parallelpark & \cgood{3.6} & \wms{12} & $8.1$ \\
\hline
\multirow{4}{*}{Quadrotor}
& block & \cgood{2.88} & \wms{245} & $6.58$ \\
 & window & \cgood{2.29} & \wms{205} & $4.62$ \\
 & inversion & \cgood{2.60} & \wms{206} & $6.14$ \\
 & inversion obs. & \cgood{3.55} & \wms{416} & $5.23$ \\
\hline
\end{tabular}
\end{minipage}\hfill
\begin{minipage}[t]{0.60\textwidth}
\centering
\setlength{\tabcolsep}{1.5pt}
\renewcommand{\arraystretch}{1.25}
\begin{tabular}{|l|cc|cc|cc|cc|}
\hline
\textbf{(c)} & \multicolumn{4}{c|}{Unicycle 1st order} &
\multicolumn{4}{c|}{Quadrotor $u{=}(c,\tau_{1},\tau_{2},\tau_{3})$} \\
\hline
 & \multicolumn{2}{c|}{bugtrap} &
\multicolumn{2}{c|}{wall} &
\multicolumn{2}{c|}{window} &
\multicolumn{2}{c|}{block} \\
\hline
planner & len (m) & wall (ms) & len (m) & wall (ms)
& len (m) & wall (ms) & len (m) & wall (ms) \\
\hline
\shortstack[l]{\rule{0pt}{2.6ex}FLASK\\RRTConnect} & $12.39$ & \wmsb{6.0}
& $4.82$ & \wmsb{0.6}
& $5.49$ & \wmsb{0.5}
& $8.08$ & \wmsb{0.9} \\
\shortstack[l]{\rule{0pt}{2.6ex}FLASK SST\textsuperscript{*}\\(BVP-aug.)} & $12.51$ & \wms{48.6}
& $5.24$ & \wms{11.0}
& $6.28$ & \wms{22.6}
& $8.61$ & \wms{1.4} \\
\shortstack[l]{\rule{0pt}{2.6ex}FLASK SST\textsuperscript{*}\\(SIMD-only)} & $12.27$ & \wms{407}
& $5.27$ & \wms{29.5}
& $6.96$ & \wms{629}
& $9.56$ & \wms{4752} \\
\shortstack[l]{\rule{0pt}{2.6ex}FLASK RRT\\(DP-based)} & $12.25$ & \wms{565}
& $5.61$ & \wms{24.8}
& $7.98$ & \wms{2631}
& $8.75$ & \wms{7602} \\
\hline
iDb-A\textsuperscript{*} & $11.51$ & \wms{400}
& $3.89$ & \wms{100}
& $5.12$ & \wms{2700}
& $7.94$ & \wms{4600} \\
SST\textsuperscript{*} & $12.38$ & \wms{200}
& $4.97$ & \wms{100}
& $6.72$ & \wms{46300}
& $9.32$ & \wms{45800} \\
Kino-PAX & $13.77$ & \wms{5.8}
& $5.28$ & \wms{2.6}
& \cbad{16.41} & \wms{178.6}
& $10.36$ & \wms{21.4} \\
\hline
DFT\textsuperscript{*} (ours) & \cgood{9.77} & \wms{73} \textcolor{grayblue}{($7$)}
& \cgood{3.21} & \wms{3.5} \textcolor{grayblue}{($1$)}
& \cgood{5.06} & \wms{406} \textcolor{grayblue}{($30$)}
& $7.75$ & \wms{4335} \textcolor{grayblue}{($270$)} \\
DFT\textsuperscript{*}-Static & $9.97$ & \wms{11} \textcolor{grayblue}{($3$)}
& \cgood{3.21} & \wms{2.6} \textcolor{grayblue}{($1$)}
& $5.18$ & \wms{199} \textcolor{grayblue}{($16$)}
& \cgood{7.61} & \wms{1564} \textcolor{grayblue}{($104$)} \\
\hline
\end{tabular}
\end{minipage}
\end{table*}

Dynobench provides a post-hoc optimization step run via
Crocoddyl~\citep{mastalli2020crocoddyl}, necessary to iron out
the discontinuities in the initial solution of
iDb-A\textsuperscript{*}. As a result we also report
post-optimization costs, seeding the optimizer with each
planner's first solution, in Table~\ref{tab:dynobench}(b).
DFT\textsuperscript{*} solutions do not contain
discontinuities, but the seeded post-hoc optimization does
help compress the terminal goal tolerance nearly to $0$ in
addition to some performance benefits. Notably, the
optimization does not erase the gap at the initial solution
seeds and the respective advantage is often maintained.

\begin{figure}[t]
  \centering
  \includegraphics[width=0.48\linewidth]{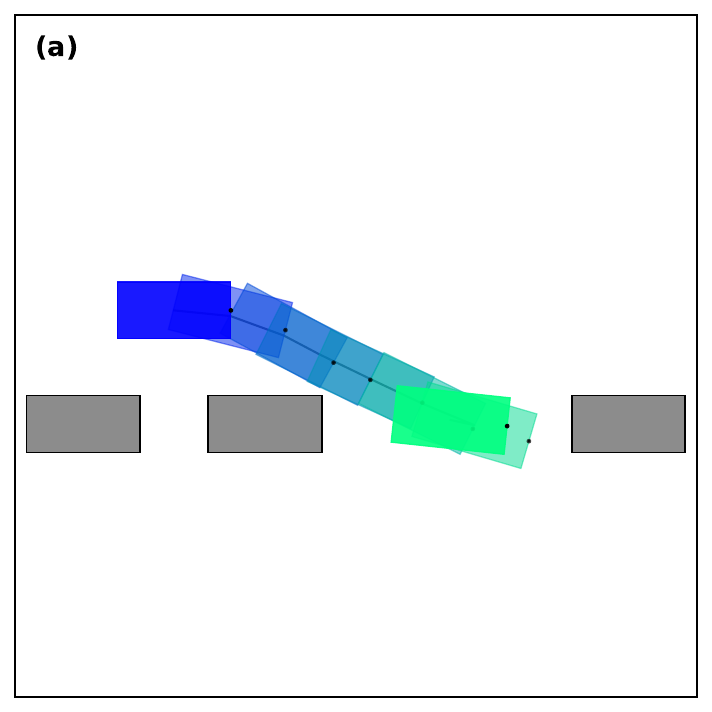}
  \hfill
  \includegraphics[width=0.48\linewidth]{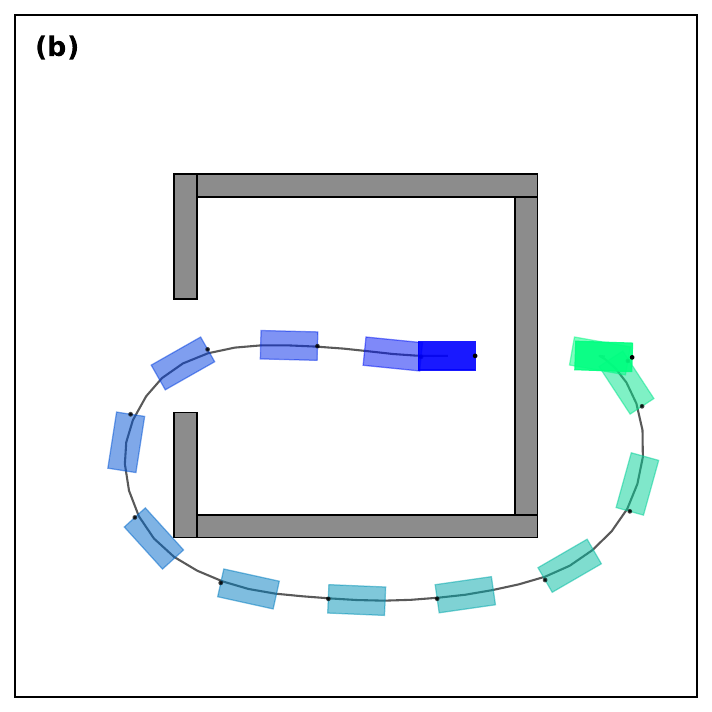}
  \\[0.4em]
  \includegraphics[width=0.48\linewidth]{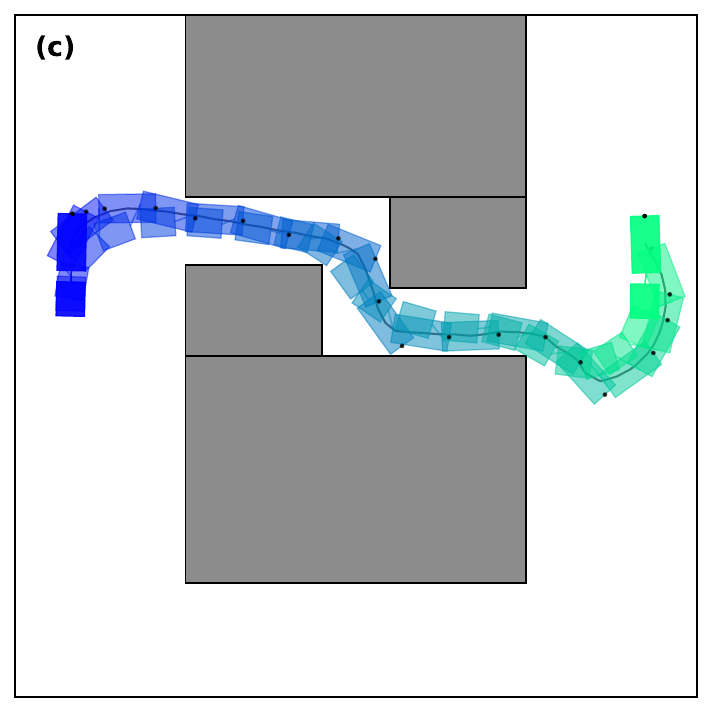}
  \hfill
  \includegraphics[width=0.48\linewidth]{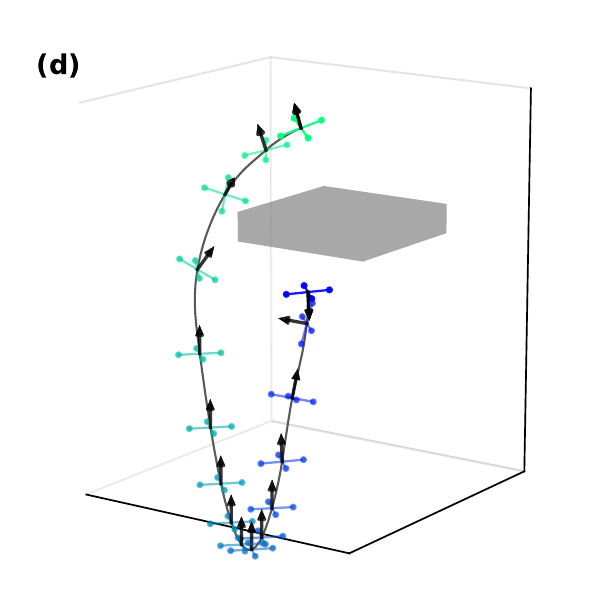}
  \caption{DFT\textsuperscript{*} first solutions, pose
  snapshots drawn along the trajectory from start (blue) to
  goal (green): (a) first-order unicycle on parallel park, (b)
  second-order unicycle on bugtrap, (c) car with trailer
  through the kink passage, (d) quadrotor recovery with an
  obstacle.}
  \label{fig:dynobench_solutions}
\end{figure}

In Table~\ref{tab:dynobench}(c), we report the performance of
DFT\textsuperscript{*} in comparison with the results reported
in FLASK~\citep{king2026akinopdf}. For this table, we match the
platform configurations of FLASK, the unicycle constrained to
the nonnegative speed in $[0,1]$ and the turn rate in
$[-1.5,1.5]$, while the quadrotor is commanded in total thrust
and torque rather than the per-rotor thrusts, the catalog
tracker clamping its torque demand to their torque box in
place of the per-rotor clamp. In this parametrization random
control propagation is considerably more well behaved, which
is why SST\textsuperscript{*} solves some of the quadrotor
tasks.

In Table~\ref{tab:variants} of Appendix~\ref{app:variants}, we
report the search variants of Section~\ref{sec:variants} on
the same instances. The Dynobench instances pose static
kinodynamic planning problems, so that both the
A\textsuperscript{*} beam ordering and the static dominance
pruning apply, and together they reduce the planning wall time
by up to $44\times$ at the same solution costs.

To summarize our findings across
Tables~\ref{tab:dynobench}(a), (b), and (c),
DFT\textsuperscript{*} finds solutions closest to the optimum
in largely competent wall time budgets while deployed on
embedded-tier parallel processing. These wall times shrink by
up to $15\times$ when we allow dedicated computation to scale
all the way to what is available on the desktop RTX~4090.

\subsection{Online Planning in Dynamic Environments}
\label{sec:wwdft}

In this section, we present an online implementation of
DFT\textsuperscript{*} to demonstrate real-time performance of
dispersive forward tree search. We address the problem of
safety-assured quadrotor flight in the presence of obstacles
with unknown dynamics. Our implementation involves a receding-horizon
setup where DFT\textsuperscript{*} is deployed to iteratively
solve the fixed-horizon OCP~\eqref{eq:online_ocp} at a regular
interval $\Delta$. The online algorithm we call
WarpWalker-DFT\textsuperscript{*} (WWDFT\textsuperscript{*}) is
presented below as Algorithm~\ref{alg:online}.

\begin{figure*}[t]
  \centering
  \includegraphics[width=\textwidth]{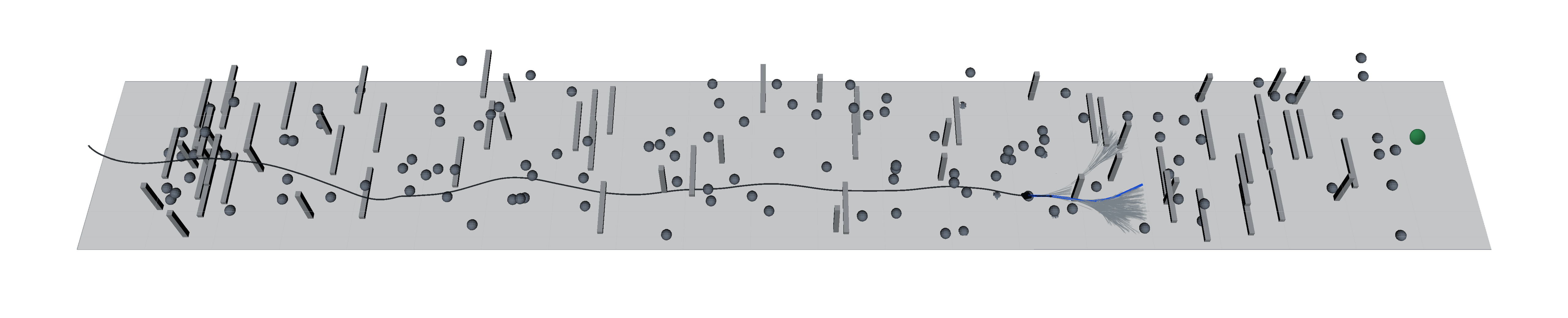}
  \caption{WWDFT\textsuperscript{*} mid-flight in the hard
  dynamic environment ($n_{\mathrm{st}} = 70$,
  $n_{\mathrm{dyn}} = 130$). The quadrotor flies from the start
  to the goal (green, right) through static
  cuboids and moving spheres, trailing its executed trajectory
  in black.}
  \label{fig:wwdft_snapshot}
\end{figure*}

\begin{table*}[t]
\centering
\caption{Real-time planning in dynamic environments.
Collision-free success rates and travel times are
reported as mean over 5 runs. Relative worst-case violation
of each of the imposed $v$/$a$/$j$ constraints is also
reported, structurally absent for WWDFT\textsuperscript{*}.}
\label{tab:wwdft}
\footnotesize
\setlength{\tabcolsep}{3.4pt}
\begin{tabular}{l l l cc cc cc ccc}
\hline
 & & & \multicolumn{2}{c}{$n_{\mathrm{st}} {=} 18$,
   $n_{\mathrm{dyn}} {=} 32$}
 & \multicolumn{2}{c}{$n_{\mathrm{st}} {=} 35$,
   $n_{\mathrm{dyn}} {=} 65$}
 & \multicolumn{2}{c}{$n_{\mathrm{st}} {=} 70$,
   $n_{\mathrm{dyn}} {=} 130$}
 & \multicolumn{3}{c}{Max violation (\%)} \\
\cline{4-5}\cline{6-7}\cline{8-9}\cline{10-12}
Method & Platform & Constraints
 & Succ. & $T_{\mathrm{trav}}$ (s)
 & Succ. & $T_{\mathrm{trav}}$ (s)
 & Succ. & $T_{\mathrm{trav}}$ (s)
 & $v$ & $a$ & $j$ \\
\hline
FAPP & Triple integrator
 & $\| v \|_{2}{\le}8$, $\| a \|_{2}{\le}12.2$,
   $\| j \|_{2}{\le}60$
 & \cmid{60\%} & 21.06 & \cmid{60\%} & 22.47
 & \cbad{40\%} & 20.02 & 39.3 & 17.0 & 2.1 \\
MIGHTY & Triple integrator
 & $\| v \|_{2}{\le}8$, $\| a \|_{2}{\le}12.2$,
   $\| j \|_{2}{\le}60$
 & \cgood{100\%} & 18.15 & \cgood{100\%} & 18.61
 & \cgood{100\%} & 19.31 & 0 & 0 & 0.03 \\
WWDFT\textsuperscript{*} & Triple integrator
 & $\| v \|_{2}{\le}8$, $\| a \|_{2}{\le}12.2$,
   $\| j \|_{2}{\le}60$
 & \cgood{100\%} & \cgood{15.10}
 & \cgood{100\%} & \cgood{15.35}
 & \cgood{100\%} & \cgood{15.20}
 & 0 & 0 & 0 \\
\hline
WWDFT\textsuperscript{*} & 3D quadrotor
 & $\| v \|_{2}{\le}8$, $| \omega |{\le}\omega_{\max}$,
   $\theta{\le}45^{\circ}$
 & \cgood{100\%} & \cgood{14.25} & \cgood{100\%} & \cgood{14.50}
 & \cgood{100\%} & \cgood{14.60} & 0 & -- & -- \\
\hline
\end{tabular}
\end{table*}

At each replanning step, we solve the following OCP from the
current state:
\begin{equation}
\label{eq:online_ocp}
  \begin{aligned}
  \min_{u(\cdot)} \quad
  & \mathrm{dist}\bigl( x( K \Delta ),\, x_{\mathrm{goal}}
    \bigr)
    \\[0.2em]
  \text{s.t.} \quad
  & \dot{x}(t) = f\bigl(x(t), u(t)\bigr), \\[0.2em]
  & \mathrm{dist}\bigl( x(t),\, o_{j} \bigr)
    \ge r + v_{\max}^{j}\, t
    \quad \forall\, j,\ \forall\, t \in [0, K \Delta],
    \\[0.2em]
  & x(0) = x_{\mathrm{now}},
  \end{aligned}
\end{equation}
with
$o_{j}$ the current position of the $j$th
obstacle, $v_{\max}^{j}$ its speed bound, zero for a static
obstacle, and $r$ the clearance radius. We denote
by $\preds$ the set of predicates of OCP~\eqref{eq:online_ocp},
and by $\preds_{k}$ its restriction to the window
$[k\Delta, (k+1)\Delta]$.
Further, for nodes $x$, $y$ of the forward tree, we write
$\tra_{x} \subseteq \tra_{y}$ to denote that the trajectory
of $y$ extends the trajectory of $x$, that is, $y$ is a tree
descendant of $x$.

\begin{algorithm}[!htb]
\caption{WWDFT\textsuperscript{*} --- receding-horizon
DFT\textsuperscript{*}}
\label{alg:online}
\begin{algorithmic}
\definecolor{kwfor}{HTML}{1F77B4}
\definecolor{kwif}{HTML}{E8710A}
\definecolor{kwrep}{HTML}{8E44AD}
\definecolor{kwend}{HTML}{7F7F7F}
\algrenewcommand\algorithmicforall{{\color{kwfor}\textbf{for all}}}
\algrenewcommand\algorithmicif{{\color{kwif}\textbf{if}}}
\algrenewcommand\algorithmicthen{}
\algrenewcommand\algorithmicend{{\color{kwend}\textbf{end}}}
\algrenewcommand\algorithmicrepeat{{\color{kwrep}\textbf{repeat}}}
\algrenewcommand\algorithmicuntil{{\color{kwrep}\textbf{until}}}
\algrenewcommand\algorithmicfor{{\color{kwfor}\textbf{for}}}
\algrenewcommand\algorithmicdo{}
\Repeat\ every $\Delta$:
  \State $V_{0} \gets \{ x_{\mathrm{now}} \}$
  \State {\color{kwend}$\triangleright$ grow a certified tree
    to depth $K$ (Algorithm~\ref{alg:bfs})}
  \For{$k = 0, \dots, K - 1$:}
    \State $\widehat{V} \gets
      \left\{ \Phi_{a}( x )
      \middle|
      \begin{gathered}
      x \in V_{k},\ a \in \A( x ), \\
      g\bigl( z^{x}_{a} \bigr) \ge 0
      \ \ \forall\, g \in \preds_{k}
      \end{gathered}
      \right\}$
    \State $V_{k+1} \gets
      \left\{ x' \in \widehat{V}
      \middle|
      \begin{gathered}
      \Sigma( x' ) \le \Sigma( x'' )
      \ \forall x'' \in \widehat{V} \\
      \text{with}\
      \mathrm{cell}( x'' ) = \mathrm{cell}( x' )
      \end{gathered}
      \right\}$
  \EndFor
  \If{$V_{1} = \emptyset$:}
    \State declare failure
  \EndIf
  \State {\color{kwend}$\triangleright$ commit to the
    deepest best trajectory}
  \State $k^{*} \gets \max\,\{\, k : V_{k} \neq \emptyset \,\}$
  \State $y^{+} \gets y \in V_{k^{*}} :
    \mathrm{dist}\bigl( y,\, x_{\mathrm{goal}} \bigr)
    \le \mathrm{dist}\bigl( V_{k^{*}},\, x_{\mathrm{goal}} \bigr)$
  \State $x_{\mathrm{now}} \gets x_{1} \in V_{1} :
    \tra_{x_{1}} \subseteq \tra_{y^{+}}$
\Until{$\mathrm{dist}\bigl( x_{\mathrm{now}},\,
  x_{\mathrm{goal}} \bigr) \le \delta_{0}$}
\end{algorithmic}
\end{algorithm}

We formalize the safety guarantee for the committed
trajectories of Algorithm~\ref{alg:online}.

\begin{proposition}[Safety assurance]
\label{prop:safety}
Suppose that \\
(i) each obstacle moves no faster than its bound
$v_{\max}^{j}$, \\
(ii) the static map and the centers $o_{j}(t)$ are known at
time $t$, \\
(iii) each committed trajectory is executed exactly. \\
Then the executed flight by Algorithm~\ref{alg:online}
maintains the clearance $r$ from every obstacle until the goal
is reached or a failure is declared.
\end{proposition}

Note that WWDFT\textsuperscript{*} does \emph{not} guarantee
recursively feasible planning, and failure is possible whenever
the certified tree empties, $V_{1} = \emptyset$, in which case
WWDFT\textsuperscript{*} declares failure. Recursive
feasibility is in general impossible to guarantee for an
unbounded-horizon problem with unknown obstacle dynamics. A
larger planning horizon $K$ adds robustness to the
implementation at the cost of some added conservatism.

\subsubsection{Benchmarking in Dynamic Environments}

We compare the real-time performance of
WWDFT\textsuperscript{*} in simulation against two
state-of-the-art trajectory planners for dynamic environments,
FAPP~\citep{lu2025fapp} and MIGHTY~\citep{kondo2026mighty}.
Trajectory planners often adopt a reduced-order
$k$-integrator model for quadrotor flight and rely on
downstream tracking for physical execution. We therefore run
two versions of WWDFT\textsuperscript{*}. In the first, we use
a reduced-order model with jerk input under kinodynamic
constraints matched to FAPP and MIGHTY. In the second, we plan
directly on the full quadrotor. The quadrotor model's
acceleration authority is significantly lesser than the
reduced-order model's, while its jerk authority is limited
only by rotor lag, an electromechanical time constant of tens
of milliseconds, and as a result it produces jerk well beyond
the reduced-order bound.

The planning task is to fly a point robot from
$p_{\mathrm{start}} = (0, 0, 3)$ to the goal
$p_{\mathrm{goal}} = (100, 0, 3)$ m through a
$110 \times 12 \times 5$ m corridor populated by $50$, $100$,
and $200$ obstacles on the easy, medium, and hard
environments. Of the obstacles, $35\%$ are static axis-aligned
cuboids, pillars and bars of $0.4$ m thickness and $4$ m
length, placed randomly through the environment. The remaining
are spheres of radius $0.4$ m following smooth nonlinear
periodic trajectories at speeds up to $v_{\max}^{j} = 0.7$
m/s. The map at any instant is known to each planner but the
trajectories of the dynamic obstacles remain unknown. MIGHTY
and FAPP receive exact current obstacle centers and velocities
online and extrapolate at constant velocity, while
WWDFT\textsuperscript{*} receives centers only and certifies
against the isotropic reachable ball of the OCP above. We run
five repetitions per environment for each planner.

\medskip\noindent
\textbf{Reduced-Order Model and Jerk Sampling:}
The reduced-order dynamical system is the trivially flat
triple integrator with the state $x = (p, v, a)$ and the
three-dimensional jerk input,
\[
  \dot p = v, \qquad \dot v = a, \qquad \dot a = j,
\]
flat with the output $z = p$ and orders $R = (3, 3, 3)$, the
jerk being itself the flat input. For this system the lattice
construction of Theorem~\ref{thm:realization} deploys
conveniently, up to pullbacks for the acceleration and
velocity constraints. To this end, at each state $x$ we use
the command set
\[
  \A( x )
  \;=\;
  \bigl\{\, j \in G_{5}\bigl( J( x ) \bigr) :
  \| j \|_{2} \le \bar\jmath \,\bigr\},
\]
where $G_{5}( J( x ) )$ places the five-level lattice
$G_{5}[l, u]$ of Section~\ref{sec:examples} on each axis
interval of the box $J( x )$, the state-conditioned window
\[
  J( x )
  \;=\;
  [-\bar\jmath, \bar\jmath]^{3}
  \,\cap\,
  \frac{[-\bar a, \bar a]^{3} - a}{\Delta}
  \,\cap\,
  \frac{[-\bar v, \bar v]^{3} - v - a \Delta}
  {\Delta^{2} / 2} .
\]
We use the state quantization metric
$\| p - p' \| + \tfrac{1}{2}\, \| v - v' \|
+ \tfrac{1}{8}\, \| a - a' \|$ and prune at the radius
$\rho = 1$.

\medskip\noindent
\textbf{Sampling for the 3D Quadrotor:}
The full-order version plans on the quadrotor of
Section~\ref{sec:examples}, with the same reference
construction and geometric tracking loop, re-instantiated on
the Holybro X500 platform at a thrust-to-weight ratio of
$1.69$ with asymmetric first-order motor lag at time constants
of $12.5$ and $25$ ms. The held force references are sampled
over the acceleration box of the platform envelope, laterally
$G_{5}[-9.8, 9.8]$ and vertically $G_{5}[-8.2, 3.7]$
m/s$^{2}$, and filtered to the combined thrust annulus
$[0.37, 16.56]$ m/s$^{2}$, which amounts to $|\A| = 101$. We
prune at the radius $\rho = 1$ under the state-space metric of
Section~\ref{sec:examples}.

Throughout we use the constraints
$\| v \|_{2} \le 8$ m/s, $\| a \|_{2} \le 12.2$ m/s$^{2}$, and
$\| j \|_{2} \le 60$ m/s$^{3}$ across the two baselines and
the $k$-integrator WWDFT\textsuperscript{*}, which enforces
them as hard rollout rejections. Whereas the baselines include
them as penalties in their trajectory optimizers. For the 3D
quadrotor we use the per-axis body-rate constraints
$(3.84, 3.84, 3.49)$ rad/s and cap the body tilt at
$45^{\circ}$, as recommended by the autopilot, which limits
the achievable acceleration to $12.2$ m/s$^{2}$ at full
collective thrust in the maximizing direction and lesser
otherwise.

We report our experimental results in
Table~\ref{tab:wwdft}. A collision-free arrival within $0.75$
m of the goal counts as success. Travel time reports the
interval between goal issue and the arrival time, so that
initial planning time is included for every method.
WWDFT\textsuperscript{*} is run at $\Delta = 250$ ms, $K = 5$
across all runs. We run
MIGHTY with its default structure and optimizer settings;
only static/dynamic collision weights are increased
$10\times$.
We evaluate FAPP at its supplied algorithmic defaults,
changing only the task-defining quantities, namely the
prescribed $v$/$a$/$j$ bounds, the clearances representing
the benchmark sphere and the point robot, and its native
ideal-execution interface. Additional parameter tuning did
not conclusively help. All of the evaluated methods make it
to the goal region each trial. FAPP's executed trajectories
however penetrate a moving sphere in $7$ of $15$ trials and
exceed at least one of the requested differential bounds in
$10$ of $15$. WWDFT\textsuperscript{*} demonstrates
safety-assured real-time flight
(Proposition~\ref{prop:safety}) across all $30$ runs, and
reports the fastest aggregate travel times. Extension~1
shows a complete WWDFT\textsuperscript{*} episode in the
hard dynamic environment.

\section{Conclusion}
\label{sec:conclusion}

We have developed a propagation-based kinodynamic planner with
deterministic finite-sample near-optimality guarantees for the
class of differentially flat nonlinear systems. Our theoretical
results show that a forward tree of locally dispersive control
commands contains a near-optimal trajectory which
asymptotically converges to the optimum at a rate certified by
the local dispersion. We also developed an efficient search
algorithm DFT\textsuperscript{*} that iteratively prunes the forward tree
alongside its growth to avoid the exponential cost of naive
dispersive search while maintaining near-optimality. We
evaluated DFT\textsuperscript{*} on offline planning problems across a diversity
of platforms using our catalog of dispersive command sets
designed for sample efficiency at low sample counts and show
that it significantly improves upon the performance of other
state-of-the-art algorithms. We additionally developed a
receding-horizon implementation WWDFT\textsuperscript{*} and demonstrated similar
performance improvements for online navigation in dynamic,
cluttered environments.

A key limitation of our approach is that although dispersive
sampling is provably realizable by the mechanism of
Theorem~\ref{thm:realization}, this mechanism is inefficient at
low sample counts due to the nonlinear relationship between the
control input, the system state, and the flat input. As a
result, a careful utilization of additional structure of the
control system to design efficient command sets is necessary
for deployment. However, as we have shown across a wide range
of platforms cataloged in this work, this is doable without
much analytic trouble. Additionally, DFT\textsuperscript{*} has been designed with a
breadth-first search structure suitable for deployment on
parallel processors such as GPUs but unsuitable for a naive
deployment on CPUs without additional low-level programming
that may plausibly leverage vectorization and fine-grained
parallelism.

\bibliographystyle{SageH}
\bibliography{refs}

\appendix

\section{Proofs}
\label{app:proofs}

The following lemma, useful in our proofs, shows that
dispersion in the supremum norm at the top-order flat derivatives
bounds the dispersion in the canonical distance of
Definition~\ref{def:dist} on the trajectory space.

\begin{lemma}[Top-order reduction]
\label{lem:toporder}
Let $z, z'$ be flat trajectories on $[0, 1]$ with
$\jet z(0) = \jet z'(0)$. Then
\begin{equation}
\label{eq:toporder}
  \begin{aligned}
  &\sup_{t \in [0,1]}
  \bigl| \jet z(t) - \jet z'(t) \bigr|_{\infty}
  \\
  &\hspace{2em}\le
  \max_{1 \le i \le d}\; \sup_{t \in [0,1]}
  \bigl| z_{i}^{(r_{i}-1)}(t)
  - z'^{(r_{i}-1)}_{i}(t) \bigr| .
  \end{aligned}
\end{equation}
\end{lemma}

\begin{proof}
Fix $i$ and $0 \le k < r_{i} - 1$; write $g := z_{i} - z'_{i}$ and
$m := r_{i} - 1 - k$. Then
\[
  g^{(k)}(0) = g^{(k+1)}(0) = \cdots = g^{(r_{i}-1)}(0) = 0 ,
\]
so
\[
  g^{(k)}(t)
  \;=\;
  \frac{1}{(m-1)!}
  \int_{0}^{t} (t - s)^{m-1}\, g^{(r_{i}-1)}(s)\, \mathrm{d}s ,
\]
whence
\[
  \begin{aligned}
  \bigl| g^{(k)}(t) \bigr|
  &\;\le\;
  \frac{t^{m}}{m!}\,
  \sup_{[0,1]} \bigl| g^{(r_{i}-1)} \bigr| \\
  &\;\le\;
  \sup_{[0,1]} \bigl| g^{(r_{i}-1)} \bigr| ,
  \qquad t \in [0, 1] . \qedhere
  \end{aligned}
\]
\end{proof}

\subsection{Proof of Proposition~\ref{prop:floor} (Existence and
Complexity)}
\label{app:realization}

Fix an initial state $x \in \X$ and a dispersion
$\delta \in (0, M]$. Let $v$ be an admissible control on $[0, 1]$
with $M$-smooth trajectory
and $z_{v}$ be its flat trace starting at $x$. Denote the tuple of
the top-order derivatives of such traces and their collection as
\[
  \begin{aligned}
  y_{v}
  &\;:=\;
  \bigl( z_{v,i}^{(r_{i}-1)} \bigr)_{i=1}^{d} , \\
  \mathcal{Y}(x)
  &\;:=\;
  \bigl\{\, y_{v} \;:\;
  v \text{ admissible on } [0, 1] \text{ from } x , \\
  &\hspace{4.5em}
  \text{trajectory $M$-smooth} \,\bigr\} .
  \end{aligned}
\]
By jet observability (Assumption (iv),
Section~\ref{sec:assumptions}), given the $M$-smoothness of the
trajectory (Definition~\ref{def:msmooth}), each component of the
tuple $y_{v}$ obeys
\[
  \begin{aligned}
  y_{v,i}(0)
  &\;=\; \Xi_{i, r_{i}-1}(x) , \\
  | \dot{y}_{v,i} |
  &\;=\; | w_{v,i} | \;\le\; M
  \quad \text{a.e. on } [0, 1] .
  \end{aligned}
\]
Therefore the translated collection is contained within a product
of $d$ classes of $M$-Lipschitz scalar curves vanishing at $0$,
\[
  \begin{aligned}
  &\mathcal{Y}(x)
  - \bigl( \Xi_{i, r_{i}-1}(x) \bigr)_{i=1}^{d}
  \\
  &\quad\subseteq
  \prod_{i=1}^{d}
  \bigl\{\, y \in C([0,1]) \;:\; y(0) = 0 ,
  \\
  &\hspace{8em}
  | y(t) - y(s) | \le M\, | t - s | \,\bigr\} .
  \end{aligned}
\]
The right-hand side is the $d$-fold product of the class
$\mathrm{Lip}_{M}$ of Theorem~\ref{fact:lipentropy}, and the
theorem applied componentwise provides an external $s$-cover
of $\mathcal{Y}(x)$ with covering number
\[
  \log N(s) \;\le\; c_{\star}\, d\, \frac{M}{s} ,
  \qquad 0 < s \le M ,
\]
with $c_{\star} := c_{0}$ the constant of
Theorem~\ref{fact:lipentropy}.

Let $g_{1}, \dots, g_{N}$ be such that
\[
  \mathcal{Y}(x)
  \;\subseteq\;
  \bigcup_{\iota=1}^{N} B_{\delta/2}( g_{\iota} ) ,
  \qquad
  \log N \;\le\; 2\, c_{\star}\, d\, \frac{M}{\delta} ,
\]
where
$B_{s}(g) := \{\, y : \max_{i} \sup_{[0,1]} | y_{i} - g_{i} |
\le s \,\}$. For each $\iota$ with
$B_{\delta/2}(g_{\iota}) \cap \mathcal{Y}(x) \ne \emptyset$, pick
a tuple
$y_{v_{\iota}} \in B_{\delta/2}(g_{\iota}) \cap \mathcal{Y}(x)$
and assemble the command set
\[
  \begin{aligned}
  \A^{*} \;:=\;
  \bigl\{\, a_{\iota} :={}&
  \bigl( w_{v_{\iota}, i} \bigr)_{i=1}^{d}
  = \bigl( \dot{y}_{v_{\iota}, i} \bigr)_{i=1}^{d}
  \\
  &:\;
  B_{\delta/2}(g_{\iota}) \cap \mathcal{Y}(x) \ne \emptyset
  \,\bigr\} .
  \end{aligned}
\]

Pick $a_{\iota} \in \A^{*}$ and solve the initial value problem
$z^{(R)} = w_{a_{\iota}}$, $\jet z(0) = \Xi(x)$, by integrating
the chain \eqref{eq:rollout}. We obtain the flat trajectory under
$a_{\iota}$ starting at $x$,
\[
  \begin{aligned}
  z^{x}_{a_{\iota},i}(t)
  \;={}&
  \sum_{k=0}^{r_{i}-1} \Xi_{i,k}(x)\, \frac{t^{k}}{k!}
  \\
  &+
  \frac{1}{(r_{i}-1)!} \int_{0}^{t} (t - s)^{r_{i}-1}\,
  w_{v_{\iota},i}(s)\, \mathrm{d}s
  \\
  ={}& z_{v_{\iota},i}(t) ,
  \end{aligned}
\]
the last equality being Taylor's formula with integral remainder
for $z_{v_{\iota}}$. Following Assumption (ii),
Section~\ref{sec:assumptions},
\[
  u^{x}_{a_{\iota}}(t)
  \;=\;
  \Psi_{u}\bigl( \jet z_{v_{\iota}}(t),\, w_{v_{\iota}}(t) \bigr)
  \;=\;
  v_{\iota}(t) \;\in\; U
  \quad \text{a.e.} ,
\]
and hence each command in $\A^{*}$ is accepted by construction,
with $\roll_{x}(a_{\iota})$ the flat trace of
$\tra^{x}_{v_{\iota}}$.

Finally, for any admissible control $v$ from $x$,
$y_{v} \in B_{\delta/2}(g_{\iota})$ for some $\iota$, and there
exists a picked tuple
$y_{v_{\iota}} \in B_{\delta/2}(g_{\iota}) \cap \mathcal{Y}(x)$.
Using the triangle inequality,
\[
  \max_{1 \le i \le d}\; \sup_{[0,1]}
  \bigl| z_{v,i}^{(r_{i}-1)} - z_{v_{\iota},i}^{(r_{i}-1)} \bigr|
  \;\le\; \delta ,
\]
and $\jet z_{v}(0) = \jet z_{v_{\iota}}(0) = \Xi(x)$. Following
Lemma~\ref{lem:toporder}, we have
\[
  \min_{a \in \A^{*}}
  d\bigl( \roll_{x}(a),\, \tra^{x}_{v};\, [0,1] \bigr)
  \;\le\; \delta .
\]
Absorbing the factor $2$ in
$\log N \le 2\, c_{\star}\, d\, M/\delta$ into $c_{\star}$ completes our proof.

\subsection{Proof of Lemma~\ref{lem:affine} (Affinity of the
Input Map)}
\label{app:affine}

Fix a flat state $\xi$ in the domain of Assumption (ii),
Section~\ref{sec:assumptions}, and $w \in \R^{d}$. Following
Theorem~\ref{thm:flmr}, the closed-loop system under the
endogenous dynamic feedback with the constant new input $w$
produces, from the flat state $\xi$, a trajectory $x(\cdot)$ of
\eqref{eq:sys} whose flat trace $z$ obeys
\[
  \jet z(0) \;=\; \xi ,
  \qquad
  z^{(R)} \;\equiv\; w ,
\]
and along which, following Assumption (ii),
\[
  x(t) \;=\; \Psi_{x}\bigl( \jet z(t) \bigr) ,
  \qquad
  u(t) \;=\; \Psi_{u}\bigl( \jet z(t),\, w \bigr) .
\]
The jet propagates by the integrator chain,
\[
  \frac{\mathrm{d}}{\mathrm{d}t}\, \jet z(t)
  \;=\;
  S_{0}\, \jet z(t) + \Sigma\, w ,
\]
where the constant matrices $S_{0}$ and $\Sigma$ act entrywise,
for $1 \le i \le d$, as
\[
  \begin{aligned}
  \bigl( S_{0}\, \xi \bigr)_{i,k}
  &\;=\;
  \begin{cases}
    \xi_{i, k+1} , & 0 \le k \le r_{i} - 2 , \\
    0 , & k = r_{i} - 1 ,
  \end{cases}
  \\
  \bigl( \Sigma\, w \bigr)_{i,k}
  &\;=\;
  \begin{cases}
    0 , & 0 \le k \le r_{i} - 2 , \\
    w_{i} , & k = r_{i} - 1 .
  \end{cases}
  \end{aligned}
\]
The map $\Psi_{x}$ is continuously differentiable per Assumption
(ii), so the chain rule gives, for every $t$,
\[
  \dot{x}(t)
  \;=\;
  D\Psi_{x}\bigl( \jet z(t) \bigr)
  \bigl( S_{0}\, \jet z(t) + \Sigma\, w \bigr) ,
\]
while the control-affine dynamics of \eqref{eq:sys} give, for
almost every $t$,
\[
  \dot{x}(t)
  \;=\;
  f_{0}\bigl( x(t) \bigr) + g\bigl( x(t) \bigr)\, u(t) .
\]
Since $g$ has full column rank,
$g^{+} := ( g^{\top} g )^{-1} g^{\top}$ left-inverts $g$;
equating the two expressions for $\dot{x}(t)$ and solving for
$u(t)$ yields, for almost every $t$,
\[
  \Psi_{u}\bigl( \jet z(t),\, w \bigr)
  \;=\;
  u(t)
  \;=\;
  \alpha\bigl( \jet z(t) \bigr)
  + B\bigl( \jet z(t) \bigr)\, w ,
\]
where
\begin{align*}
  \alpha( \xi )
  &\;:=\;
  g\bigl( \Psi_{x}(\xi) \bigr)^{+}
  \bigl( D\Psi_{x}(\xi)\, S_{0}\, \xi
  - f_{0}( \Psi_{x}(\xi) ) \bigr) ,
  \\
  B( \xi )
  &\;:=\;
  g\bigl( \Psi_{x}(\xi) \bigr)^{+} D\Psi_{x}(\xi)\, \Sigma .
\end{align*}
Both sides of the identity are continuous in $t$, following the
continuity of $\Psi_{u}$, $\Psi_{x}$, $D\Psi_{x}$, $f_{0}$, and
$g$, so it extends from almost every $t$ to every $t$; at
$t = 0$, where $\jet z(0) = \xi$,
\[
  \Psi_{u}( \xi,\, w )
  \;=\;
  \alpha( \xi ) + B( \xi )\, w .
\]
Finally, $\alpha(\xi) = \Psi_{u}(\xi, 0)$ and
$B(\xi)\, e_{k} = \Psi_{u}(\xi, e_{k}) - \Psi_{u}(\xi, 0)$ for
the coordinate vectors $e_{k}$ of $\R^{d}$, both continuously
differentiable per Assumption (ii), so $\alpha$ and $B$ are
locally Lipschitz.

\subsection{Proof of Theorem~\ref{thm:realization} (Dispersion
Realization)}
\label{app:lattice}

We first prove the following lemma, which shows that the
acceptance test (Definition~\ref{def:accepted}) retains every
command that approximates the flat input of an eroded admissible
control in the specific sense introduced below. The realization
proof thereafter rests on this lemma.

\begin{lemma}[Margin transfer]
\label{lem:margin}
Let $\delta \in (0, M/2]$. Let $v$ be an admissible control on
$[0, 1]$ from a state $x \in \X$ whose trajectory is
$M$-smooth, and let its flat input be
$w_{v}$. Let $j \ge 4 \sqrt{d}\, M/\delta$ be an integer,
$I^{j}_{\ell} = [\tfrac{\ell-1}{j}, \tfrac{\ell}{j}]$,
$1 \le \ell \le j$, the equal subwindows of $[0, 1]$, and
\[
  \bar{w}_{\ell}
  \;:=\;
  j \int_{I^{j}_{\ell}} w_{v}(s)\, \mathrm{d}s
  \;\in\; [-M, M]^{d}
\]
the subwindow averages of $w_{v}$. Let $a$ be a command, constant
on each subwindow with values $a_{\ell}$, satisfying
\[
  | a_{\ell} - \bar{w}_{\ell} | \;\le\; \delta/2 ,
  \qquad
  \sup_{t \in [0,1]}
  \bigl| \jet z^{x}_{a}(t) - \jet z_{v}(t) \bigr|_{\infty}
  \;\le\; \delta .
\]
If $v \in \mathcal{U}^{\varepsilon}$ with
$\varepsilon \ge 2\, L_{\Psi}\, \delta$, then $a \in \A(x)$.
\end{lemma}

\begin{proof}
Let $\mathrm{dist}(u, S) := \inf_{u' \in S} | u - u' |_{\infty}$
denote the distance in the supremum norm on $\R^{m}$. Fix
$1 \le \ell \le j$ and $t$ in the interior of $I^{j}_{\ell}$,
where $u^{x}_{a}(t) = \Psi_{u}( \jet z^{x}_{a}(t),\, a_{\ell} )$.
Note that, for every $s \in I^{j}_{\ell}$, the pairs
\[
  \begin{aligned}
  &\bigl( \jet z_{v}(s),\, w_{v}(s) \bigr) ,\quad
  \bigl( \jet z_{v}(t),\, w_{v}(s) \bigr) , \\
  &\bigl( \jet z_{v}(t),\, \bar{w}_{\ell} \bigr) ,\quad
  \bigl( \jet z^{x}_{a}(t),\, a_{\ell} \bigr)
  \end{aligned}
\]
are bounded entrywise by $2M$, hence within the scope of the
Lipschitz constant $L_{\Psi}$.
This follows from the $M$-smoothness of the trajectory of
$v$, together with the rollout bound
$\sup_{[0,1]} | \jet z^{x}_{a} - \jet z_{v} |_{\infty}
\le \delta \le M/2$ hypothesized in the lemma.
Since $v \in \mathcal{U}^{\varepsilon}$,
\[
  \Psi_{u}\bigl( \jet z_{v}(s),\, w_{v}(s) \bigr)
  \;=\; v(s) \;\in\; U^{\varepsilon}
  \qquad \text{for a.e. } s \in I^{j}_{\ell} ,
\]
and every entry of $\jet z_{v}$ is $M$-Lipschitz, again by
the smoothness of the trajectory, so
$| \jet z_{v}(t) - \jet z_{v}(s) |_{\infty} \le M/j$ and
\[
  \begin{aligned}
  &\mathrm{dist}\bigl(
  \Psi_{u}( \jet z_{v}(t),\, w_{v}(s) ),\,
  U^{\varepsilon} \bigr)
  \\
  &\hspace{3em}\le L_{\Psi}\, \frac{M}{j}
  \qquad \text{for a.e. } s \in I^{j}_{\ell} .
  \end{aligned}
\]
Following Lemma~\ref{lem:affine},
\begin{align*}
  \Psi_{u}\bigl( \jet z_{v}(t),\, \bar{w}_{\ell} \bigr)
  &\;=\;
  \alpha\bigl( \jet z_{v}(t) \bigr)
  \\
  &\quad
  + B\bigl( \jet z_{v}(t) \bigr)\,
  j \int_{I^{j}_{\ell}} w_{v}(s)\, \mathrm{d}s
  \\
  &\;=\;
  j \int_{I^{j}_{\ell}}
  \Bigl(
  \alpha\bigl( \jet z_{v}(t) \bigr)
  \\
  &\qquad
  + B\bigl( \jet z_{v}(t) \bigr)\, w_{v}(s)
  \Bigr)\, \mathrm{d}s
  \\
  &\;=\;
  j \int_{I^{j}_{\ell}}
  \Psi_{u}\bigl( \jet z_{v}(t),\,
  w_{v}(s) \bigr)\, \mathrm{d}s .
\end{align*}
$U^{\varepsilon}$ is convex as the erosion of the convex $U$, so
the function $\mathrm{dist}( \cdot\,, U^{\varepsilon} )$ is convex
and Jensen's inequality provides
\[
  \begin{aligned}
  &\mathrm{dist}\bigl(
  \Psi_{u}( \jet z_{v}(t),\, \bar{w}_{\ell} ),\,
  U^{\varepsilon} \bigr)
  \\
  &\quad\le
  j \int_{I^{j}_{\ell}}
  \mathrm{dist}\bigl(
  \Psi_{u}( \jet z_{v}(t),\, w_{v}(s) ),\,
  U^{\varepsilon} \bigr)\, \mathrm{d}s
  \\
  &\quad\le L_{\Psi}\, \frac{M}{j} .
  \end{aligned}
\]
Finally
$| \jet z^{x}_{a}(t) - \jet z_{v}(t) |_{\infty} \le \delta$ and
$| a_{\ell} - \bar{w}_{\ell} |_{\infty} \le \delta/2$ by
hypothesis, so
\[
  \begin{aligned}
  \mathrm{dist}\bigl( u^{x}_{a}(t),\, U^{\varepsilon} \bigr)
  &\;\le\;
  L_{\Psi}\, \delta + L_{\Psi}\, \frac{M}{j} \\
  &\;\le\;
  \frac{5}{4}\, L_{\Psi}\, \delta
  \;<\;
  2\, L_{\Psi}\, \delta
  \;\le\; \varepsilon ,
  \end{aligned}
\]
using $M/j \le \delta/(4\sqrt{d}) \le \delta/4$. Hence
$u^{x}_{a}(t) \in U^{\varepsilon} + [-\varepsilon,
\varepsilon]^{m} \subseteq U$ for almost every $t \in [0,1]$,
i.e., $a \in \A(x)$.
\end{proof}

We now prove Theorem~\ref{thm:realization}. Fix a state
$x \in \X$ and an $\varepsilon$-eroded admissible control
$v \in \mathcal{U}^{\varepsilon}_{M}(x)$ with flat input
$w_{v}$. Its averages over the subwindows $I^{j}_{\ell}$ obey
\[
  \bar{w}_{\ell} \;:=\; j \int_{I^{j}_{\ell}} w_{v}(s)\,
  \mathrm{d}s
  \;\in\; [-M, M]^{d} ,
  \qquad 1 \le \ell \le j ,
\]
following the convexity of the box $[-M, M]^{d}$. The grid $G$ of
Theorem~\ref{thm:realization} has per-axis spacing $2M/n$, and
every point of $[-M,\, M]$ lies within half a spacing of
some center, so there is $a_{\ell} \in G$ with
\[
  | a_{\ell} - \bar{w}_{\ell} |_{\infty}
  \;\le\;
  \frac{M}{n}
  \;\le\;
  \frac{\delta}{2 \sqrt{d}} ,
  \qquad
  | a_{\ell} - \bar{w}_{\ell} |
  \;\le\;
  \frac{\delta}{2} ,
\]
where the two respectively follow using
$n \ge 2 \sqrt{d}\, M/\delta$ and
$| \cdot | \le \sqrt{d}\, | \cdot |_{\infty}$ on $\R^{d}$. The
command $a := w_{g}$ with
$g = (a_{1}, \dots, a_{j}) \in G^{j}$ belongs to $\A$.

Write
$F(t) := \int_{0}^{t} ( w_{v}(s) - w_{a}(s) )\, \mathrm{d}s$. At
the subwindow boundaries $t_{\ell} = \ell/j$,
\[
  | F(t_{\ell}) |
  \;\le\;
  \sum_{\ell' \le \ell}
  \Bigl| \int_{I^{j}_{\ell'}} ( w_{v} - a_{\ell'} ) \Bigr|
  \;=\;
  \sum_{\ell' \le \ell}
  \frac{ | \bar{w}_{\ell'} - a_{\ell'} | }{j}
  \;\le\;
  \frac{\delta}{2} ,
\]
while for $t \in I^{j}_{\ell}$,
\[
  | F(t) - F(t_{\ell-1}) |
  \;\le\;
  \int_{I^{j}_{\ell}} | w_{v} - a_{\ell} |
  \;\le\;
  \frac{2 \sqrt{d}\, M}{j}
  \;\le\;
  \frac{\delta}{2} ,
\]
so $\sup_{[0,1]} | F | \le \delta$, in the Euclidean norm and
hence entrywise. The rollouts $z^{x}_{a}$ and $z_{v}$ share the
initial jet, $\jet z^{x}_{a}(0) = \jet z_{v}(0) = \Xi(x)$, and
the top-order gap is $F$ itself:
\[
  z^{x\,(r_{i}-1)}_{a,i}(t) - z^{(r_{i}-1)}_{v,i}(t)
  \;=\;
  \int_{0}^{t} ( w_{a,i} - w_{v,i} )
  \;=\;
  - F_{i}(t) .
\]
Lemma~\ref{lem:toporder} thus applies and we get
$\sup_{[0,1]} | \jet z^{x}_{a} - \jet z_{v} |_{\infty} \le \delta$.
The constructed command $a$ thus satisfies the hypotheses of
Lemma~\ref{lem:margin}: for every
$\varepsilon \ge 2\, L_{\Psi}\, \delta$ we obtain $a \in \A(x)$
and
\[
  d\bigl( \roll_{x}(a),\, \tra^{x}_{v};\, [0,1] \bigr)
  \;=\;
  \sup_{[0,1]} | \jet z^{x}_{a} - \jet z_{v} |_{\infty}
  \;\le\; \delta .
\]
Hence $\A$ is $\delta$-dispersive over
$\mathcal{U}^{\varepsilon}_{M}(x)$ at every state $x$.

It remains to estimate the cardinality of the constructed command
set. We have $| G | = n^{d}$ and $| \A | \le | G |^{j}$. Taking logarithms, and using
$j = \lceil 4 \sqrt{d}\, M/\delta \rceil \le 5 \sqrt{d}\,
M/\delta$ together with
$n \le 1 + 2 \sqrt{d}\, M/\delta \le 2 \sqrt{d}\,
( 1 + M/\delta )$,
\[
  \begin{aligned}
  \log | \A |
  &\;\le\;
  j\, d\, \log n \\
  &\;\le\;
  \frac{5 \sqrt{d}\, M}{\delta}\; d\,
  \Bigl( \log( 2 \sqrt{d} )
  + \log\Bigl( 1 + \frac{M}{\delta} \Bigr) \Bigr)
  \\
  &\;\le\;
  c_{d}\, \frac{M}{\delta}\,
  \log\Bigl( 1 + \frac{M}{\delta} \Bigr) ,
  \end{aligned}
\]
the last inequality since
$\log( 1 + M/\delta ) \ge \log 2$ for $\delta \le M$, with
$c_{d} := 5\, d^{3/2} \bigl( 1 + \log( 2 \sqrt{d} ) / \log 2
\bigr)$ depending on the flat dimension $d$ alone.

\subsection{Proof of Lemma~\ref{lem:corrector2} (Margin Corrector)}
\label{app:corrector}

Denote
\[
  \eta \;:=\; \Xi(x') - \jet z_{u}(t) ,
\]
so that the components obey
\[
  \begin{aligned}
  | \eta_{i,k} |
  &\;=\;
  \bigl| \Xi_{i,k}(x') - z_{u,i}^{(k)}(t) \bigr|
  \;\le\; s , \\
  &\hspace{3em}
  0 \le k \le r_{i} - 1 ,\quad 1 \le i \le d ,
  \end{aligned}
\]
by hypothesis. For each component $i$, let
$Q_{i} : [0, 1] \to \R$ be the unique polynomial of
degree at most $2 r_{i} - 1$ satisfying the Hermite interpolation
conditions
\[
  Q_{i}^{(k)}(0) \;=\; \eta_{i,k},
  \qquad
  Q_{i}^{(k)}(1) \;=\; 0,
  \qquad
  0 \le k \le r_{i} - 1 .
\]
The claim below lets us control the regularity of this
polynomial.

\begin{claim}[Two-point Hermite interpolation,
{\cite[Section~2.1.5]{stoer2002introduction}}]
\label{clm:hermite}
Fix $r \in \mathbb{N}$. For every data vector
$( \eta_{0}, \dots, \eta_{r-1} ) \in \R^{r}$ there is a unique
polynomial $Q$ of degree at most $2r - 1$ satisfying the Hermite
interpolation conditions
\[
  Q^{(k)}(0) \;=\; \eta_{k} ,
  \qquad
  Q^{(k)}(1) \;=\; 0 ,
  \qquad
  0 \le k \le r - 1 ,
\]
and a constant $c_{Q}(r)$, the \emph{joint norm} of the
interpolation, depending on the order $r$ alone, with
\[
  \sup_{[0, 1]} \bigl| Q^{(k)} \bigr|
  \;\le\;
  c_{Q}(r) \max_{0 \le k' \le r-1} | \eta_{k'} | ,
  \qquad 0 \le k \le r .
\]
\end{claim}

With $r = r_{i}$ and the data
$( \eta_{i,0}, \dots, \eta_{i,r_{i}-1} )$,
\[
  \sup_{[0, 1]} \bigl| Q_{i}^{(k)} \bigr|
  \;\le\;
  c_{Q}(r_{i})\, s
  \;\le\;
  C_{B}\, s ,
  \qquad 0 \le k \le r_{i} ,
\]
where $C_{B} = \max_{1 \le i \le d} c_{Q}(r_{i})$
as fixed in Section~\ref{sec:correction}.

Define the connecting trace
\[
  \begin{aligned}
  y(t') &\;:=\; z_{u}(t') + Q(t' - t),
  \qquad t' \in [t,\, t + 1] , \\
  Q &:= (Q_{1}, \dots, Q_{d}) .
  \end{aligned}
\]
Its jet at $t$ is $\Xi(x')$, at $t + 1$ it is the jet
of $z_{u}$, and its gap to the target obeys
\[
  \begin{aligned}
  \bigl| \jet y(t') - \jet z_{u}(t') \bigr|_{\infty}
  &\;\le\; C_{B}\, s , \\
  \bigl| y^{(R)}(t') - z_{u}^{(R)}(t') \bigr|_{\infty}
  &\;\le\; C_{B}\, s ,
  \qquad t' \in [t,\, t + 1] .
  \end{aligned}
\]

Following
$s \le \varepsilon / ( 2 L_{\Psi} C_{B} )$ and
$\varepsilon \le L_{\Psi} M$,
\[
  C_{B}\, s
  \;\le\;
  \frac{\varepsilon}{2 L_{\Psi}}
  \;\le\;
  \frac{M}{2} ,
\]
so, for $1 \le k \le r_{i}$, $1 \le i \le d$, and
$t' \in [t,\, t + 1]$,
\[
  \begin{aligned}
  \bigl| y_{i}^{(k)}(t') \bigr|
  &\;\le\;
  \bigl| z_{u,i}^{(k)}(t') \bigr|
  + \bigl| Q_{i}^{(k)}(t' - t) \bigr| \\
  &\;\le\;
  M + C_{B}\, s
  \;\le\; M' ,
  \end{aligned}
\]
so the trace $y$ is $M'$-smooth on $[t,\, t+1]$, placing the
jets and flat inputs of both $y$ and $z_{u}$ within the scope of
the Lipschitz constant $L_{\Psi}$.

Define the bridge control
\[
  \nu(t')
  \;:=\;
  \Psi_{u}\bigl( \jet y(t'),\, y^{(R)}(t') \bigr) ,
  \qquad t' \in [t,\, t + 1] .
\]
Along the target,
$u(t') = \Psi_{u}( \jet z_{u}(t'),\, z_{u}^{(R)}(t') ) \in
U^{\varepsilon}$ almost everywhere, so
\[
  \begin{aligned}
  | \nu(t') - u(t') |_{\infty}
  &\;\le\;
  L_{\Psi} \max\bigl(
  | \jet y(t') - \jet z_{u}(t') |_{\infty}, \\
  &\hspace{7em}
  | y^{(R)}(t') - z_{u}^{(R)}(t') |_{\infty} \bigr)
  \\
  &\;\le\;
  L_{\Psi}\, C_{B}\, s
  \;\le\;
  \frac{\varepsilon}{2} ,
  \end{aligned}
\]
whence
$\nu(t') \in U^{\varepsilon/2} \subseteq U$
almost everywhere, since every $q$ with
$| q |_{\infty} \le \varepsilon/2$ puts
$\nu(t') + q$ within $\varepsilon$ of
$u(t') \in U^{\varepsilon}$, hence inside $U$. The bridge control
is therefore admissible, and following the converse in Assumption
(ii), Section~\ref{sec:assumptions}, the trace $y$ is realized by
a trajectory $\tra_{\nu}$ under $\nu$ from the state
$\Psi_{x}\bigl( \Xi(x') \bigr) = x'$.

On $[t,\, t + 1]$ the flat trace of $\tra_{\nu}$ is
$y = z_{u} + Q(\cdot - t)$, so that
\[
  \begin{aligned}
  &\bigl| \jet z_{\tra_{\nu}}(t') - \jet z_{u}(t') \bigr|_{\infty}
  \\
  &\quad=\;
  \max_{1 \le i \le d}\; \max_{0 \le k \le r_{i}-1}\,
  \bigl| Q_{i}^{(k)}(t' - t) \bigr|
  \;\le\;
  C_{B}\, s ,
  \end{aligned}
\]
and at the far end the interpolation conditions
$Q^{(k)}(1) = 0$ give
$\tra_{\nu}(t + 1) =
\Psi_{x}\bigl( \jet z_{u}(t + 1) \bigr)
= \tra_{u}(t + 1)$, completing the proof.

\subsection{Proof of Theorem~\ref{thm:coverage} (Eroded Coverage)}
\label{app:coverage}

Write $\delta := \delta_{\A}$. We choose the commands inductively and
write $\tra$ for the trajectory of $(a_{1}, \dots, a_{K})$ from
$x_{\mathrm{init}}$; the choice of $a_{k+1}$ will depend only on $\tra$
over $[0, k]$, so this is well defined. We maintain, for every
$1 \le k \le K$, the two invariants
\[
  \begin{aligned}
  \text{(i)}\quad&
  d\bigl(\tra, \tra_{u};\, [0, k]\bigr)
  \;\le\; C_{B}\, \delta + \delta , \\
  \text{(ii)}\quad&
  \bigl| \jet z_{\tra}(k) - \jet z_{u}(k) \bigr|_{\infty}
  \;\le\; \delta .
  \end{aligned}
\]

\emph{Seed.} The restriction of $u$ to $[0, 1]$ lies in
$\mathcal{U}^{\varepsilon}_{M}(x_{\mathrm{init}}) \subseteq
\mathcal{U}^{\varepsilon/2}_{M'}(x_{\mathrm{init}})$,
and $\tra$ and $\tra_{u}$ share the state $x_{\mathrm{init}}$ at
time $0$. Since $\A$ is $\delta$-dispersive over
$\mathcal{U}^{\varepsilon/2}_{M'}(x_{\mathrm{init}})$ at
$x_{\mathrm{init}}$, there is an accepted command
$a_{1} \in \A(x_{\mathrm{init}})$ with
$d\bigl( \roll_{x_{\mathrm{init}}}(a_{1}),\, \tra_{u};\,
[0, 1] \bigr) \le \delta$. Both invariants hold at $k = 1$.

\emph{Step.} Suppose $a_{1}, \dots, a_{k}$ are chosen and the
invariants hold at $k$. By (ii), the node $\tra(k)$ is a state whose
jet gap at time $k$ is at most
$\delta \le c_{B}\, \varepsilon$, so
Lemma~\ref{lem:corrector2} (with $t = k$, $s = \delta$,
$x' = \tra(k)$) furnishes a corrector control
$\nu \in \mathcal{U}^{\varepsilon/2}_{M'}\bigl(\tra(k)\bigr)$
on $[k,\, k+1]$ whose trajectory $\tra_{\nu}$ from $\tra(k)$
satisfies
\[
  \begin{aligned}
  d\bigl( \tra_{\nu},\, \tra_{u};\, [k,\, k + 1] \bigr)
  &\;\le\; C_{B}\, \delta, \\
  \tra_{\nu}(k + 1) &\;=\; \tra_{u}(k + 1) .
  \end{aligned}
\]
By dispersion over $\mathcal{U}^{\varepsilon/2}_{M'}$ at
$\tra(k)$,
applied to $\nu$ upon identifying $[k,\, k+1]$ with $[0, 1]$,
there is an accepted command $a_{k+1} \in \A(\tra(k))$
with $d\bigl( \roll_{\tra(k)}(a_{k+1}),\, \tra_{\nu};\,
[k,\, k+1] \bigr) \le \delta$. The trajectory extended by $a_{k+1}$
therefore satisfies
\begin{equation}
\label{eq:stepwindow}
  \begin{aligned}
  d\bigl(\tra, \tra_{u};\, [k,\, k+1]\bigr)
  &\;\le\;
  d\bigl(\tra, \tra_{\nu}\bigr)
  + d\bigl(\tra_{\nu}, \tra_{u}\bigr) \\
  &\;\le\;
  \delta + C_{B}\, \delta,
  \end{aligned}
\end{equation}
and, at the instant $k + 1$, where $\tra_{\nu}(k+1) = \tra_{u}(k+1)$,
\begin{equation}
\label{eq:stepjet}
  \bigl| \jet z_{\tra}(k+1) - \jet z_{u}(k+1) \bigr|_{\infty}
  \;\le\; \delta .
\end{equation}
The bound \eqref{eq:stepjet} establishes invariant (ii) at
$k + 1$, and \eqref{eq:stepwindow}, combined with invariant
(i) at $k$, extends invariant (i) to $[0,\, k+1]$. Continuing the induction
of invariant (i) through $k = K$ completes the proof.

\subsection{Proof of Theorem~\ref{thm:optimality} (Eroded Optimality)}
\label{app:optimality}

Write $E := C_{B}\, \delta_{\A} + \delta_{\A}$, so that the
assumed margin reads $\gamma \ge L_{\preds}\, E$.

Fix $n \ge 1$. Since
$\mathcal{U}_{M,\gamma,\varepsilon} \ne \emptyset$, we may pick
$u_{n} \in \mathcal{U}_{M,\gamma,\varepsilon}$ with
\[
  J\bigl( z_{u_{n}} \bigr)
  \;\le\;
  J^{*}_{M,\gamma,\varepsilon} + \tfrac{1}{n} .
\]
The control $u_{n}$ lies in
$\mathcal{U}^{\varepsilon}_{M}(x_{\mathrm{init}})$ on $[0, T]$.
We apply Theorem~\ref{thm:coverage} to
$u_{n}$ with horizon $K = T$ and obtain an accepted command
sequence $(a^{n}_{1}, \dots, a^{n}_{T}) \in \A^{T}$ whose
trajectory $\tra_{n}$ from $x_{\mathrm{init}}$ obeys
\[
  d\bigl( \tra_{n},\, \tra_{u_{n}};\, [0, T] \bigr)
  \;\le\; E .
\]
Following the Lipschitz supposition of
Section~\ref{sec:optimality}, for every $g \in \preds$,
\begin{equation}
\label{eq:feasible}
  g\bigl(z_{\tra_{n}}\bigr)
  \;\ge\;
  g\bigl(z_{u_{n}}\bigr) - L_{\preds}\, E
  \;\ge\;
  \gamma - L_{\preds}\, E
  \;\ge\; 0 ,
\end{equation}
and
\begin{equation}
\label{eq:costbound}
  J\bigl(z_{\tra_{n}}\bigr)
  \;\le\;
  J\bigl(z_{u_{n}}\bigr) + L_{J}\, E
  \;\le\;
  J^{*}_{M,\gamma,\varepsilon} + \tfrac{1}{n} + L_{J}\, E .
\end{equation}
The set $\A^{T}$ is finite, so some sequence
$(a_{1}, \dots, a_{T})$ recurs for infinitely many $n$. Its
trajectory $\tra$ is feasible following \eqref{eq:feasible}, and
passing $n \to \infty$ along the recurrence in
\eqref{eq:costbound} gives
$J( z_{\tra} ) \le J^{*}_{M,\gamma,\varepsilon} + L_{J}\, E$.

\subsection{Proof of Proposition~\ref{prop:covsize} (Coverage
Complexity)}
\label{app:covsize}

\emph{Part (i).} Let
$\{\tra_{v_{1}}, \dots, \tra_{v_{P}}\} \subset \mathcal{Z}_{K}$,
$P = P_{2E}(K)$, be pairwise more than $2E$ apart, and let
$\mathcal{T}$ be an $\A$-tree as hypothesized. For each $j$ pick a
depth-$K$ trajectory $\tau_{j}$ of $\mathcal{T}$ with
$d(\tau_{j}, \tra_{v_{j}};\, [0, K]) \le E$. If
$\tau_{j} = \tau_{j'}$ for $j \ne j'$, then
\[
  d\bigl(\tra_{v_{j}}, \tra_{v_{j'}};\, [0, K]\bigr)
  \;\le\;
  d\bigl(\tra_{v_{j}}, \tau_{j}\bigr)
  + d\bigl(\tau_{j'}, \tra_{v_{j'}}\bigr)
  \;\le\; 2E,
\]
a contradiction. The $\tau_{j}$ are therefore $P$ distinct
trajectories of $\mathcal{T}$.

\emph{Part (ii).} Fix a flat component $i$ with
$r_{i} = \bar{r}$, any component
serving with a constant of its own order, and the polynomial
bump on the unit interval,
\[
  \begin{aligned}
  \phi(t)
  &\;:=\;
  \frac{1}{B}\, t^{\bar{r}} (1 - t)^{\bar{r}} , \\
  B
  &\;:=\;
  \max_{0 \le k \le \bar{r}}\;
  \sup_{t \in [0, 1]}\,
  \Bigl| \frac{\mathrm{d}^{k}}{\mathrm{d} t^{k}}\,
  t^{\bar{r}} (1 - t)^{\bar{r}} \Bigr| ,
  \end{aligned}
\]
so that $| \phi^{(k)} | \le 1$ on $[0, 1]$ for
$0 \le k \le \bar{r}$, while
$\phi^{(k)}(0) = \phi^{(k)}(1) = 0$ for
$0 \le k \le \bar{r} - 1$. The derivative
$\phi^{(\bar{r} - 1)}$ is a polynomial of degree $\bar{r} + 1$,
not identically zero, so
\[
  \kappa
  \;:=\;
  \sup_{t \in [0, 1]}
  \bigl| \phi^{(\bar{r} - 1)}(t) \bigr|
  \;>\; 0 ,
\]
attained at some $t^{*} \in [0, 1]$; the constant depends on
$\bar{r}$ alone.

Let $A := \eta / L_{\Psi} \le M/2$ and
\[
  n \;:=\; \Bigl\lfloor \frac{\kappa\, A}{s} \Bigr\rfloor
  \;\ge\; 1 ,
\]
following the hypothesis $s \le \kappa\, \eta / L_{\Psi}$.
Partition $[0, K]$ into the $K n$ intervals
$I_{b} := [(b-1)/n,\, b/n]$, $1 \le b \le K n$, and for a sign
pattern $\tau \in \{-1, +1\}^{K n}$ define the scalar
perturbation
\[
  \begin{aligned}
  p_{\tau}(t)
  &\;:=\;
  \frac{A}{n^{\bar{r}}}\, \tau_{b}\,
  \phi\bigl( n t - (b - 1) \bigr) , \\
  &\hspace{5em}t \in I_{b} , \quad 1 \le b \le K n .
  \end{aligned}
\]
Since $\phi$ vanishes at the interval ends together with its
derivatives through order $\bar{r} - 1$, the pieces glue into
$p_{\tau} \in W^{\bar{r}, \infty}([0, K])$ with
$p_{\tau}^{(k)}(0) = 0$ for $0 \le k \le \bar{r} - 1$ and
\[
  \bigl| p_{\tau}^{(k)}(t) \bigr|
  \;\le\;
  A\, n^{k - \bar{r}}
  \;\le\; A ,
  \qquad 0 \le k \le \bar{r} .
\]

By hypothesis there is $u_{0} \in U$ with
$u_{0} + [-\eta, \eta]^{m} \subseteq U^{\varepsilon}$; let
$u^{*} \equiv u_{0}$ be the constant admissible control on
$[0, K]$ and let $z^{*}$ be its flat trajectory from
$x_{\mathrm{init}}$, per the regularity of Assumption (i).
Perturb its component $i$ alone,
\[
  z_{\tau, i} \;:=\; z^{*}_{i} + p_{\tau} ,
  \qquad
  z_{\tau, j} \;:=\; z^{*}_{j} , \quad j \ne i .
\]
The trajectory of the constant control is $(M/2)$-smooth by
hypothesis, and the perturbation is bounded by
$A \le M/2$ through order $\bar{r}$, so the perturbed curves are
$M$-smooth, their jets and flat inputs bounded entrywise well
within $2M$, and the Lipschitz
bound on $\Psi_{u}$ gives, almost everywhere on $[0, K]$,
\[
  \bigl| \Psi_{u}\bigl( \jet z_{\tau},\, z_{\tau}^{(R)} \bigr)
  - \Psi_{u}\bigl( \jet z^{*},\, z^{*\,(R)} \bigr)
  \bigr|_{\infty}
  \;\le\;
  L_{\Psi}\, A
  \;=\; \eta .
\]
The induced control $u_{\tau}$ therefore deviates from $u_{0}$
by at most $\eta$ and takes values in
$u_{0} + [-\eta, \eta]^{m} \subseteq U^{\varepsilon}$. By the flatness converse of Assumption (ii), and
since $p_{\tau}$ vanishes at $t = 0$ together with its jet,
$z_{\tau}$ is realized from $x_{\mathrm{init}}$ by the
admissible $\varepsilon$-eroded control $u_{\tau}$, and its
trajectory belongs to $\mathcal{Z}_{K}$.

Finally, let $\tau \ne \tau'$ differ at an interval $b$ and let
$t_{b} := (b - 1 + t^{*})/n \in I_{b}$. The perturbations are
supported on their own intervals, whereby
\[
  \begin{aligned}
  &\bigl| z_{\tau, i}^{(\bar{r} - 1)}(t_{b})
  - z_{\tau', i}^{(\bar{r} - 1)}(t_{b}) \bigr|
  \\
  &\quad=\;
  \bigl| \tau_{b} - \tau'_{b} \bigr|\,
  \frac{A}{n}\,
  \bigl| \phi^{(\bar{r} - 1)}(t^{*}) \bigr|
  \\
  &\quad=\;
  \frac{2\, \kappa\, A}{n}
  \;\ge\; 2 s
  \;>\; s .
  \end{aligned}
\]
This implies there exist $2^{K n}$ trajectories in
$\mathcal{Z}_{K}$ pairwise more than $s$ apart, hence
\[
  \log P_{s}(K)
  \;\ge\;
  K\, n \log 2
  \;\ge\;
  \frac{\log 2}{2}\, \kappa\, K\,
  \frac{\eta}{L_{\Psi}\, s} ,
\]
the last step since $\lfloor x \rfloor \ge x / 2$ for $x \ge 1$.

\subsection{Proof of Theorem~\ref{prop:genpruning} (Sparse
Tree Optimality)}
\label{app:genpruning}

Denote $\rho := \delta_{\A} + s$, so that the hypotheses of
Theorem~\ref{prop:genpruning} read
\[
  \rho \;\le\; c_{B}\, \varepsilon ,
  \qquad
  \gamma \;\ge\; L_{\preds}\, (C_{B} + 1)\, \rho .
\]
Write $\Sigma( x )$ for the summary of the accepted prefix
of a node $x$ along the recursion \eqref{eq:sumcost}.

\emph{Tree size.} Every node of the pruned tree is reached by an
accepted command sequence each window of which passed the
smoothness screen of Procedure~\ref{def:genpruning}, so the flat
trace of its trajectory is $2M$-smooth throughout.
At depth $k \le K$ the jet
coordinates of a node therefore obey
\[
  \begin{aligned}
  \bigl| z_{i}^{(m)}(k) \bigr|
  &\;\le\; 2M ,
  \quad 1 \le m \le r_{i} - 1 , \\
  \bigl| z_{i}(k) - z_{i}(0) \bigr|
  &\;\le\; \int_{0}^{k} \bigl| z_{i}'(t) \bigr|\, \mathrm{d}t
  \;\le\; 2 M K ,
  \end{aligned}
\]
so every derivative coordinate ranges over an interval of length
$4M$, and every position coordinate over an interval of length
$4 M K$. An interval of length $\ell$ meets at most
$\lfloor \ell / s \rfloor + 2 \le \ell / s + 2$ cells of the
grid $s \Z$, and the hypotheses give
$s \le \rho \le c_{B}\, \varepsilon \le M / 2$,
whence
\[
  \frac{4M}{s} + 2 \;\le\; \frac{8M}{s} ,
  \qquad
  \frac{4 M K}{s} + 2 \;\le\; \frac{8 K M}{s} .
\]
The retention of Procedure~\ref{def:genpruning} keeps at
most one node per occupied cell after dominance pruning, whereby
the size of the pruned tree is bounded by the enclosing volume,
\[
  \bigl| V_{k} \bigr|
  \;\le\;
  \Bigl( \frac{8 K M}{s} \Bigr)^{d}
  \Bigl( \frac{8 M}{s} \Bigr)^{|R| - d} .
\]
The pruned tree thus holds at most
$K\, ( 8 K M / s )^{d}\, ( 8 M / s )^{|R| - d} + 1$ nodes,
including the root.

\emph{Eroded Coverage and Optimality.} Fix $n \ge 1$. Since
$\mathcal{U}_{M, \gamma, \varepsilon} \ne \emptyset$, we may pick
$u \in \mathcal{U}_{M, \gamma, \varepsilon}$ with
\[
  J\bigl( z_{u} \bigr)
  \;\le\;
  J^{*}_{M, \gamma, \varepsilon} + \tfrac{1}{n} ,
\]
an $\varepsilon$-eroded control on $[0, K]$ with $M$-smooth
trajectory. Let $\Sigma^{*}_{k}$ denote the
summary of $z_{u}$ along the recursion \eqref{eq:sumcost}, so
that $\Sigma^{*}_{0} = \Sigma_{0}$ and
$J( z_{u} ) = \varphi( \Sigma^{*}_{K} )$. Following the
running-minimum decomposition
$g( z_{u} )
= \min_{0 \le k < K} g_{k}\bigl( z_{u}|_{[k,\, k+1]} \bigr)$
of \eqref{eq:predsep} and $g( z_{u} ) \ge \gamma$, every window
of $u$ carries the full predicate margin,
\begin{equation}
\label{eq:guwindowmargin}
  g_{k}\bigl( z_{u}\big|_{[k,\, k+1]} \bigr) \;\ge\; \gamma
  \qquad \forall\, g \in \preds ,
  \quad 0 \le k < K .
\end{equation}
We claim, by induction on $0 \le k \le K$, that the retained set
$V_{k}$ holds a node $x_{k}$ satisfying
\begin{equation}
\label{eq:gpruneinv}
  \begin{aligned}
  \bigl| \Xi( x_{k} ) - \jet z_{u}(k) \bigr|_{\infty}
  &\;\le\; \rho , \\
  \Sigma( x_{k} )
  &\;\le\;
  \Sigma^{*}_{k} + k\, L\, ( C_{B} + 1 )\, \rho .
  \end{aligned}
\end{equation}
At $k = 0$ the node $x_{0} = x_{\mathrm{init}}$ serves, the jets
and the summaries equal.

Suppose \eqref{eq:gpruneinv} holds at $k < K$. The tail of $u$ is
$\varepsilon$-eroded and $\rho \le c_{B}\, \varepsilon$, so
Lemma~\ref{lem:corrector2} (with $t = k$, $s = \rho$,
$x' = x_{k}$) furnishes a corrector control
$\nu \in \mathcal{U}^{\varepsilon/2}_{M'}(x_{k})$ on
$[k,\, k+1]$ whose trajectory $\tra_{\nu}$ from $x_{k}$
satisfies
\[
  \begin{aligned}
  d\bigl( \tra_{\nu},\, \tra_{u};\, [k,\, k+1] \bigr)
  &\;\le\; C_{B}\, \rho , \\
  \tra_{\nu}(k+1) &\;=\; \tra_{u}(k+1) .
  \end{aligned}
\]
By dispersion over $\mathcal{U}^{\varepsilon/2}_{M'}$ at
$x_{k}$,
applied to $\nu$ upon identifying $[k,\, k+1]$ with $[0, 1]$,
there is an accepted command $a \in \A( x_{k} )$ with
$d\bigl( \roll_{x_{k}}(a),\, \tra_{\nu};\,
[k,\, k+1] \bigr) \le \delta_{\A}$. Write $y := \Phi_{a}( x_{k} )$
for the child, $\tra_{y}$ for its trajectory, extending
$\tra_{x_{k}}$ by the realized window of $a$ from $x_{k}$. Then
\begin{equation}
\label{eq:gprunewindow}
  \begin{aligned}
  d\bigl( \tra_{y},\, \tra_{u};\, [k,\, k+1] \bigr)
  &\;\le\;
  \delta_{\A} + C_{B}\, \rho
  \;\le\;
  ( C_{B} + 1 )\, \rho , \\
  \bigl| \Xi( y ) - \jet z_{u}( k + 1 ) \bigr|_{\infty}
  &\;\le\; \delta_{\A} ,
  \end{aligned}
\end{equation}
and the jet bound holds at the instant $k + 1$, where
$\tra_{\nu}( k + 1 ) = \tra_{u}( k + 1 )$. The following
consequences are then immediate. The summary over the window
$[k,\, k+1]$ tracks the reference's summary,

\begin{align*}
  \Sigma( y )
  \;&=\;
  \varrho_{k}\Bigl( \Sigma( x_{k} ),\;
  z_{\tra_{y}}\big|_{[k,\, k+1]} \Bigr)
  \\
  \;&\le\;
  \varrho_{k}\Bigl( \Sigma^{*}_{k}
  + k\, L\, ( C_{B} + 1 )\, \rho ,\;
  z_{\tra_{y}}\big|_{[k,\, k+1]} \Bigr)
  \\
  &\hspace{2em}\text{by \eqref{eq:gpruneinv}, monotonicity}
  \\
  \;&\le\;
  \varrho_{k}\Bigl( \Sigma^{*}_{k},\;
  z_{\tra_{y}}\big|_{[k,\, k+1]} \Bigr)
  + k\, L\, ( C_{B} + 1 )\, \rho
  \\
  &\hspace{2em}\text{by non-expansiveness}
  \\
  \;&\le\;
  \varrho_{k}\Bigl( \Sigma^{*}_{k},\;
  z_{u}\big|_{[k,\, k+1]} \Bigr)
  + ( k + 1 )\, L\, ( C_{B} + 1 )\, \rho
  \\
  &\hspace{2em}
  \text{by the window Lipschitz bound, \eqref{eq:gprunewindow}}
  \\
  \;&=\;
  \Sigma^{*}_{k+1}
  + ( k + 1 )\, L\, ( C_{B} + 1 )\, \rho
  \\
  &\hspace{2em}\text{by \eqref{eq:sumcost}} .
\end{align*}%
Similarly, the windowed Lipschitz structure of the predicates
of Definition~\ref{as:sumstruct}, alongside
\eqref{eq:guwindowmargin} and
$\gamma \ge L_{\preds}\, (C_{B} + 1)\, \rho$, gives
\[
  g_{k}\Bigl( z_{\tra_{y}}\big|_{[k,\, k+1]} \Bigr)
  \;\ge\;
  \gamma - L_{\preds}\, ( C_{B} + 1 )\, \rho
  \;\ge\; 0
  \qquad \forall\, g \in \preds .
\]
The window of $\tra_{y}$ also passes the smoothness screen,
its jets deviating from those of the $M$-smooth reference by at
most $( C_{B} + 1 )\, \rho \le M$ following
$\rho \le c_{B}\, \varepsilon$, $\varepsilon \le L_{\Psi} M$,
and $C_{B} \ge 1$, while its flat input is the command
input $w_{a}$, bounded by $2M$ by hypothesis.
Hence $y \in \widehat{V}_{k+1}$
(Procedure~\ref{def:genpruning}). At least one node
$x_{k+1} \in V_{k+1}$ in the occupied cell
$q_{s}( \Xi( y ) )$ then survives the pruning with
$\Sigma( x_{k+1} ) \le \Sigma( y )$, and since the jet
$\Xi( x_{k+1} )$
deviates from $\Xi( y )$ by at most $s$ in every coordinate,
\[
  \bigl| \Xi( x_{k+1} ) - \jet z_{u}( k + 1 ) \bigr|_{\infty}
  \;\le\;
  s + \delta_{\A}
  \;=\; \rho ,
\]
and both bounds required by the induction \eqref{eq:gpruneinv}
hold at $k + 1$. Continuing the induction through $K$, it
follows by the monotonicity and the Lipschitz bound of the
read-out $\varphi$ that the retained
set $V_{K}$ holds a node $x_{K} = x_{K}^{n}$ with
\begin{align*}
  J\bigl( z_{\tra_{x_{K}}} \bigr)
  \;&=\;
  \varphi\bigl( \Sigma( x_{K} ) \bigr)
  \\
  \;&\le\;
  \varphi\bigl( \Sigma^{*}_{K}
  + K\, L\, ( C_{B} + 1 )\, \rho \bigr)
  \\
  \;&\le\;
  J\bigl( z_{u} \bigr)
  + L_{\varphi}\, K\, L\, ( C_{B} + 1 )\, \rho
  \\
  \;&\le\;
  J^{*}_{M, \gamma, \varepsilon} + \tfrac{1}{n}
  + L_{\varphi}\, K\, L\, ( C_{B} + 1 )\, \rho .
\end{align*}
Every window of every candidate of
Procedure~\ref{def:genpruning} passed the screen
$g_{j} \ge 0$, so
$g\bigl( z_{\tra_{x_{K}}} \bigr) \ge 0$ for every
$g \in \preds$ by the running-minimum decomposition
\eqref{eq:predsep}, and the trajectory of $x_{K}^{n}$ is feasible
for \eqref{eq:ocp}. Since the pruned tree is independent of $n$
and $V_{K}$ is finite, some node recurs for infinitely many $n$,
and passing $n \to \infty$ along the recurrence establishes
\[
  J\bigl( z_{\tra} \bigr)
  \;\le\;
  J^{*}_{M, \gamma, \varepsilon}
  + L_{\varphi}\, K\, L\, ( C_{B} + 1 )\,
  \bigl( \delta_{\A} + s \bigr) ,
\]
the claimed gap with $L_{J} = L_{\varphi}\, K\, L$.

\subsection{Proof of Proposition~\ref{prop:safety} (Safety
Assurance)}
\label{app:safety}

Consider a window committed at a time $t_{0}$ and let $x(t)$,
$t \in [t_{0}, t_{0} + \Delta]$, be the trajectory executed
through it, exact by supposition (iii). The commit certifies
the predicates of OCP~\eqref{eq:online_ocp} along the committed
trajectory against the centers $o_{j}(t_{0})$, known exactly at
the commit by supposition (ii), so that
\[
  \mathrm{dist}\bigl( x(t),\, o_{j}(t_{0}) \bigr)
  \;\ge\; r + v_{\max}^{j}\, ( t - t_{0} ) .
\]
Meanwhile, supposition (i) bounds the displacement of each
obstacle,
\[
  \bigl\| o_{j}(t) - o_{j}(t_{0}) \bigr\|
  \;\le\; v_{\max}^{j}\, ( t - t_{0} ) ,
\]
and the triangle inequality yields
$\mathrm{dist}( x(t), o_{j}(t) ) \ge r$ throughout the window.
Every window of the executed flight is committed, and the
claim follows by induction over the commits.

\section{Command Set Dispersion for the Platform Catalog}
\label{app:platform_catalog}

We derive the dispersion bounds for each of the platforms
asserted in Section~\ref{sec:examples}, each in the notation
of its construction. Each construction samples its command
channels from explicit bounded grids, so the command inputs
satisfy the uniform bound hypothesized by
Theorem~\ref{prop:genpruning} whenever the grid boxes are
sized within $[-2M, 2M]^{d}$.

\paragraph{1.\ Unicycle, first order.}
Let $\tra^{*}$ be an $M$-smooth admissible reference from the
same state. Following
$|\ddot x^{*}|, |\ddot y^{*}| \le M$, the tangential rate
obeys $|\dot v^{*}| \le \sqrt{2}\, M$. Choose $v_{\ell}$
closest to $v^{*}(t_{\ell})$ on $I_{\ell}$ and
$\omega_{\ell}$ closest to $\omega^{*}$ averaged on
$I_{\ell}$. Then
\[
  \begin{aligned}
  \sup_{t \le \Delta} | v_{a}(t) - v^{*}(t) |
  &\;\le\; \tfrac{1}{2(n-1)} + \tfrac{\sqrt{2} M \Delta}{j}, \\
  \sup_{t \le \Delta} | \theta(t) - \theta^{*}(t) |
  &\;\le\; \tfrac{1}{2}
  \Bigl( \tfrac{1}{n-1} + \tfrac{1}{2j} \Bigr) \Delta ,
  \end{aligned}
\]
Since
$\dot z = v\, ( \cos\theta, \sin\theta )$ and
$| v^{*} | \le \tfrac{1}{2}$,
\[
  \begin{aligned}
  \sup_{t \le \Delta} | \dot z - \dot z^{*} |
  &\;\le\; \tfrac{1}{2(n-1)} + \tfrac{\sqrt{2} M \Delta}{j}
  + \tfrac{1}{4}
  \Bigl( \tfrac{1}{n-1} + \tfrac{1}{2j} \Bigr) \Delta , \\
  \sup_{t \le \Delta} | z - z^{*} |
  &\;\le\; \Delta \, \sup_{t \le \Delta}
  | \dot z - \dot z^{*} |,
  \end{aligned}
\]
so in the metric of Definition~\ref{def:dist}, for
$\Delta \le 1$,
$$
  d\bigl( \tra_{a},\, \tra^{*};\, [0, \Delta] \bigr)
  \;\le\;
  \frac{1 + \Delta/2}{2\,(n-1)}
  + \Bigl( \sqrt{2}\, M + \tfrac{1}{8} \Bigr) \frac{\Delta}{j} .
$$

\paragraph{2.\ Unicycle, second order.}
Let $\tra^{*}$ be an $M$-smooth admissible reference from the
same state. The jet of the reduction is
\[
  \begin{aligned}
  \dot z &= v\, ( \cos\theta,\, \sin\theta ), \\
  \ddot z &= a\, ( \cos\theta,\, \sin\theta )
  + v \omega\, ( -\sin\theta,\, \cos\theta ),
  \end{aligned}
\]
so that
$a = \bigl\langle ( \cos\theta, \sin\theta ),\,
\ddot z \bigr\rangle$. Differentiating along the trajectory,
\[
  \begin{aligned}
  \dot a
  &= \omega \bigl\langle ( -\sin\theta, \cos\theta ),\,
  \ddot z \bigr\rangle
  + \bigl\langle ( \cos\theta, \sin\theta ),\,
  \dddot z \bigr\rangle \\
  &= v \omega^{2}
  + \bigl\langle ( \cos\theta, \sin\theta ),\,
  \dddot z \bigr\rangle .
  \end{aligned}
\]
With $|\dddot x^{*}|, |\dddot y^{*}| \le M$ and
$|v^{*}|, |\omega^{*}| \le \tfrac{1}{2}$,
\[
  | \dot a^{*} |
  \;\le\;
  \tfrac{1}{2} \cdot \tfrac{1}{4} + \sqrt{2}\, M
  \;=\; \sqrt{2}\, M + \tfrac{1}{8}
  \;=:\; M_{a} .
\]
Using $a_{\ell}$ and $\alpha_{\ell}$ closest to $a^{*}$ and
$\alpha^{*}$ averaged on $I_{\ell}$, both speeds are
integrals of gridded channels, and
\[
  \sup_{t \le \Delta} | v - v^{*} | ,
  \quad
  \sup_{t \le \Delta} | \omega - \omega^{*} |
  \;\le\;
  E
  \;:=\;
  \tfrac{1}{4} \Bigl( \tfrac{1}{n-1} + \tfrac{1}{2j} \Bigr)
  \Delta ,
\]
and integrating once more,
$\sup_{t \le \Delta} | \theta - \theta^{*} | \le E \Delta$.
Since $|v^{*}| \le \tfrac{1}{2}$ and
$| ( \cos\theta, \sin\theta )
- ( \cos\theta^{*}, \sin\theta^{*} ) |
\le | \theta - \theta^{*} |$,
\[
  \begin{aligned}
  | \dot z - \dot z^{*} |
  &\;\le\; | v - v^{*} | + | v^{*} |\, | \theta - \theta^{*} |
  \;\le\; E \Bigl( 1 + \tfrac{\Delta}{2} \Bigr), \\
  | z - z^{*} |
  &\;\le\; \Delta \sup_{t \le \Delta}
  | \dot z - \dot z^{*} | .
  \end{aligned}
\]
Further,
\[
  \begin{aligned}
  | \ddot z - \ddot z^{*} |
  \;\le{}&
  | a_{\ell} - a^{*} |
  + | v \omega - v^{*} \omega^{*} |
  + \tfrac{1}{2} | \theta - \theta^{*} | \\
  \;\le{}&
  \tfrac{1}{4(n-1)} + \tfrac{M_{a} \Delta}{j}
  + E \bigl( 1 + \tfrac{\Delta}{2} \bigr) .
  \end{aligned}
\]
So, for $\Delta \le 1$,
$$
  d\bigl( \tra_{a},\, \tra^{*};\, [0, \Delta] \bigr)
  \;\le\;
  \frac{2 + 3 \Delta}{8\,(n-1)}
  + \Bigl( \sqrt{2}\, M + \tfrac{5}{16} \Bigr)
  \frac{\Delta}{j} .
$$

\paragraph{3.\ Car with trailer.}
Let $\beta = \theta_{0} - \theta_{1}$ be the hitch angle and
$s = v \cos\beta$ the signed speed of the trailer. Along any
trajectory,
\begin{equation}
\label{eq:car-jet}
  \begin{aligned}
  \dot z &= s\, ( \cos\theta_{1},\, \sin\theta_{1} ), \\
  \ddot z &= \dot s\, ( \cos\theta_{1},\, \sin\theta_{1} )
  + \frac{v^{2} \sin\beta \cos\beta}{\ell}\,
  ( -\sin\theta_{1},\, \cos\theta_{1} ) .
  \end{aligned}
\end{equation}

Let $\tra^{*}$ be an $M$-smooth admissible reference from the
same state, with controls $u^{*} = (v^{*}, \phi^{*})$.
Differentiating $s = v \cos\beta$,
\[
  \begin{aligned}
  \ddot z ={}& \Bigl( \dot v \cos\beta
  - \tfrac{v^{2}}{L} \sin\beta \tan\phi
  + \tfrac{v^{2}}{\ell} \sin^{2}\!\beta \Bigr)
  ( \cos\theta_{1},\, \sin\theta_{1} ) \\
  &+ \frac{v^{2} \sin\beta \cos\beta}{\ell}\,
  ( -\sin\theta_{1},\, \cos\theta_{1} ) .
  \end{aligned}
\]
We suppose
\[
  |\dot v^{*}| \le \tfrac{1}{2},
  \qquad
  |\ddot v^{*}|, |\dot\phi^{*}| \le M
  \quad \text{a.e.},
\]
which are mild additional restrictions on the input signal.

Let $h_{0} = \tfrac{3}{25}$, $h_{1} = \tfrac{1}{2}$, and
$h_{\phi} = \tfrac{2\pi}{9}$ be the spacings of $G_{6}$,
$G_{3}$, and $G_{4}$. Using $v_{0}$ closest to $v^{*}(0)$,
and $\dot v_{\ell}$ and $\phi_{\ell}$ whose value and whose
tangent are closest to $\dot v^{*}$ and $\tan\phi^{*}$
averaged on $I_{\ell}$,
\[
  \begin{aligned}
  | v - v^{*} |
  &\;\le\; \tfrac{h_{0}}{2} + \tfrac{h_{1}}{2} \Delta
  + c\, \tfrac{M \Delta}{j}, \\
  | \theta_{k} - \theta_{k}^{*} |
  &\;\le\; c \Bigl( h_{0} + h_{1} + h_{\phi}
  + \tfrac{1}{j} \Bigr) \Delta,
  \quad k = 0, 1,
  \end{aligned}
\]
the trailer angle obeying
$\dot\theta_{1} = \tfrac{v}{\ell} \sin\beta$ with its
comparison factor $e^{\bar v \Delta / \ell}$ absorbed into
$c$. Through \eqref{eq:car-jet},
\[
  \begin{aligned}
  | \dot z - \dot z^{*} | ,
  \quad
  | \ddot z - \ddot z^{*} |
  &\;\le\; c \Bigl( h_{0} + h_{1} + h_{\phi}
  + \tfrac{M \Delta}{j} \Bigr), \\
  | z - z^{*} |
  &\;\le\; \Delta \sup_{t \le \Delta}
  | \dot z - \dot z^{*} | ,
  \end{aligned}
\]
so in the metric of Definition~\ref{def:dist}
$$
  d\bigl( \tra_{g},\, \tra^{*};\, [0, \Delta] \bigr)
  \;\le\;
  c \Bigl( h_{0} + h_{1} + h_{\phi} + \frac{M \Delta}{j} \Bigr),
$$
with $c = c(\bar v, L, \ell, \Delta)$ a constant depending on
the speed bound $\bar v = \tfrac{1}{2}$, the wheelbase $L$, the
hitch length $\ell$, and the window $\Delta$ alone, vanishing
in the grids and the subdivision together.

\paragraph{4.\ Quadrotor.}
Along any trajectory,
\begin{equation}
\label{eq:quad-jet}
  \begin{aligned}
  \ddot p &\;=\; \frac{c}{m}\, R(q)\, e_{3} - g e_{3}, \\
  \dddot p &\;=\; \frac{\dot c}{m}\, R(q)\, e_{3}
  + \frac{c}{m}\, R(q) \bigl( \omega \times e_{3} \bigr),
  \end{aligned}
\end{equation}
so that the flat jet depends on the thrust and its rate
$(c, \dot c)$ in addition to the state.

Let $\tra^{*}$ be an $M$-smooth admissible reference from
the same state,
$f^{*} = \ddot p^{*} + g e_{3}$ its thrust, and
$s^{*} = \langle R(q^{*})\, \omega^{*}, b_{3}^{*} \rangle$
its axial body rate.

We suppose the thrust $\gamma$-interior to the authority,
the body rate bounded, and the thrust axis nowhere parallel
to $e_{\psi^{*}} = (\cos\psi^{*}, \sin\psi^{*}, 0)$:
\[
  \begin{aligned}
  &F_{-} + \gamma \le \| f^{*} \| \le F_{+} - \gamma, \\
  &\omega^{*} \in [-\bar\omega, \bar\omega]^{3}, \\
  &\| b_{3}^{*} \times e_{\psi^{*}} \| \ge \gamma_{\psi} .
  \end{aligned}
\]
Along the reference, the angular velocity decomposes about
the thrust axis,
\[
  R(q^{*})\, \omega^{*}
  = b_{3}^{*} \times \dot b_{3}^{*} + s^{*} b_{3}^{*},
  \quad
  \dot b_{3}^{*}
  = \bigl( I - b_{3}^{*} (b_{3}^{*})^{\top} \bigr)
  \frac{\dot f^{*}}{\| f^{*} \|} ,
\]
so
$| s^{*} | \le \| \omega^{*} \| \le \sqrt{3}\, \bar\omega$.

$M$-smoothness bounds
$| \dddot p^{*} |, | p^{(4)*} |, | \ddot\psi^{*} | \le M$.
The rate $\dot f^{*} = \dddot p^{*}$ is therefore
$\sqrt{3} M$-Lipschitz.

Differentiating the axial body rate
$s^{*} = \langle R(q^{*})\, \omega^{*}, b_{3}^{*} \rangle$,
\[
  \dot s^{*}
  = \bigl\langle R(q^{*})\, \dot\omega^{*},\,
  b_{3}^{*} \bigr\rangle
  + \bigl\langle R(q^{*})\, \omega^{*},\,
  \dot b_{3}^{*} \bigr\rangle
  = \bigl\langle R(q^{*})\, \dot\omega^{*},\,
  b_{3}^{*} \bigr\rangle,
\]
the rotation contributing
$R(q^{*}) ( \omega^{*} \times \omega^{*} ) = 0$ and the
second term vanishing against the decomposition,
$\langle b_{3}^{*} \times \dot b_{3}^{*}
+ s^{*} b_{3}^{*},\, \dot b_{3}^{*} \rangle = 0$. Following
$J \dot\omega^{*} = \tau^{*} - \omega^{*} \times J
\omega^{*}$,
\[
  | \dot s^{*} | \le \| \dot\omega^{*} \|
  = \bigl\| J^{-1} \bigl( \tau^{*}
  - \omega^{*} \times J \omega^{*} \bigr) \bigr\|
  \le M_{s},
\]
with $\tau^{*}$ from the admissible rotors
$u^{*} \in [0, 1.3]^{4}$ and $M_{s}$ a constant of the
inertia, the mixer, and $\bar\omega$.

Let
$h_{0}$, $h_{1}$, $h_{2}$, and $h_{s}$ be the interval
spacings of $G_{n_{0}}$, $G_{n_{1}}$, $G_{n_{2}}$, and
$G_{n_{s}}$. Pick
$\bar a \ge F_{+} + g$, $\bar\jmath, \bar s \ge M$, and
$\bar r \ge \sqrt{3}\, \bar\omega$. Using $f_{0}$ and
$\dot f_{0}$ closest to $f^{*}(0)$ and $\dot f^{*}(0)$, and
$\ddot f$ and $r$ closest to $\ddot f^{*}$ and
$s^{*}$ averaged on $[0, \Delta]$, we get, for
$t \le \Delta$,
\begin{equation}
\label{eq:quad-ref}
\begin{aligned}
  \bigl\| f^{\mathrm{ref}}(t) - f^{*}(t) \bigr\|
  &\;\le\; \frac{\sqrt{3}\,h_{0}}{2}
  + c \bigl( h_{1} + h_{2} \bigr) \Delta
  + c\, M \Delta^{2},\\
  \bigl\| \dot f^{\mathrm{ref}}(t) - \dot f^{*}(t) \bigr\|
  &\;\le\; \frac{\sqrt{3}\,h_{1}}{2}
  + \frac{\sqrt{3}\,h_{2}}{2}\, \Delta
  + c\, M \Delta.
\end{aligned}
\end{equation}
Both trajectories leave the same state, so the attitude and
the body rate at $t = 0$ are shared, and
$f^{*}(0) = \tfrac{c^{*}(0)}{m}\, R(q(0))\, e_{3}$ is along
$b_{3}(0)$. Let $\theta$ be the angle from $b_{3}$ to
$b_{3}^{\mathrm{ref}}$ and
$e_{\omega} = \omega - \omega^{\mathrm{ref}}$ the body-rate
error. Following
$\| f_{0} - f^{*}(0) \| \le \tfrac{\sqrt{3}}{2} h_{0}$, and
supposing $\tfrac{\sqrt{3}}{2} h_{0} \le \gamma$ so that
$\| f_{0} \| \ge \| f^{*}(0) \|
- \tfrac{\sqrt{3}}{2} h_{0} \ge F_{-}$,
\[
  \sin\theta(0)
  = \frac{\bigl\| \bigl( f_{0} - f^{*}(0) \bigr) \times
  b_{3}(0) \bigr\|}{\| f_{0} \|}
  \le \frac{\sqrt{3}\, h_{0}}{2\, F_{-}} .
\]
Comparing the demand of \eqref{eq:quad-rate} at $t = 0$
with the decomposition of $R(q^{*})\, \omega^{*}$,
\[
  \begin{aligned}
  | e_{\omega}(0) |
  \le{}& k_{b}\, \theta(0)
  + \bigl\| b_{3}^{\mathrm{ref}} \times
  \dot b_{3}^{\mathrm{ref}}
  - b_{3}^{*} \times \dot b_{3}^{*} \bigr\| \\
  &+ | r - s^{*}(0) |
  + | s^{*}(0) |\,
  \bigl\| b_{3}^{\mathrm{ref}} - b_{3}^{*} \bigr\| \\
  \le{}& c \bigl( h_{0} + h_{1} + h_{s}
  + M_{s} \Delta \bigr),
  \end{aligned}
\]
the last two terms from $r$ matched to $s^{*}$ averaged
on $[0, \Delta]$. Substituting the torque \eqref{eq:quad-torque}
into $J \dot\omega = \tau - \omega \times J \omega$,
\[
  \dot e_{\omega}
  = \dot\omega - \dot\omega^{\mathrm{ref}}
  = - k_{\omega}\, e_{\omega},
  \qquad
  | e_{\omega}(t) | = | e_{\omega}(0) |\, e^{- k_{\omega} t}.
\]
Differentiating
$\cos\theta = \langle b_{3}, b_{3}^{\mathrm{ref}} \rangle$
with $\dot b_{3} = R\, ( \omega \times e_{3} )$ and
$\omega = \omega^{\mathrm{ref}} + e_{\omega}$, the
feedforward and the axial terms cancel, leaving
\[
  \begin{aligned}
  - \sin\theta\, \dot\theta
  &= k_{b}\, \theta \sin\theta
  + \bigl\langle R\, ( e_{\omega} \times e_{3} ),\,
  b_{3}^{\mathrm{ref}} \bigr\rangle \\
  &\ge \bigl( k_{b}\, \theta - | e_{\omega} | \bigr)
  \sin\theta ,
  \end{aligned}
\]
so $\dot\theta \le - k_{b}\, \theta + | e_{\omega} |$ and,
for $k_{\omega} > k_{b}$,
\[
  \theta(t) \le \theta(0)\, e^{- k_{b} t}
  + \frac{| e_{\omega}(0) |}{k_{\omega} - k_{b}} .
\]
Following \eqref{eq:quad-ref}, let
$\| f^{\mathrm{ref}} - f^{*} \| \le \gamma / 2$, so that
$\| f^{\mathrm{ref}} \| \in [F_{-} + \gamma/2,\, F_{+}]$, and
let the grids be fine enough that the bounds above keep
$\theta \le \sqrt{\gamma / F_{+}}$. Then
\[
  \langle f^{\mathrm{ref}}, b_{3} \rangle
  \;=\; \| f^{\mathrm{ref}} \| \cos\theta
  \;\ge\; F_{-} + \tfrac{\gamma}{2}
  - F_{+}\, \tfrac{\theta^{2}}{2}
  \;\ge\; F_{-} ,
\]
the clamp of \eqref{eq:quad-thrust} is inactive,
$c = m \langle f^{\mathrm{ref}}, b_{3} \rangle$, and
\[
  \tfrac{c}{m}\, b_{3} - f^{\mathrm{ref}}
  = - \bigl( I - b_{3} b_{3}^{\top} \bigr) f^{\mathrm{ref}},
  \quad
  \bigl\| \tfrac{c}{m}\, b_{3} - f^{\mathrm{ref}} \bigr\|
  \le F_{+}\, \theta .
\]
Combining with \eqref{eq:quad-ref} through
\eqref{eq:quad-jet}, where $\omega \times b_{3}$ discards
the axial channel,
\[
  \begin{aligned}
  \| \ddot p - \ddot p^{*} \| ,\;
  \| \dddot p - \dddot p^{*} \|
  &\le c \bigl( h_{0} + h_{1} + h_{2}
  + M \Delta \bigr), \\
  \| \dot p - \dot p^{*} \|
  &\le \Delta \sup_{t \le \Delta}
  \| \ddot p - \ddot p^{*} \| , \\
  \| p - p^{*} \|
  &\le \Delta \sup_{t \le \Delta}
  \| \dot p - \dot p^{*} \| .
  \end{aligned}
\]
\begin{table*}[t]
\caption{The search variants of Section~\ref{sec:variants} on
the Dynobench instances of Section~\ref{sec:dynobench},
averaged over $5$ runs at the Orin Nano compute tier and
reported against the headline algorithm
DFT\textsuperscript{*}. The static variants run at the pruning
radius lowered by $0.01$, without which their cross-depth
claims can starve the frontier short of the goal on some runs.}
\label{tab:variants}
\centering
\footnotesize
\setlength{\tabcolsep}{3pt}
\begin{tabular*}{\textwidth}{|@{\extracolsep{\fill}\hspace{\tabcolsep}}ll|cc|cc|cc|cc|}
\hline
 & & \multicolumn{2}{c|}{DFT\textsuperscript{*}} &
\multicolumn{2}{c|}{DFT\textsuperscript{*}-Static} &
\multicolumn{2}{c|}{DFT-A\textsuperscript{*}} &
\multicolumn{2}{c|}{DFT-A\textsuperscript{*}-Static} \\
\hline
system & task &
cost (s) & wall (ms) &
cost (s) & wall (ms) &
cost (s) & wall (ms) &
cost (s) & wall (ms) \\
\hline
\multirow{3}{*}{\shortstack[l]{Unicycle\\1st order}}
& bugtrap & $21.4$ & \wms{169}
& \cgood{21.2} & \wmsb{18}
& $21.3$ & \wms{82}
& $21.5$ & \wms{24} \\
 & kink & $14.1$ & \wms{50}
& $13.6$ & \wms{8}
& $14.1$ & \wms{9}
& \cgood{13.4} & \wmsb{6} \\
 & parallelpark & $3.5$ & \wms{2}
& $3.2$ & \wmsb{1}
& $3.4$ & \wms{2}
& \cgood{3.1} & \wms{2} \\
\hline
\multirow{3}{*}{\shortstack[l]{Unicycle\\2nd order}}
& bugtrap & $24.5$ & \wms{3220}
& $24.2$ & \wms{388}
& $24.4$ & \wms{1117}
& \cgood{23.7} & \wmsb{333} \\
 & kink & $18.8$ & \wms{1100}
& $18.0$ & \wms{176}
& $19.2$ & \wms{176}
& \cgood{17.6} & \wmsb{76} \\
 & parallelpark & $6.7$ & \wmsb{10}
& \cgood{6.0} & \wmsb{10}
& $7.0$ & \wms{11}
& \cgood{6.0} & \wmsb{10} \\
\hline
\multirow{3}{*}{\shortstack[l]{Car with\\a trailer}}
& bugtrap & $19.8$ & \wms{2718}
& \cgood{19.5} & \wms{217}
& $20.0$ & \wms{928}
& $19.7$ & \wmsb{148} \\
 & kink & $15.6$ & \wms{1251}
& $15.2$ & \wms{152}
& $15.9$ & \wms{89}
& \cgood{14.9} & \wmsb{28} \\
 & parallelpark & \cgood{3.5} & \wmsb{6}
& $4.4$ & \wms{12}
& $4.1$ & \wms{9}
& $4.9$ & \wms{12} \\
\hline
\multirow{4}{*}{\shortstack[l]{Quadrotor\\
  $u = (u_{i})_{i=1}^{4}$}}
& block & $2.55$ & \wms{2719}
& \cgood{2.50} & \wms{995}
& \cgood{2.50} & \wms{1727}
& \cgood{2.50} & \wmsb{931} \\
 & window & $2.00$ & \wms{385}
& \cgood{1.90} & \wmsb{150}
& $2.00$ & \wms{362}
& $1.95$ & \wms{164} \\
 & inversion & \cgood{3.10} & \wms{313}
& $3.60$ & \wmsb{205}
& $3.20$ & \wms{653}
& $3.65$ & \wms{300} \\
& inversion obs. & \cgood{3.90} & \wms{3345}
& $4.25$ & \wmsb{934}
& $4.10$ & \wms{5262}
& $4.05$ & \wms{1012} \\
\hline
\end{tabular*}
\end{table*}

With $\| b_{3}^{*} \times e_{\psi^{*}} \| \ge \gamma_{\psi}$
and $\| b_{3} - b_{3}^{*} \| \le \tfrac{\gamma_{\psi}}{2}$,
$\psi$ and $\dot\psi$ are smooth functions of the attitude
and the body rate along both trajectories, so in the metric
of Definition~\ref{def:dist}
$$
  d( \tra_{\mathrm{ref}}, \tra^{*}; [0, \Delta] )
  \le
  c ( h_{0} + h_{1} + h_{2} + h_{s}
  + (M + M_{s}) \Delta ),
$$
with $c = c(k_{b}, k_{\omega}, F_{-}, F_{+}, \bar\omega, M,
\Delta)$ a constant depending on the gains, the thrust
envelope, the body-rate box, the envelope constant, and the
window alone, vanishing in the grids and the window together.

It remains to account for the rotor projection of the
construction, $u = \mathrm{clip}\bigl( B_{0}^{-1}( c, \tau ),\,
[0, 1.3]^{4} \bigr)$. Abbreviate the bound above,
$\varrho := h_{0} + h_{1} + h_{2} + h_{s}
+ ( M + M_{s} )\, \Delta$. The reference is admissible,
$u^{*} \in [0, 1.3]^{4}$, so the projection is nonexpansive
and fixes $u^{*}$, while the demand map of
\eqref{eq:quad-thrust} and \eqref{eq:quad-torque} is
Lipschitz for $\theta < \pi$ and deviates from $u^{*}$ by
$c\, \varrho$ along the reference. The state deviation $e$ of
the executed rollout hence obeys
\[
  \dot e
  \;\le\;
  L \bigl( e + \bigl\| \mathrm{clip}\, B_{0}^{-1}( c, \tau )
  - u^{*} \bigr\| \bigr)
  \;\le\;
  L'\, e + c\, \varrho ,
\]
and Gronwall over the window recovers the displayed bound
with $c$ enlarged by $e^{L' \Delta}$, mirroring the rejection
step of Theorem~\ref{thm:realization}.

\section{Search Variants on the Dynobench Instances}
\label{app:variants}

See Table~\ref{tab:variants}.

\balance
\section{Index to Multimedia Extensions}

\noindent
\begin{tabular}{@{}cll@{}}
\hline
Extension & Media type & Description \\
\hline
1 & Video & WWDFT\textsuperscript{*} hard dynamic episode \\
2 & Video & DFT\textsuperscript{*} Dynobench solutions \\
\hline
\end{tabular}

\end{document}